\documentclass[11pt]{article}

\usepackage{amsmath,amssymb,amsthm,mathtools}
\usepackage[round]{natbib}
\usepackage[margin=1in]{geometry}
\usepackage{thmtools}

\numberwithin{equation}{section}
\usepackage{todonotes}
\usepackage{hyperref}
\hypersetup{
    colorlinks,
    citecolor=blue,
    filecolor=blue,
    linkcolor=blue,
    urlcolor=blue
}
\allowdisplaybreaks
\newtheorem{theorem}{Theorem}[section]
\newtheorem{proposition}{Proposition}

\newtheorem{definition}{Definition}[section]
\newtheorem{lemma}{Lemma}[section]
\newtheorem{corollary}{Corollary}[section]
\newtheorem{remark}{Remark}[section]
\newcommand{\ip}[2]{\left\langle #1,#2\right\rangle}
\newcommand{\KL}{\mathrm{KL}}
\newcommand{\E}{\mathbb{E}}
\newcommand{\R}{\mathbb{R}}
\newcommand{\Var}{\textsf{Var}}

\renewcommand{\Pr}{\mathsf{Pr}}
\usepackage{pifont}
\newcommand{\cmark}{\ding{51}}
\newcommand{\xmark}{\ding{55}}

\newcommand{\norm}[1]{\left\lVert #1\right\rVert}
\newcommand{\mnorm}[2]{\left\lVert #1\right\rVert_{#2}}

\newcommand{\gauss}[3]{\varphi_{#3}\!\left(#1;#2\right)} % \gauss{y}{\mu}{\Sigma} = phi_\Sigma(y;\mu)

\title{A Quantitative Characterization of Forgetting in Post-Training}
\author{Krishnakumar Balasubramanian$^{1,2}$, Shiva Prasad Kasiviswanathan$^{2}$\\
$^1$Department of Statistics, University of California, Davis\\
$^2$Amazon}
\date{\today}

\begin{document}
\maketitle

\begin{abstract}

Continual post-training of generative models is widely used, yet a principled understanding of \emph{when} and \emph{why} forgetting occurs remains limited. We develop theoretical results under a two-mode mixture abstraction (representing old and new tasks), proposed by~\cite{chen2025retaining}, and formalize forgetting in two forms: (i) \emph{mass forgetting}, where the old mixture weight collapses to zero, and (ii) \emph{old-component drift}, where an already-correct old component shifts during training. For equal-covariance Gaussian modes, we prove that forward-KL objectives trained on data from the new distribution drive the old weight to zero, while reverse-KL objectives converge to the true target (thereby avoiding mass forgetting) and perturb the old mean only through overlap-gated misassignment probabilities controlled by the Bhattacharyya coefficient, yielding drift that decays exponentially with mode separation and a locally well-conditioned geometry with exponential convergence.  We further quantify how replay interacts with these objectives. For forward-KL, replay must modify the training distribution to change the population optimum; for reverse-KL, replay leaves the population objective unchanged but prevents finite-batch “old-mode starvation” through bounded importance weighting. Finally, we analyze three recently proposed near–on-policy post-training methods, SDFT~\citep{shenfeld2026self}, TTT-Discover~\citep{yuksekgonul2026learning}, and OAPL~\citep{ritter2026llms},  via the same lens and derive explicit conditions under which each retains old mass and exhibits overlap-controlled drift. Overall, our results show that forgetting can by precisely quantified based on the interaction between divergence direction, geometric behavioral overlap, sampling regime, and the visibility of past behavior during training.  
\end{abstract}

\section{Introduction}

Continual learning investigates how sequentially trained models can acquire new capabilities without erasing old ones, a process fundamentally challenged by catastrophic forgetting, where performance on earlier tasks rapidly degrades. While the literature contains many algorithmic responses (see Section~\ref{relatedworks}), the mechanisms behind forgetting are less unified, especially for post-training pipelines in modern generative models whose ``behavior'' is best represented as a probability distribution over outputs. In this paper, by viewing training procedures as a divergence-minimization or distribution-matching step, we ask a basic question: 
\begin{center}
\emph{Can we precisely quantify when a post-training procedure \\induces forgetting and when it does not?}
\end{center}

We aim to answer this question by studying the two-mode mixture model proposed by~\cite{chen2025retaining} that abstracts a continual-learning step into ``old'' and ``new'' distributions. Let $p_\mathrm{o}$ and $p_\mathrm{n}$ denote the old and new data-generating distributions over an output space $\mathcal{Y}$ (where the subscript $\mathrm{o}$ and $\mathrm{n}$
denotes \emph{old} and \emph{new} respectively). We define a \emph{true} target mixture
\[
p_\alpha(y) \;=\; \alpha p_\mathrm{o}(y) + (1-\alpha)p_\mathrm{n}(y), \qquad \alpha\in [0,1],
\]
which represents the ideal outcome of learning the new behavior while retaining an $\alpha$ fraction of the old behavior. We consider an \emph{learner} model family that explicitly contains two components (one intended for the old mode and one for the new mode),
\[
q_{\beta}(y) \;=\; \beta q_\mathrm{o}(y) + (1-\beta)q_\mathrm{n}(y), \qquad \beta\in [0,1].
\]
In this formulation, the parameters to be learned consist of the mixture weight ($\beta$) and the parameters governing the component distributions ($q_\mathrm{o}$ and $q_\mathrm{n}$) of the model. 
Throughout this paper, we assume that the component $q_\mathrm{o}$ has already been trained to approximate the old distribution $p_\mathrm{o}$. 
Continual learning in this setup (and more broadly) typically refers to the process of learning the new distribution $p_\mathrm{n}$ while preserving the previously learned behavior encoded by $p_\mathrm{o}$. In the mixture formulation above, this corresponds to updating the parameters of the component $q_\mathrm{n}$ and the mixture weight $\beta$ so that the learned model $q_\beta$ approximates the target mixture $p_\alpha$, while ensuring that the component $q_\mathrm{o}$ continues to represent the previously learned distribution $p_\mathrm{o}$.

% Several approaches have been proposed to achieve this. One approach assumes access to a sampling oracle that can generate on-policy data from the current model~\citep{}. In such on-policy methods, the training process continually samples from the current model so that earlier behaviors remain represented in the evolving training distribution. Another approach relies on an explicit memory oracle that stores and replays past data~\citep{}. In these replay-based methods, samples from previous tasks are stored and incorporated into the training process for subsequent tasks, ensuring that previously learned behaviors continue to influence learning.

% Continual learning in this setup (and in general) can be characterized to rely on some mechanism for maintaining exposure to previously learned behaviors (i.e., information about $p_\mathrm{o}(y)$) during training. In practice, this is done by either the availability of a \emph{sampling oracle} capable of generating on-policy data from the current model, or an explicit \emph{memory oracle} that stores and replays past data. On-policy approaches work by continually sampling from the current model itself so that earlier behaviors remain represented in the current training distribution. Replay-based methods assume access to a memory oracle to store samples from previous tasks to be used in the current training process. 

We distinguish two distinct forms of forgetting in this continual learning setup. Mass forgetting corresponds to the collapse of mixture mass on the old mode, whereas old-component drift occurs when the model retains nonzero mass on the old mode but its parameters move away from the true old distribution:
\begin{itemize}
    \item[(i)] \emph{Mass Forgetting (Mass Collapse)}: This occurs when the optimal mixture weight satisfies 
$\beta^\star = 0$, meaning that the learned model places zero mass on the old mode. 
Equivalently, the mixture weight on the old component undergoes \emph{mass collapse}. 
We show that this can arise even when $q_\mathrm{o}(y)=p_\mathrm{o}(y)$ and 
$q_\mathrm{n}(y)=p_\mathrm{n}(y)$, i.e., when the learner is given the correct 
forms of both the old and new distributions. In this setting the only learnable 
parameter is the mixing proportion $\beta$. Thus $\beta^\star=0$ (instead of the 
desired $\beta^\star=\alpha$) represents a strong form of forgetting in which 
the learned model discards the old behavior despite having access to its exact distribution.
    
   \item[(ii)] \emph{Old-Component Drift}: 
This occurs when, during continual training, the 
parameters of the learned old component $q_\mathrm{o}$ drift away from the true 
old distribution $p_\mathrm{o}$. In this case the model may still allocate 
nontrivial mass to the old mode (i.e., $\beta$ need not collapse to zero), but the 
parameters governing the old component shift so that the learned distribution no 
longer faithfully represents the original behavior. For example, in a location 
family this corresponds to the mean parameter of the old component drifting away 
from the true old mean.
   \end{itemize}
This setting is intentionally minimal\footnote{Appendix~\ref{app:finitek} shows that similar conclusions hold for finite-mixture models.}: there are only two modes, the model family is expressive enough to represent both, and yet the aforementioned forms of forgetting can still occur. The benefit of this simplicity is that it enables exact decompositions and sharp theorems that cleanly separate objective-driven forgetting from representational limits.

A central theme in this work, motivated by modern post-training pipelines, is the contrast between \emph{forward} and \emph{reverse} KL divergence based training objectives given respectively by 
\[
\min_\theta \mathrm{KL}(p \,\|\, q_\theta), \qquad\text{and}\qquad \min_\theta \mathrm{KL}(q_\theta \,\|\, p).
\]
The forward-KL is the population analogue of maximum likelihood on a ``data'' distribution $p$. In the context of the model above, forward-KL correspond to SFT-based training with only new data (i.e., $p=p_\mathsf{n}$). The reverse-KL objective, is the population analogue of matching the model to a target distribution $p$ under on-policy sampling from $q_\theta$ (a common lens for KL-regularized policy improvement and RL-style updates). 

We consider the case where $p_\mathrm{o} = \mathcal{N}(\mu_\mathrm{o},\Sigma)$ and 
$p_\mathrm{n} = \mathcal{N}(\mu_\mathrm{n},\Sigma)$ with separation 
$\delta := \|\mu_\mathrm{n}-\mu_\mathrm{o}\|_{\Sigma^{-1}}$. The learner model $q$ is parameterized as a two-component mixture with weight $\beta$ and component means $(m_\mathrm{o},m_\mathrm{n})$, with both components sharing covariance $\Sigma$. In this setting we can explicitly analyze the dynamics of forgetting under the different training objectives described above. To isolate the effect of learning the mixture weight, we further set $m_\mathrm{o}=\mu_\mathrm{o}$ and $m_\mathrm{n}=\mu_\mathrm{n}$ in the learner model and optimize only over $\beta$. Our first main result shows that, in this setting, the new-data-only forward-KL-based SFT training objective
\[
L_{\mathrm{SFT}}(\beta) := \mathrm{KL}(p_\mathrm{n} \,\|\, q_\beta)
\]
is strictly increasing in $\beta$ (even when the component shapes are already correct), so $\beta^\star=0$ is the unique population minimizer. 
Moreover, under logit-parameterized gradient flow, the trajectory $\beta(t)$ (corresponding to the population training objective) decreases monotonically to $0$. The analysis reveals an intuitive mechanism: the gradient is given by the difference between the current old mass $\beta$ and the \emph{expected old responsibility} (i.e., average posterior probability, under a given sampling distribution, that an observation is assigned to the old component) under the new data distribution. When the modes are well separated, this responsibility is exponentially small, so the update effectively reduces to repeatedly shrinking $\beta$ until the old mode vanishes.

Our second main result considers reinforcement learning with a reverse-KL objective
\[
L_{\mathrm{RL}}(\beta,m_\mathrm{o},m_\mathrm{n}) := \mathrm{KL}(q_{\beta,m_\mathrm{o},m_\mathrm{n}} \,\|\, p_\alpha),
\]
which corresponds to KL-regularized on-policy RL updates toward a target distribution that explicitly retains the old behavior\footnote{See Appendix~\ref{app:rltokl} for details.}. When the old component is already correct (e.g., $m_\mathrm{o}=\mu_\mathrm{o}$), the gradient with respect to the old parameters admits an \emph{exact decomposition}: the only terms capable of moving the old mode arise from \emph{misassignment probabilities}, i.e., responsibilities that incorrectly attribute an old-mode sample to the new component and vice versa under the target.

These misassignment probabilities are controlled by an overlap quantity (the Bhattacharyya coefficient), yielding bounds that decay exponentially with the squared Mahalanobis distance between the means for equal-covariance Gaussians. Consequently, in the well-separated regime, reverse-KL updates can meaningfully adjust the new mode while perturbing the old mode only through exponentially small overlap effects. Finally, a local Polyak--\L{}ojasiewicz (PL) analysis shows that, in sufficiently separated regimes, the reverse-KL objective exhibits a favorable local geometry that implies exponential convergence under gradient flow.

\begin{table}[t]
\hspace*{-3ex}
\centering
\def\arraystretch{1.25}
\begin{tabular}{|c|c|c|}
\hline
Method & Prevents Mass Forgetting? & Controls Old-Component Drift? \\
\hline \hline
Forward-KL (SFT) & \xmark & \cmark  \\ 
 & (Theorem~\ref{thm:kl_main})& (but unimportant  \\ 
  & & as mass collapses)  \\ \hline
Reverse-KL (RL) & \cmark & \cmark\ Exponentially small in $\delta$ \\ 
 & (Theorem~\ref{thm:kl_stationary_beta_mn}) & (Theorem~\ref{lem:fwdKL_oldmean_drift}) \\ \hline
SDFT & \cmark\ If demonstrator strength is $>0$ & \cmark\ Finite total drift$^\star$ \\
 \citep{shenfeld2026self} & (Theorem~\ref{thm:sdft_demo_ema}(A)) & (Theorem~\ref{thm:sdft_demo_ema}(B)) \\ \hline
TTT-Discover & \cmark\ If anchor sufficiently strong; & \cmark\ Exponentially small in $\delta$ \\
\citep{yuksekgonul2026learning} & Collapses if anchor too weak &(Theorem~\ref{thm:ttt_gaussian}(B)) \\
& (Theorem~\ref{thm:ttt_gaussian}(A)) &  \\ \hline
OAPL & Partial: bounded by old-mode weight & \cmark\ Exponentially small in $\delta$ \\
\citep{ritter2026llms}& of the frozen reference policy & (Theorem~\ref{thm:OAPL_gaussian}(B))\\
\citep{brantley2025accelerating}& (Theorem~\ref{thm:OAPL_gaussian}(A)) &  \\
\hline
\end{tabular}
\caption{Summary of forgetting behavior across training objectives. 
``Prevents Mass Forgetting'' means the population optimum satisfies $\beta^\star>0$. 
``Controls Old-Component Drift'' means the gradient $\|\nabla_{m_o}L\|$ is provably small 
when $m_o=\mu_o$. Mode separation $\delta=\|\mu_n-\mu_o\|_{\Sigma^{-1}}$. $\star$: We show it is exponentially small under additional assumptions; see Remark~\ref{rem:sdft_exp_separation}.}
\label{tab:forgetting_summary}
\end{table}

\paragraph{Effect of Replay on SFT and RL.} We also examine the effect of replay and quantify how it interacts with forward- and reverse-KL objectives in fundamentally different ways. For forward-KL (SFT), replay prevents mass forgetting only when it enters the training distribution (i.e., the numerator of the objective): mixing a $\lambda$ fraction of old samples into the data shifts the population optimum to retain $\beta^\star=\lambda$. In contrast, mixing old samples only on the model side leaves the learned parameter collapsing to $\beta^\star=0$ and merely imposes an external retention floor.  For reverse-KL (KL-regularized RL), replay does not alter the population objective but instead addresses a finite-batch failure mode. By ensuring that old-mode samples appear in minibatches with high probability and using bounded importance weights, replay preserves the same reverse-KL gradient in expectation while preventing stochastic ``old-mode starvation'' that can otherwise mimic new-only updates. Together, these results show that replay plays fundamentally different roles in the two settings: for SFT it modifies the \emph{population objective}, whereas for reverse-KL methods it improves the \emph{stochastic optimization dynamics}.

\paragraph{Near-on-policy Methods.} We next consider three recent near-on-policy post-training methods, namely SDFT~\cite{shenfeld2026self}, TTT-discover~\cite{yuksekgonul2026learning}  and OAPL~\cite{ritter2026llms,brantley2025accelerating}. Our mixture-model analysis reveals a sharp difference between the different algorithms. SDFT behaves like a reverse-KL update toward an evolving teacher distribution generated from the model itself based on a demonstrator. It avoids mass forgetting if the demonstrator is strong enough, while avoiding the old-component drift. TTT-Discover’s entropic objective is intrinsically mode-seeking: without a sufficiently strong KL anchor it can still collapse mass onto the higher-reward mode, although the drift of an already-correct old mode remains overlap-gated and decays exponentially with separation. OAPL behaves differently because its target is an exponential tilt of a frozen reference policy: it can only preserve or reweight modes already present in that reference, but its parametric updates are likewise geometrically local, with cross-mode influence controlled by exponentially small overlap terms.  Together, these results show that these three methods inherit the stability of on-policy learning, but forgetting and retention is governed by different mechanisms. A summary of our results and conclusions is provided in Table~\ref{tab:forgetting_summary}.

%Across all of our results, the core message is that forgetting in continual post-training can be quantified, and in many cases predicted, by a small set of interacting factors: (i) which divergence direction the update implicitly minimizes (forward vs.\ reverse), (ii) whether the training signal is off-policy (data-to-model) or on-policy (model-to-target), (iii) the geometric overlap between old and new behaviors, and (iv) whether past behavior remains visible at finite batch sizes during optimization.  

%\paragraph{Takeaway.}
%Across all of these results, the core message is that forgetting in continual learning can be quantified---and in some cases predicted---from the interaction between (i) the direction of divergence being minimized (forward vs.\ reverse), (ii) whether training is off-policy or on-policy, (iii) the overlap between old and new behaviors, and (iv) the finite-sample visibility of old behavior during optimization. 

\subsection{Intuition via Disjoint-support Case}
As a prelude to our results that follow in the rest of this draft, in this section, we study the limiting case where each ``mode'' lives on a separate region of the sample space. The core intuition behind our general results are well-captured by this simplified setup.

\begin{definition}\label{def:disjointsupp}
    Let $(\mathcal{Y},\mathcal{F},\mu)$ be a measurable space with reference measure $\mu$.
Assume there exist measurable sets $A_{\mathrm{o}},A_{\mathrm{n}}\subseteq\mathcal{Y}$ forming a partition:
$A_{\mathrm{o}}\cap A_{\mathrm{n}}=\emptyset$, $A_{\mathrm{o}}\cup A_{\mathrm{n}}=\mathcal{Y}$.
Assume the component densities satisfy
\[
p_{\mathrm{o}}(y)=0=q_{\mathrm{o}}(y)\ \ \text{for }y\notin A_{\mathrm{o}},
\qquad
p_{\mathrm{n}}(y)=0=q_{\mathrm{n}}(y)\ \ \text{for }y\notin A_{\mathrm{n}},
\]
where $q_{\mathrm{o}}(\cdot)$ and $q_{\mathrm{n}}(\cdot)$ are the (model) component densities.
Define the mixtures
\[
p_\alpha(y)=\alpha p_{\mathrm{o}}(y)+(1-\alpha)p_{\mathrm{n}}(y),
\qquad
q_\beta(y)=\beta q_{\mathrm{o}}(y)+(1-\beta)q_{\mathrm{n}}(y),
\]
and note that on $A_{\mathrm{o}}$ we have $p_\alpha=\alpha p_{\mathrm{o}}$, $q_\beta=\beta q_{\mathrm{o}}$, and similarly on $A_{\mathrm{n}}$.
\end{definition}

In this disjoint-support limit, both forward and reverse-KL admit exact decompositions into a \emph{mixture-weight term} and within-mode terms. This makes the key point transparent: if training uses \emph{new-only} data ($p=p_\mathrm{n}$), then the forward-KL objective contains an explicit penalty that is strictly increasing in the old-mode weight $\beta$, and thus the unique optimizer is $\beta^\star=0$. Crucially, this collapse occurs \emph{regardless of how well the old component is modeled}, because the new-only objective has no incentive to allocate probability mass to a region it never observes. This provides a clean caricature of catastrophic forgetting as mass collapse driven by off-policy training.

\begin{lemma}[Exact KL decompositions under disjoint supports]\label{lem:disjoint-kl}
Under the disjoint-support assumption,
\begin{align*}
\KL(p_\alpha\|q_\beta)
&=
\alpha\log\frac{\alpha}{\beta} + (1-\alpha)\log\frac{1-\alpha}{1-\beta}
\;+\;
\alpha\,\KL(p_{\mathrm{o}}\|q_{\mathrm{o}})
\;+\;
(1-\alpha)\,\KL(p_{\mathrm{n}}\|q_{\mathrm{n}}),
\\
\KL(q_\beta\|p_\alpha)
&=
\beta\log\frac{\beta}{\alpha} + (1-\beta)\log\frac{1-\beta}{1-\alpha}
\;+\;
\beta\,\KL(q_{\mathrm{o}}\|p_{\mathrm{o}})
\;+\;
(1-\beta)\,\KL(q_{\mathrm{n}}\|p_{\mathrm{n}}).
\end{align*}
\end{lemma}
\begin{proof}[Proof of Lemma~\ref{lem:disjoint-kl}]
We prove the first identity; the second is analogous.
Split the integral over $A_{\mathrm{o}}\cup A_{\mathrm{n}}$:
\[
\KL(p_\alpha\|q_\beta)=\int p_\alpha\log\frac{p_\alpha}{q_\beta}\,d\mu
= \int_{A_{\mathrm{o}}} \alpha p_{\mathrm{o}}\log\frac{\alpha p_{\mathrm{o}}}{\beta q_{\mathrm{o}}}\,d\mu
+ \int_{A_{\mathrm{n}}} (1-\alpha) p_{\mathrm{n}}\log\frac{(1-\alpha)p_{\mathrm{n}}}{(1-\beta)q_{\mathrm{n}}}\,d\mu.
\]
Inside each region, expand the log:
$\log\frac{\alpha p_{\mathrm{o}}}{\beta q_{\mathrm{o}}}=\log\frac{\alpha}{\beta}+\log\frac{p_{\mathrm{o}}}{q_{\mathrm{o}}}$, and similarly for the new mode.
Using $\int_{A_{\mathrm{o}}}p_{\mathrm{o}}\,d\mu=\int_{A_{\mathrm{n}}}p_{\mathrm{n}}\,d\mu=1$ yields the stated decomposition.
\end{proof}

\begin{remark}[Exact mode locality for \emph{shape} parameters]\label{rmk:drift}
If $q_{\mathrm{o}}(\cdot)$ and $q_{\mathrm{n}}(\cdot)$ have separate parameter vectors (say $\theta_{\mathrm{o}}$ and $\theta_{\mathrm{n}}$),
Lemma~\ref{lem:disjoint-kl} implies that, \emph{holding mixture weights fixed}, minimizing either KL w.r.t.\ $\theta_{\mathrm{o}}$ depends only on the old-mode divergence, and similarly for $\theta_{\mathrm{n}}$.
In particular, if $q_{\mathrm{o}}=p_{\mathrm{o}}$, then $\nabla_{\theta_{\mathrm{o}}}\KL(q_\beta\|p_\alpha)=0$ and $\nabla_{\theta_{\mathrm{o}}}\KL(p_\alpha\|q_\beta)=0$:
the old mode is exactly stationary while the new mode can be updated.
\end{remark}

\begin{remark}[Why \emph{Weights} Can Still Collapse Under forward-KL]
Even in the disjoint-support case, if the training distribution is new-only, i.e.\ $p=p_{\mathrm{n}}$,
then for the forward-KL objective $\KL(p_{\mathrm{n}}\|q_\beta)$ the decomposition reduces to
\[
\KL(p_{\mathrm{n}}\|q_\beta) = \log\frac{1}{1-\beta} + \KL(p_{\mathrm{n}}\|q_{\mathrm{n}}),
\]
which is strictly increasing in $\beta$.
Thus, optimizing forward-KL on new-only data drives $\beta\to 0$ (all mass on the new mode), regardless of how well the old mode was modeled.
This is a clean caricature of catastrophic forgetting by mass collapse.
By contrast, for reverse-KL, the weight term
$\beta\log(\beta/\alpha)+(1-\beta)\log((1-\beta)/(1-\alpha))$ penalizes moving $\beta$ far from $\alpha$.
\end{remark}

\subsection{Related Works}\label{relatedworks}

%Catostriphic forgetting in continual learning has it origins from the study of biological systems~\citep{mccloskey1989catastrophic,mcclelland1995there}. In the context of machine learning, classical works on continual learning and forgetting include~\cite{thrun1995learning} and~\cite{french1999catastrophic}. 

From a methodological perspective, continual learning can be broadly categorized to rely on some mechanism for maintaining exposure to previously learned behaviors (i.e., information about $p_\mathrm{o}(y)$ in our setup) during training. In practice, this is done by either the availability of an explicit \emph{memory oracle} that stores and replays past data or a \emph{sampling oracle} capable of generating on-policy data from the current model. We refer to~\cite{wang2024comprehensive} and \cite{shi2025continual} for recent surveys. Below we provide a admittedly incomplete overview of a few related works on continual learning, with a focus on methods for generative models and general theoretical results.

Replay-based methods~\citep{schaul2015prioritized,lopez2017gradient} assume access to a memory oracle to store samples from previous tasks to be used in the current training process. A special case of replay-based methods are regularization-based method where the model parameter is stored (instead of the training data) and used as a regularizer when training for new tasks~\citep{kirkpatrick2017overcoming,li2017learning, schwarz2018progress}. Theoretical lower bounds on the amount of memory required for continual learning was established by~\cite{chen2022memory}. In the context of large generative models, replay-based typically involved re-training a pre-trained model and hence maybe inefficient (or even in-feasible if the pre-trained model is closed-weights). Nevertheless, several works have proposed algorithms to efficiently use replay methods in the context of such large models~\citep{shin2017continual}.

In the context of large pre-trained generative models, on-policy approaches that continually sample from the current model and train on the resulting data are widely used. Such methods appear in both reinforcement learning and supervised fine-tuning settings, for example in on-policy distillation and policy-improvement style updates~\citep{tajwar2024preference,lu2025onpolicydistillation,zhao2026policy,chen2025retaining,shenfeld2025rl}. Related ideas also arise in mid-training procedures that bridge pre-training and post-training distributions~\citep{liu2025midtraining}, as well as in self-distillation methods that iteratively train a model on samples generated by its own policy~\citep{shenfeld2026self,zhao2026self,hubotter2026reinforcement,penaloza2026privileged}. Earlier work explored connections between reinforcement learning and distribution matching for language-model fine-tuning~\citep{korbak2022reinforcement}, and recent methods such as OAPL construct improvement targets relative to a lagged reference policy~\citep{ritter2026llms,brantley2025accelerating}.

Theoretical results on forgetting under (overparameterized) linear models have been studied recently by many authors, including~\cite{evron2022catastrophic,lin2023theory,li2025theory,ding2024understanding,deng2025unlocking,banayeeanzade2025theoretical,karpel2026optimal}. The linear classification setting was further analyzed by~\cite{evron2023continual}. PAC-Bayes bounds for continual learning were established recently by~\cite{friedman2026pacbayes}, and gradient descent dynamics in continual learning problems was studied by~\cite{bennani2021generalisation, doan2021theoretical, karakida2022learning, cai2025last, taheri2025theory,graldi2025the}. In contrast to these works, our results aim to provide a principled understanding of practical post-training methods used for generative models, in particular foward and backward KL based fine-tuning. Perhaps the closest to our work is~\cite{chan2022greedification}, which studies the role of forward and reverse-KL divergences in approximate policy iteration and analyzes their policy-improvement properties in reinforcement learning. In contrast, our work focuses on continual learning and forgetting in generative-model post-training, where the forward and reverse-KL objectives arise from SFT and RL-style updates, and we quantify how these objectives induce or prevent forgetting through a distributional mixture model.

\section{Forgetting in Forward and Reverse-KL Objectives}\label{sec:main_results}

A key aspect of our analysis is the difference between forward- and reverse-KL objectives and how they affect continual learning dynamics. We emphasize many of the forthcoming result assume mixture-of-two Gaussians for simplicity of exposition; extensions to finite-mixture and strongly log-concave densities are provided in Appendices~\ref{app:finitek} and~\ref{app:logconcaveext} respectively. Similarly, extensions to a class of $f$-divergence is provided in Appendix~\ref{app:f_extension}. All proofs are presented in Appendix~\ref{app:proofs}.

\subsection{Forgetting in Two-component Mixture Model}\label{sec:mog}
In this section, we describe the minimalist mixture model, also considered by~\cite{chen2025retaining} for their empirical observations. Fix $d\in\mathbb{N}$ and a positive definite covariance matrix $\Sigma\in\R^{d\times d}$.
For $\mu\in\R^d$, denote the Gaussian density by
\[
\gauss{y}{\mu}{\Sigma}
:= (2\pi)^{-d/2}|\Sigma|^{-1/2}\exp\!\Big(-\tfrac12(y-\mu)^\top\Sigma^{-1}(y-\mu)\Big).
\]
We use a shared covariance $\Sigma$ (with bounded spectrum) for both modes so that separation and overlap are controlled purely by the means and the mixture weights. Let $p_{\mathrm{o}}(y):=\gauss{y}{\mu_{\mathrm{o}}}{\Sigma},$ and $p_{\mathrm{n}}(y):=\gauss{y}{\mu_{\mathrm{n}}}{\Sigma}$ with $\mu_{\mathrm{o}}\neq \mu_{\mathrm{n}}$. These densities represent the pre-existing (old) behavior distribution and the newly learned (new) behavior distribution, respectively. Define the Mahalanobis separation as
\[
\delta := \mnorm{\mu_{\mathrm{n}}-\mu_{\mathrm{o}}}{\Sigma^{-1}}
= \sqrt{(\mu_{\mathrm{n}}-\mu_{\mathrm{o}})^\top\Sigma^{-1}(\mu_{\mathrm{n}}-\mu_{\mathrm{o}})}.
\]
The scalar $\delta$ is dimensionless and will quantitatively govern overlap quantities (and hence misassignment and forgetting rates) through exponential-in-$\delta^2$ bounds. Fix a target mixture weight $\alpha\in(0,1)$ and define
\begin{align}\label{eq:truemodel}
p_{\alpha}(y) := \alpha\,p_{\mathrm{o}}(y) + (1-\alpha)\,p_{\mathrm{n}}(y),
\end{align}
with target responsibilities $s_{\mathrm{o}}(y):=\frac{\alpha\,p_{\mathrm{o}}(y)}{p_\alpha(y)}$ and $s_{\mathrm{n}}(y):=1-s_{\mathrm{o}}(y)$. The mixture $p_\alpha$ formalizes the desired post-training outcome: retain an $\alpha$-fraction of the old distribution while incorporating a $(1-\alpha)$-fraction of the new distribution.

We now introduce the learner model family used in post-training. For parameters $\beta\in(0,1)$ and means $m_{\mathrm{o}},m_{\mathrm{n}}\in\R^d$, define
\begin{align}\label{eq:estmodel}
q_{\beta,m_{\mathrm{o}},m_{\mathrm{n}}}(y)
:= \beta\,\gauss{y}{m_{\mathrm{o}}}{\Sigma} + (1-\beta)\,\gauss{y}{m_{\mathrm{n}}}{\Sigma}.
\end{align}
Here $\beta$ encodes how much probability mass the model allocates to the old mode, while $m_{\mathrm{o}}$ and $m_{\mathrm{n}}$ control the within-mode locations. Define the \emph{model responsibilities} (posterior component probabilities under $q$):
\begin{align}\label{eq:modelrespon}
r_{\mathrm{o}}(y)
:= \frac{\beta\,\gauss{y}{m_{\mathrm{o}}}{\Sigma}}{q_{\beta,m_{\mathrm{o}},m_{\mathrm{n}}}(y)},
\qquad r_{\mathrm{n}}(y):=1-r_{\mathrm{o}}(y).
\end{align}
The responsibilities act as soft assignments of a sample $y$ to the old versus new component and will serve as the gate through which overlap induces cross-mode gradient effects. We now introduce two notions of forgetting: (a) \emph{mass forgetting} (mass collapse of the old mixture weight) and (b) \emph{old-component drift} (distortion of the old component itself).

\begin{definition}[Mass Forgetting]\label{def:strong_forgetting}
Assume that the ``old'' mean $\mu_{\mathrm{o}}$ and the ``new'' mean $\mu_{\mathrm{n}}$ are available to the model, i.e.,
\[
q_{\beta}(y) := \beta\,p_{\mathrm{o}}(y) + (1-\beta)\,p_{\mathrm{n}}(y), \qquad \beta \in [0,1],
\]
where $\beta$ is the learnable parameter. We say that a training objective exhibits \emph{mass forgetting} if its optimal solution satisfies $\beta^\star = 0$. Minimizing such an objective therefore leads the learned model to assign zero mixture mass to the old component, even in this favorable setting. 
\end{definition}
Equivalently, the learned model reduces to $q_{\beta^\star}(y) = p_{\mathrm{n}}(y)$, so the model no longer represents the old distribution $p_{\mathrm{o}}(y)$ despite it being available in the model class. 

\begin{definition}[$\varepsilon$-Bounded Drift of the Old Component]\label{def:weak_forgetting}
Suppose that the old mean is set correctly (i.e., $m_{\mathrm{o}}=\mu_{\mathrm{o}}$). 
We say that a training objective $L(\beta,m_{\mathrm{o}},m_{\mathrm{n}})$ exhibits 
$\varepsilon$-bounded drift of the old component if
\[
\|\nabla_{m_{\mathrm{o}}}L(\beta,m_{\mathrm{o}}=\mu_{\mathrm{o}},m_{\mathrm{n}})\|\le \varepsilon,
\]
for some problem-dependent quantity $\varepsilon$ that tends to $0$ in the regime of interest.
This certifies that, at the correct old mean, the objective exerts only a small update signal on the old component, so gradient-based optimization can induce at most $\varepsilon$-scale drift of the old distribution while learning the remaining parameters.
\end{definition}
This notion captures a form of ``retention'', as opposed to forgetting captured by Definition~\ref{def:strong_forgetting}: if this definition is violated, although the mixture weight on the old component may remain nonzero, updates induced by the objective can gradually shift the parameters of the old distribution away from the true old behavior.

A desirable training objective for continual learning should therefore avoid mass forgetting in the sense of Definition~\ref{def:strong_forgetting} and, at the same time, induce only vanishingly small drift of the old distribution in the sense of Definition~\ref{def:weak_forgetting}.

\paragraph{Bhattacharyya Overlap and a Responsibility Bound.}

The Bhattacharyya coefficient~\citep{bhattacharyya1943measure} between densities $f,g$ is defined as $\mathrm{BC}(f,g) := \int_{\R^d}\sqrt{f(y)\,g(y)}\,dy\in(0,1].$

\begin{restatable}[Posterior Leakage bound via Bhattacharyya Coefficient]{lemma}{PLBC}
\label{lem:leakage}
Let $f,g$ be densities and let $w\in(0,1)$.
Define the mixture $h(y):=w f(y)+(1-w)g(y)$ and the responsibility
$r_f(y):=\frac{w f(y)}{h(y)}$.
Then
\begin{align*}
\E_{Y\sim g}\big[r_f(Y)\big]
\le \frac12\sqrt{\frac{w}{1-w}}\;\mathrm{BC}(f,g),\qquad
\E_{Y\sim f}\big[1-r_f(Y)\big]
\le \frac12\sqrt{\frac{1-w}{w}}\;\mathrm{BC}(f,g).
\end{align*}
\end{restatable}

\begin{remark}[Bhattacharyya Coefficient for Equal-covariance Gaussians]\label{lem:bc-gauss}
Let $f(y)=\gauss{y}{\mu_1}{\Sigma}$ and $g(y)=\gauss{y}{\mu_2}{\Sigma}$ with $\Sigma\succ 0$.
Then $\mathrm{BC}(f,g) = \exp\!\left(-\frac18 \mnorm{\mu_1-\mu_2}{\Sigma^{-1}}^{2}\right).$ To see this, first note that a completion-of-squares computation yields
\[
\sqrt{f(y)g(y)} = \gauss{y}{(\mu_1+\mu_2)/2}{\Sigma}\,
\exp\!\left(-\frac18 \mnorm{\mu_1-\mu_2}{\Sigma^{-1}}^{2}\right).
\]
Integrating over $y$ gives the claim.
\end{remark}

The Bhattacharyya coefficient $\mathrm{BC}(f,g)$ provides a symmetric proxy for mode overlap and converts geometric separation into quantitative bounds on posterior mis-assignment. Lemma~\ref{lem:leakage} shows that the leakage probabilities $\E_{Y\sim g}[r_f(Y)]$ and $\E_{Y\sim f}[1-r_f(Y)]$—the chances that samples from one mode receive posterior responsibility from the other in the mixture—are bounded by a simple prefactor times $\mathrm{BC}(f,g)$. In our setting, instances such as $(f,g)=(p_{\mathrm{o}},p_{\mathrm{n}})$ with $w=\beta$ control $\E_{p_{\mathrm{n}}}[r_{\mathrm{o}}(Y)]$ (new samples incorrectly assigned to the old component), while $w=\alpha$ controls $\E_{p_{\mathrm{o}}}[1-s_{\mathrm{o}}(Y)]$ (old samples attributed to the new component under the true target). These leakage terms act as the gates through which forgetting occurs. Under forward-KL training on new-only data, the logit gradient has the form $\beta-\E_{p_{\mathrm{n}}}[r_{\mathrm{o}}(Y)]$, so the exponentially small leakage term leaves a net push toward $\beta\downarrow 0$. Under reverse-KL training to a true target, old-parameter updates are proportional to misassignment probabilities such as $\E_{p_{\mathrm{o}}}[1-r_{\mathrm{o}}(Y)]$ and $\E_{p_{\mathrm{o}}}[1-s_{\mathrm{o}}(Y)]$, so drift is suppressed when overlap is small. Finally, Lemma~\ref{lem:bc-gauss} makes this explicit for equal-covariance Gaussians, giving $\mathrm{BC}(p_{\mathrm{o}},p_{\mathrm{n}})=\exp(-\delta^2/8)$ and hence exponential decay of cross-mode influence with Mahalanobis separation $\delta$.

\subsection{Forward-KL SFT Exhibits Mass Forgetting}\label{sec:kl_main}
We now analyze the behavior of the forward-KL objective in the two-mode mixture model when training is performed using only new data. The following theorem shows that, in this setting, the objective drives the mixture weight on the old mode to collapse, resulting in strong forgetting.

\begin{restatable}[Mass Forgetting in Forward-KL SFT]{theorem}{fKL}

% \begin{theorem}
\label{thm:kl_main}
Consider the target model in~\eqref{eq:truemodel} with $\mu_{\mathrm{o}}\neq\mu_{\mathrm{n}}$. Fix the learner model (see \eqref{eq:estmodel}) means at the true old/new means and define $q_\beta(y):=q_{\beta,\mu_{\mathrm{o}},\mu_{\mathrm{n}}}(y)=\beta\,p_{\mathrm{o}}(y)+(1-\beta)\,p_{\mathrm{n}}(y)
$, $\beta\in[0,1].$ Consider $L_{\mathrm{SFT}}(\beta):=\KL(p_{\mathrm{n}}\|q_\beta).$ Then:
\begin{enumerate}
\item $L_{\mathrm{SFT}}(0)=0$ and $L_{\mathrm{SFT}}(\beta)>0$ for every $\beta\in(0,1]$.
Hence $\beta=0$ is the unique global minimizer over $[0,1]$.
\item $L_{\mathrm{SFT}}$ is strictly increasing on $[0,1]$.
\item Let $\phi\in\R$ be a logit with $\beta=\sigma(\phi)$ and consider gradient flow
\[
\dot\phi(t)=-\frac{d}{d\phi}L_{\mathrm{SFT}}(\sigma(\phi(t))).
\]
Then $\phi(t)$ is strictly decreasing and $\beta(t)=\sigma(\phi(t))$ satisfies $\beta(t)\downarrow 0$ as $t\to\infty$.
Moreover, recalling the model responsibility in~\eqref{eq:modelrespon}, 
%\[
%r_{\mathrm{o}}(y):=\frac{\beta\,p_{\mathrm{o}}(y)}{q_\beta(y)},
%\]
the logit gradient has the explicit form
\begin{equation}\label{eq:kl_sft_logit_grad}
\frac{d}{d\phi}L_{\mathrm{SFT}}(\sigma(\phi))=\beta-\E_{Y\sim p_{\mathrm{n}}}\big[r_{\mathrm{o}}(Y)\big],
\qquad\text{so}\qquad
\dot\phi(t)=\E_{p_{\mathrm{n}}}[r_{\mathrm{o}}(Y)]-\beta(t)<0.
\end{equation}
Finally, the leakage bound (Lemma~\ref{lem:leakage} + Remark~\ref{lem:bc-gauss}) yields
\begin{equation}\label{eq:kl_sft_leak_bound}
\E_{Y\sim p_{\mathrm{n}}}\big[r_{\mathrm{o}}(Y)\big]
\le \frac12\sqrt{\frac{\beta}{1-\beta}}\exp\!\left(-\frac{\delta^2}{8}\right),
\end{equation}
and hence $\dot\phi(t)\le \tfrac12\sqrt{\frac{\beta(t)}{1-\beta(t)}}e^{-\delta^2/8}-\beta(t)$.
\end{enumerate}
\end{restatable}

\begin{remark}
We emphasize that except for the explicit bound in~\eqref{eq:kl_sft_leak_bound}, the above result does not rely on the Gaussian assumption; it holds for any pair of distinct densities $p_{\mathrm{o}}$ and $p_{\mathrm{n}}$.
\end{remark}

The above result shows a strong form of forgetting: when the forward-KL objective is optimized on new-only data, the optimal mixture weight places zero mass on the old mode. The mechanism is transparent in the logit-gradient expression \eqref{eq:kl_sft_logit_grad}: the update compares the current mass $\beta$ on the old mode with the probability that new data are assigned to that mode, and since this assignment probability is exponentially small in the separation $\delta$, the update consistently pushes $\beta$ downward. Thus even though the model family explicitly contains the correct old component, the forward-KL objective provides no incentive to retain it once the training distribution contains only new data, illustrating a strong form of forgetting.

\subsubsection{When Does Replay Prevent Forgetting under Forward-KL SFT?}\label{subsec:replay_forwardKL_consolidated}

Theorem~\ref{thm:kl_main} established a \emph{population-level} strong-forgetting phenomenon for forward-KL SFT:
when the training distribution is new-only ($p_{\mathrm{n}}$), the forward-KL objective
$L_{\mathrm{SFT}}(\beta)=\KL(p_{\mathrm{n}}\|q_\beta)$ uniquely minimizes at $\beta^\star=0$ even though the model class contains the correct
old component $p_{\mathrm{o}}$.
A natural next question is whether replay-style modifications can mitigate this collapse.

The key structural point is that forward-KL has the form $\KL(\text{data}\,\|\,\text{model})$, so its population optimizer is determined by the
\emph{data distribution}.
Consequently, replay can only prevent strong forgetting at the population level if it changes the effective training distribution.
We therefore distinguish two canonical interventions:
(i) \emph{denominator replay}, which mixes the old mode into the \emph{model side} while keeping the data distribution new-only, and
(ii) \emph{numerator replay}, which mixes old samples into the \emph{data side}.
Only (ii) changes the population optimum.

%In this subsection we fix both component densities and optimize only the mixture mass:
%\[
%q_\beta(y):=\beta\,p_{\mathrm{o}}(y)+(1-\beta)\,p_{\mathrm{n}}(y),\qquad \beta\in[0,1].
%\]

\begin{restatable}{lemma}{ReplayFKL}
\label{lem:replay_forwardKL_consolidated}
Fix $\lambda\in(0,1)$.

\medskip
\noindent\textbf{(A)}% Denominator (model-side) replay does not prevent $\beta$-collapse.}
Define the replay-mixed learner \emph{model} as $\tilde q_{\beta,\lambda}(y):=(1-\lambda)\,q_\beta(y)+\lambda\,p_{\mathrm{o}}(y).$ Then $\tilde q_{\beta,\lambda}=q_{\tilde\beta}$ with $\tilde\beta=\lambda+(1-\lambda)\beta\in[\lambda,1],$ and the optimization problem $\min_{\beta\in[0,1]}\KL\!\big(p_{\mathrm{n}}\,\|\,\tilde q_{\beta,\lambda}\big)$ has the unique minimizer $\beta^\star=0$ (equivalently $\tilde\beta^\star=\lambda$).

\medskip
\noindent\textbf{(B)}
Define the replay-mixed \emph{data} distribution $\tilde p_\lambda(y):=(1-\lambda)\,p_{\mathrm{n}}(y)+\lambda\,p_{\mathrm{o}}(y).$ Then $\min_{\beta\in[0,1]}\KL\!\big(\tilde p_\lambda\,\|\,q_\beta\big)$ has the unique minimizer $\beta^\star=\lambda$, and the minimum value is $0$.
\end{restatable}

%\begin{remark}[Forgetting takeaway for replay under forward-KL]\label{rem:replay_forwardKL_takeaway}
%Denominator replay (model-side mixing) does not change the forward-KL population incentive: as long as the data distribution is new-only($p_{\mathrm{n}}$), the optimal \emph{trainable} parameter still collapses to $\beta^\star=0$. Any apparent ``retention'' is externally imposed by the forced mixture $\tilde q_{\beta,\lambda}$, not learned.In contrast, numerator replay (data-side mixing) changes the population objective by changing the training distribution: forward-KL then recovers the replay fraction exactly ($\beta^\star=\lambda$), giving the minimal rehearsal-based mechanism by which SFT avoids forgetting in this toy model.
%\end{remark}

%\paragraph{Explanation of Part (A): why denominator replay cannot change the learned optimum.}
Part (A) shows that denominator replay does not change the \emph{directional preference} of forward-KL; it only restricts the model family.
Indeed, $\tilde q_{\beta,\lambda}$ is not a new family of distributions: it is exactly the original mixture family with an \emph{affine reparameterization}
$\tilde\beta=\lambda+(1-\lambda)\beta$ and hence a constraint $\tilde\beta\in[\lambda,1]$.
Therefore the optimization reduces to minimizing $\KL(p_{\mathrm{n}}\|q_{\tilde\beta})$ over $\tilde\beta\in[\lambda,1]$.
Since the map $\tilde\beta\mapsto \KL(p_{\mathrm{n}}\|q_{\tilde\beta})$ is strictly increasing (Theorem~\ref{thm:kl_main}),
the optimizer chooses the smallest admissible old mass, namely $\tilde\beta^\star=\lambda$, which corresponds to $\beta^\star=0$.
Operationally, any nonzero old mass present in the deployed distribution $\tilde q_{\beta^\star,\lambda}=q_\lambda$ is therefore a \emph{hard-coded floor}
imposed by the replay mixing, rather than something learned from new-only data.

%\paragraph{Explanation of Part (B): why numerator replay changes the population target.}
Part (B) is qualitatively different because numerator replay changes the training distribution itself.
When the data distribution is $\tilde p_\lambda=(1-\lambda)p_{\mathrm{n}}+\lambda p_{\mathrm{o}}$, the model class contains a member that matches it
\emph{exactly}, namely $q_{\beta}$ with $\beta=\lambda$.
Thus the forward-KL can attain its global minimum value $0$ at $\beta^\star=\lambda$, and uniqueness follows because two distinct mixture weights
cannot represent the same density when $p_{\mathrm{o}}\not\equiv p_{\mathrm{n}}$.
In other words, forward-KL performs a mode-covering projection of the \emph{data} mixture onto the model family, and it necessarily retains
whatever fraction of old data is present in the training distribution.

\subsection{Reverse-KL RL Avoids Mass Forgetting and Controls Old-Component Drift}\label{sec:reverseKL} 

We now verify that the reverse-KL objective is correctly aligned with the intended true target at the parameter level. %The first result computes the gradients of the objective with respect to the mixture weight and the new-mode mean, and shows that when these parameters match the target mixture, both gradients vanish. Consequently, the parameter pair corresponding to the target distribution is not only stationary but also a global minimizer, achieving zero divergence.
We first show that the reverse-KL objective is \emph{consistent} with respect to the target distribution. In particular, when the learner parameters match the target mixture, i.e., when the mixture weight equals the target weight and the new-mode mean equals the true new mean, $(\beta,m_{\mathrm{n}})=(\alpha,\mu_{\mathrm{n}})$, then the model distribution $q_{\beta,m_{\mathrm{n}}}$ coincides exactly with the target $p_\alpha$. At this point the reverse-KL divergence vanishes and the gradients with respect to both parameters are zero, so the point is stationary. Moreover, since KL divergence is nonnegative and equals zero only when the two distributions coincide, this parameter choice is in fact a \emph{global minimizer} of the objective. Thus reverse-KL RL is aligned with the target at the population level: the optimal solution preserves the correct mixture mass on the old mode and therefore avoids mass forgetting.

\begin{restatable}[Consistency of Reverse-KL RL]{theorem}{RKLN}
\label{thm:kl_stationary_beta_mn} 
Consider the target model in~\eqref{eq:truemodel} and the learner  model family in~\eqref{eq:estmodel} with $m_{\mathrm{o}}=\mu_{\mathrm{o}}$.
%Fix $\Sigma\succ 0$, $\alpha\in(0,1)$, and means $\mu_{\mathrm{o}},\mu_{\mathrm{n}}\in\mathbb{R}^d$. Set  and define, for $\beta\in(0,1)$ and $m_{\mathrm{n}}\in\mathbb{R}^d$, 
%\[ 
%q_{\beta,m_{\mathrm{n}}}(y) := q_{\beta,\mu_{\mathrm{o}},m_{\mathrm{n}}}(y) = \beta\,\gauss{y}{\mu_{\mathrm{o}}}{\Sigma}+(1-\beta)\,\gauss{y}{m_{\mathrm{n}}}{\Sigma}, \qquad p_\alpha(y) := \alpha\,\gauss{y}{\mu_{\mathrm{o}}}{\Sigma}+(1-\alpha)\,\gauss{y}{\mu_{\mathrm{n}}}{\Sigma}.
%\] 
Let 
\[ 
L(\beta,m_{\mathrm{n}}) := \KL\!\big(q_{\beta,m_{\mathrm{n}}}\,\|\,p_\alpha\big) = \int_{\mathbb{R}^d} q_{\beta,m_{\mathrm{n}}}(y)\,\log\frac{q_{\beta,m_{\mathrm{n}}}(y)}{p_\alpha(y)}\,dy. 
\] 
Then $L$ is continuously differentiable on $(0,1)\times\mathbb{R}^d$, and its gradients are 
\begin{align} 
\frac{\partial}{\partial\beta}L(\beta,m_{\mathrm{n}}) &= \int_{\mathbb{R}^d}\Big(\gauss{y}{\mu_{\mathrm{o}}}{\Sigma}-\gauss{y}{m_{\mathrm{n}}}{\Sigma}\Big)\, \log\frac{q_{\beta,m_{\mathrm{n}}}(y)}{p_\alpha(y)}\,dy,\label{eq:beta-grad-kl}\\[3pt] \nabla_{m_{\mathrm{n}}}L(\beta,m_{\mathrm{n}}) &= (1-\beta)\int_{\mathbb{R}^d}\gauss{y}{m_{\mathrm{n}}}{\Sigma}\,\Sigma^{-1}(y-m_{\mathrm{n}})\, \log\frac{q_{\beta,m_{\mathrm{n}}}(y)}{p_\alpha(y)}\,dy.\label{eq:mn-grad-kl} 
\end{align} 
In particular, the point $(\beta,m_{\mathrm{n}})=(\alpha,\mu_{\mathrm{n}})$ is stationary: 
\[ 
\left.\frac{\partial}{\partial\beta}L(\beta,m_{\mathrm{n}})\right|_{(\beta,m_{\mathrm{n}})=(\alpha,\mu_{\mathrm{n}})}=0, \qquad \left.\nabla_{m_{\mathrm{n}}}L(\beta,m_{\mathrm{n}})\right|_{(\beta,m_{\mathrm{n}})=(\alpha,\mu_{\mathrm{n}})}=0. 
\] 
Moreover, since $\KL(\cdot\|\cdot)\ge 0$ with equality iff the two densities coincide a.e., $(\alpha,\mu_{\mathrm{n}})$ is a global minimizer of $L$ and achieves $L(\alpha,\mu_{\mathrm{n}})=0$. 
\end{restatable} 

We next bound the drift of the old distribution's parameters when learning the new distribution.

\begin{restatable}[Bounding Drift of the Old Distribution in Reverse-KL RL]{theorem}{RKL}
\label{lem:fwdKL_oldmean_drift}
Consider the learner model family in~\eqref{eq:estmodel} and the reverse-KL objective $
L_{\mathrm{RL}}(\beta,m_{\mathrm{o}},m_{\mathrm{n}}):=\KL\!\big(q_{\beta,m_{\mathrm{o}},m_{\mathrm{n}}}\,\|\,p_\alpha\big)$. Assume the old mean is already correct: $m_{\mathrm{o}}=\mu_{\mathrm{o}}$ (with $m_{\mathrm{n}}$ arbitrary).
Recalling~\eqref{eq:modelrespon}, let $r_{\mathrm{o}}(y)$ be the learner responsibility under $q_{\beta,\mu_{\mathrm{o}},m_{\mathrm{n}}}$ and let $s_{\mathrm{o}}(y)$ be the target responsibility under $p_\alpha$. Then the gradient w.r.t.\ the old mean admits the \emph{exact} decomposition
\begin{equation}\label{eq:kl_oldmean_exact}
\nabla_{m_{\mathrm{o}}}L_{\mathrm{RL}}(\beta,\mu_{\mathrm{o}},m_{\mathrm{n}})
=
\beta\,\Sigma^{-1}\Big(
\varepsilon_q(\beta,m_{\mathrm{n}})\,(m_{\mathrm{n}}-\mu_{\mathrm{o}})
-
\varepsilon_p(\alpha)\,(\mu_{\mathrm{n}}-\mu_{\mathrm{o}})
\Big),
\end{equation}
where the misassignment probabilities are $\varepsilon_q(\beta,m_{\mathrm{n}}):=\E_{Y\sim p_{\mathrm{o}}}\big[1-r_{\mathrm{o}}(Y)\big],$ and $\varepsilon_p(\alpha):=\E_{Y\sim p_{\mathrm{o}}}\big[1-s_{\mathrm{o}}(Y)\big].$ Moreover, these satisfy the explicit Gaussian overlap bounds
\begin{equation}\label{eq:kl_eps_bounds}
\varepsilon_q(\beta,m_{\mathrm{n}})
\le \frac12\sqrt{\frac{1-\beta}{\beta}}
\exp\!\left(-\frac18\mnorm{m_{\mathrm{n}}-\mu_{\mathrm{o}}}{\Sigma^{-1}}^{2}\right),
\qquad
\varepsilon_p(\alpha)
\le \frac12\sqrt{\frac{1-\alpha}{\alpha}}
\exp\!\left(-\frac{\delta^{2}}{8}\right),
\end{equation}
and hence (with $\|\cdot\|_2$ denoting the operator norm), we have
\[
\norm{\nabla_{m_{\mathrm{o}}}L_{\mathrm{RL}}(\beta,\mu_{\mathrm{o}},m_{\mathrm{n}})}
\le
\beta\,\norm{\Sigma^{-1}}_{2}\left(
\varepsilon_q(\beta,m_{\mathrm{n}})\,\norm{m_{\mathrm{n}}-\mu_{\mathrm{o}}}
+
\varepsilon_p(\alpha)\,\norm{\mu_{\mathrm{n}}-\mu_{\mathrm{o}}}
\right).
\]
\end{restatable}

The above result shows that when optimizing reverse-KL toward a target that explicitly retains the old mode, the learning signal acting on the old parameters is tightly controlled. In particular, when the old mean is already correct, the gradient admits an \emph{exact decomposition} in which the only terms capable of moving the old mode arise from misassignment probabilities $(1-r_{\mathrm{o}})$ and $(1-s_{\mathrm{o}})$ under old-mode samples. The theorem further shows that these misassignment probabilities are governed by Gaussian overlap quantities (Bhattacharyya coefficients) and decay exponentially with the squared Mahalanobis separation, scaling as $\exp(-\delta^{2}/8)$ in the equal-covariance case. Consequently, in well-separated regimes the reverse-KL objective exerts only an exponentially small update pressure on an already-correct old component, so optimization can adjust the new mode while perturbing the old mode only through negligible overlap effects.

Theorems~\ref{thm:kl_stationary_beta_mn} and~\ref{lem:fwdKL_oldmean_drift}  together quantify precisely that the reverse-KL objective is naturally aligned with the desired continual-learning outcome: it favors a solution that simultaneously preserves the old mode and correctly represents the new one and the mixing proportions.

\subsubsection{A Local Exponential Rate for Reverse-KL Minimization} 

In this subsection we prove a \emph{local} (near-optimum) exponential convergence rate for the reverse-KL minimization when $m_{\mathrm{o}}$ is fixed at $\mu_{\mathrm{o}}$ and we optimize over $(\beta,m_{\mathrm{n}})$ by gradient flow. Because $\beta$ is constrained to $[0,1]$, we parameterize $\beta=\sigma(\phi)$ using the logit $\phi\in\mathbb{R}$. 

%\todo[inline]{Proof of this Theorem in Appendix has  statements not in this theorem}

\begin{restatable}[Explicit Local PL bound and Exponential Convergence for Reverse-KL Gradient Flow]{theorem}{PL}\label{thm:kl_RL_local_rate}
Consider the reverse-KL objective $L(\phi,m):=\KL(q_{\phi,m}\|p_\alpha),$ and where
\[
q_{\phi,m}(y):=\beta(\phi)\,\phi_\Sigma(y;\mu_{\mathrm{o}})+(1-\beta(\phi))\,\phi_\Sigma(y;m),
\qquad
\beta(\phi)=\frac{1}{1+e^{-\phi}},
\]
with \(m_{\mathrm{o}}=\mu_{\mathrm{o}}\) fixed, and let
\[
\theta:=(\phi,m)\in\R^{d+1},
\qquad
\theta^\star:=(\phi^\star,m^\star),
\qquad
\phi^\star=\log\frac{\alpha}{1-\alpha},
\qquad
m^\star=\mu_{\mathrm{n}}.
\]
Then $H_\star:=\nabla^2L(\theta^\star)\succ 0$ and $\mu_\star:=\lambda_{\min}(H_\star)>0.$ Fix \(r_0>0\) and let \(K=B_{r_0}(\theta^\star)\).
Let \(L_H:=L_H(K)\) be the Hessian-Lipschitz constant given by Lemma~\ref{lem:hessian_lipschitz_reverseKL}.
Define
\[
\rho:=\min\!\left\{r_0,\frac{\mu_\star}{2L_H}\right\},
\qquad
\varepsilon_{\mathrm{loc}}:=\frac{\mu_\star}{8}\rho^2.
\]
Then:
\begin{enumerate}
\item For every \(\theta\in B_\rho(\theta^\star)\),
\begin{equation}\label{eq:local_hessian_lower}
\nabla^2L(\theta)\succeq \frac{\mu_\star}{2}I.
\end{equation}
Consequently, for every \(\theta\in B_\rho(\theta^\star)\),
\begin{equation}\label{eq:local_PL_final}
\|\nabla L(\theta)\|^2\ge \mu_\star\,L(\theta),
\qquad
L(\theta)\ge \frac{\mu_\star}{4}\|\theta-\theta^\star\|^2.
\end{equation}

\item Let \(\theta(t)\) solve the Euclidean gradient flow $\dot\theta(t)=-\nabla L(\theta(t)).$ If $\theta(0)\in B_\rho(\theta^\star)$ and $L(\theta(0))\le \varepsilon_{\mathrm{loc}},
$, then \(\theta(t)\in B_\rho(\theta^\star)\) for all \(t\ge 0\), and
\begin{equation}\label{eq:exp_loss_final}
L(\theta(t))\le L(\theta(0))\,e^{-\mu_\star t}
\qquad
\forall t\ge 0.
\end{equation}
Moreover,
\begin{equation}\label{eq:exp_param_final}
\|\theta(t)-\theta^\star\|
\le
\frac{2}{\sqrt{\mu_\star}}\sqrt{L(\theta(0))}\,e^{-\mu_\star t/2}
\qquad
\forall t\ge 0.
\end{equation}
\end{enumerate}
\end{restatable}

Theorem~\ref{thm:kl_RL_local_rate} says that once the reverse-KL dynamics enter a neighborhood in which the local curvature stays uniformly positive, the loss behaves like a strongly convex function and gradient flow converges exponentially fast to the optimum. The locality is quantified explicitly by \(\rho\) and \(\varepsilon_{\mathrm{loc}}\): the radius is determined by the curvature at the optimum \(\mu_\star\) and by the local Hessian-Lipschitz constant \(L_H\), while the admissible sublevel threshold is \(\varepsilon_{\mathrm{loc}}=\mu_\star\rho^2/8\). Thus the same Hessian curvature quantity that measures local identifiability also determines the exponential convergence rate and the size of the certified local basin.

\subsection{Replay Improves Reverse-KL Methods in Practice}\label{subsec:replay_reverseKL}

Our earlier population-level reverse-KL analysis suggests that KL-regularized on-policy updates are naturally aligned with retention: when the true target $p_\alpha$ explicitly includes the old mode, reverse-KL gradients that would move an already-correct old component are overlap-gated and become exponentially small as the separation between modes increases. However, this population-level picture does \emph{not} by itself rule out a purely \emph{stochastic} failure mode. In particular, when the current old-mode weight $\beta$ is small, a minibatch drawn from the current model $q_{\beta,m_{\mathrm{n}}}$ may contain no old-mode samples, causing the stochastic update in that iteration to behave effectively as a ``new-only'' update.

A simple semi-on-policy remedy is to inject a small replay fraction of old-mode samples into the behavior distribution used to construct minibatches. Importantly, this can be done while still estimating the \emph{same} reverse-KL gradient in expectation via \emph{bounded} importance weighting. This intervention simultaneously (i) guarantees that old-mode samples appear in minibatches regardless of the current value of $\beta$, and (ii) avoids the high-variance pathologies typically associated with general off-policy importance sampling, since the resulting importance ratios remain uniformly bounded.

\paragraph{Replay-Mixed Sampling with Bounded Importance Weights.}
Fix a replay rate $\lambda\in(0,1)$ and define
\begin{equation}\label{eq:replay_behavior_def_consolidated}
b_{\lambda,\beta,m_{\mathrm{n}}}(y)
:=(1-\lambda)\,q_{\beta,m_{\mathrm{n}}}(y)+\lambda\,p_{\mathrm{o}}(y),
\qquad
w_\lambda(y)
:=\frac{q_{\beta,m_{\mathrm{n}}}(y)}{b_{\lambda,\beta,m_{\mathrm{n}}}(y)}.
\end{equation}
A minibatch is formed by drawing $Y_1,\dots,Y_N\overset{i.i.d.}{\sim}b_{\lambda,\beta,m_{\mathrm{n}}}$.

\begin{restatable}{lemma}{ReplayRKL}
\label{lem:replay_reverseKL_consolidated}
Fix $\lambda\in(0,1)$, $\beta\in(0,1)$, and $m_{\mathrm{n}}\in\R^d$.  Let $b_{\lambda,\beta,m_{\mathrm{n}}}$ and $w_\lambda$ be defined in
\eqref{eq:replay_behavior_def_consolidated}, and let $Y_1,\dots,Y_N\overset{i.i.d.}{\sim}b_{\lambda,\beta,m_{\mathrm{n}}}$.

\medskip
\noindent\textbf{(A)} There exists $\tilde\beta\in(0,1)$ such that
\[
b_{\lambda,\beta,m_{\mathrm{n}}}(y)=q_{\tilde\beta,m_{\mathrm{n}}}(y),
\qquad
\tilde\beta=\lambda+(1-\lambda)\beta\ \ge\ \lambda.
\]
Moreover, the importance ratio is uniformly bounded,
\[
0\le w_\lambda(y)\le \frac{1}{1-\lambda}\qquad\forall y\in\R^d.
\]
Consequently, for any integrable $h:\R^d\to\R^k$,
\begin{equation}\label{eq:unbiased_general_consolidated}
\E_{Y\sim b_{\lambda,\beta,m_{\mathrm{n}}}}\big[w_\lambda(Y)\,h(Y)\big]
=\E_{Y\sim q_{\beta,m_{\mathrm{n}}}}\big[h(Y)\big],
\end{equation}
and if additionally $\E_{b_{\lambda,\beta,m_{\mathrm{n}}}}[\|h(Y)\|^2]<\infty$, then
\begin{equation}\label{eq:second_moment_bound_consolidated}
\E_{Y\sim b_{\lambda,\beta,m_{\mathrm{n}}}}\big[\|w_\lambda(Y)\,h(Y)\|^2\big]
\le \frac{1}{(1-\lambda)^2}\,\E_{Y\sim b_{\lambda,\beta,m_{\mathrm{n}}}}\big[\|h(Y)\|^2\big].
\end{equation}

\medskip
\noindent\textbf{(B)} Under the mixture generative model of $b_{\lambda,\beta,m_{\mathrm{n}}}=q_{\tilde\beta,m_{\mathrm{n}}}$, let $Z_i\in\{0,1\}$ be the latent indicator
that $Y_i$ came from the old component.  Then $Z_i\sim\mathrm{Bernoulli}(\tilde\beta)$ with $\tilde\beta\ge\lambda$, hence
\begin{equation}\label{eq:no_old_prob_consolidated}
\Pr\Big(\sum_{i=1}^N Z_i=0\Big)=(1-\tilde\beta)^N=((1-\lambda)(1-\beta))^N\le (1-\lambda)^N,
\end{equation}
and a multiplicative Chernoff bound implies
\begin{equation}\label{eq:chernoff_oldcount_consolidated}
\Pr\Big(\sum_{i=1}^N Z_i \le \tfrac{\lambda}{2}N\Big)\ \le\ \exp\!\Big(-\tfrac{\lambda}{8}N\Big).
\end{equation}
\end{restatable}

Part (A) formalizes that replay-mixing does \emph{not} change the reverse-KL population objective: any quantity that can be expressed as an
expectation under the current model $q_{\beta,m_{\mathrm{n}}}$ (in particular, standard score-form expressions for
$\nabla_\theta\KL(q_{\beta,m_{\mathrm{n}}}\|p_\alpha)$) can be estimated \emph{unbiasedly} from replay-mixed samples by weighting with $w_\lambda$.
At the same time, the correction is benign: the pointwise dominance
$b_{\lambda,\beta,m_{\mathrm{n}}}\ge (1-\lambda)q_{\beta,m_{\mathrm{n}}}$ forces a uniform weight bound $w_\lambda\le (1-\lambda)^{-1}$,
which immediately yields the second-moment control \eqref{eq:second_moment_bound_consolidated} and prevents the high-variance failure modes of
general importance sampling.

Part (B) addresses the stochastic ``old-mode starvation'' pathology directly.
Even if $\beta$ is tiny, replay ensures the effective old-component probability under the behavior distribution satisfies
$\tilde\beta\ge\lambda$, so the probability that a minibatch contains \emph{no} old samples decays as $(1-\lambda)^N$, independent of $\beta$.
Thus replay-mixing decouples \emph{old-mode visibility} in stochastic gradients from the current policy's old weight.
Combined with our earlier overlap-gated reverse-KL gradient identities (which show that, at the population level, the old mode is only weakly
perturbed when it is already correct), this yields a minimalistic mechanism by which a method can be ``not exactly on-policy'' yet still exhibit
on-policy-like resistance to catastrophic forgetting in finite-batch optimization.

Recent work of KL-regularized on-policy RL~\citep{shah2025comedy,tang2025few} emphasize that \emph{how} the KL term is estimated and differentiated matters: common
surrogate/stop-gradient constructions can yield biased gradients or even optimize a qualitatively different objective than intended.
Lemma~\ref{lem:replay_reverseKL_consolidated} is complementary to these estimator-correctness issues.  Even if one uses a \emph{correct}
(on-policy) gradient expression of the form $\nabla_\theta \mathcal{J}(\theta)=\E_{Y\sim q_\theta}[h_\theta(Y)]$, minibatches drawn purely from
$q_\theta$ can still suffer a coverage problem: low-probability regions (e.g.\ previously learned behavior/modes) may be absent, making the step behave
effectively like ``new-only'' optimization.  Replay-mixing fixes this by sampling from
$b_{\lambda,\theta}=(1-\lambda)q_\theta+\lambda p_{\mathrm{o}}$, and using the importance ratio
$w_\lambda=q_\theta/b_{\lambda,\theta}$ to recover unbiased $q_\theta$-expectations, while keeping $w_\lambda\le (1-\lambda)^{-1}$ uniformly.
Thus, replay-mixing provides a principled way to reuse replay/stale samples and stabilize gradient estimation \emph{without} the unbounded importance weights
that typically make off-policy correction high-variance. It can alleviate coverage/variance pathologies, while the KL-estimator bias mechanisms studied in
\cite{shah2025comedy,tang2025few} must still be addressed by choosing an appropriate KL-gradient estimator.

\subsection{Summarizing Consequences for SFT and RL Post-Training}

In our two-mode Gaussian mixture model, SFT corresponds to minimizing a forward-KL objective of the form $\KL(\text{data}\,\|\,\text{model})$. When the training distribution is new-only ($p_{\mathrm{n}}$), this objective exhibits a population-level pathology: although the model family explicitly contains the correct old component $p_{\mathrm{o}}$, the unique minimizer of $L_{\mathrm{SFT}}(\beta)=\KL(p_{\mathrm{n}}\|q_\beta)$ collapses the old mixture weight to $\beta^\star=0$. The logit-gradient formula $\dot\phi=\E_{p_{\mathrm{n}}}[r_{\mathrm{o}}(Y)]-\beta$ makes the mechanism transparent: updates compare the current old mass $\beta$ with the frequency with which new data are (mis)assigned to the old component, and this leakage probability is exponentially small in the separation $\delta$. Replay interacts with SFT in an asymmetric manner. Mixing old samples only on the model side (denominator replay) does not alter the population objective and therefore still drives the trainable parameter to $\beta^\star=0$, with any apparent ``retention'' arising solely from the externally imposed floor $\tilde\beta\ge\lambda$. In contrast, incorporating replay into the data distribution (numerator replay) genuinely changes the forward-KL objective: the population optimum then recovers the replay fraction exactly ($\beta^\star=\lambda$). In this sense, numerator replay provides the minimal rehearsal mechanism by which SFT can avoid strong forgetting in this toy setting.

By contrast, RL-style post-training with KL regularization can be naturally expressed as minimizing a reverse-KL objective $\KL(q\,\|\,p_\alpha)$ toward a true target
$p_\alpha=\alpha p_{\mathrm{o}}+(1-\alpha)p_{\mathrm{n}}$ that explicitly preserves the old behavior. In this setting the learned parameters are aligned with retention: the matching parameters $(\beta,m_{\mathrm{n}})=(\alpha,\mu_{\mathrm{n}})$ are stationary and globally minimize the reverse-KL, while the learning signal for the old mean at $m_{\mathrm{o}}=\mu_{\mathrm{o}}$ is provably gated by misassignment probabilities that decay exponentially with mode separation. This formalizes the mode-locality intuition: reverse-KL updates can adapt the new mode while perturbing an already-correct old mode only through exponentially small overlap regions. In practice, however, purely on-policy stochastic optimization can suffer a finite-batch starvation failure when $\beta$ becomes small, since minibatches sampled from the current model may contain no old-mode samples. A semi-on-policy replay mixture resolves this issue without altering the population objective. Specifically, mixing a small fraction $\lambda$ of true old samples into the behavior distribution and applying bounded importance weights yields an unbiased estimator of the same reverse-KL gradient, while guaranteeing $\Omega(\lambda N)$ old-mode samples per minibatch with high probability.

\section{Forgetting in Near-on-policy Algorithms}\label{sec:near_on_policy_alternatives}

Based on the quantification established in Section~\ref{sec:main_results}, we now study forgetting in three recently proposed algorithms: Self-Distillation Fine-Tuning (SDFT)~\citep{shenfeld2026self}, TTT-Discover~\citep{yuksekgonul2026learning}, and OAPL~\citep{ritter2026llms,brantley2025accelerating}. All three are \emph{near on-policy} in the sense that their update targets are constructed from distributions closely tied to the model’s own behavior, but they do so in different ways: SDFT distills the current policy toward an evolving teacher induced by demonstrations, TTT-Discover reweights samples drawn from the current policy using an entropic reward objective together with a KL anchor, and OAPL samples from a lagged reference policy and defines both its improvement target and regression objective relative to that same frozen reference. In this sense, demonstrations (via distillation) and rewards (via RL-style tilting) are two practical mechanisms for constructing on-policy update targets that keep training supported on the model’s current distribution. The following subsections show that these design choices lead to distinct forgetting profiles in our mixture model: SDFT behaves most like a reverse-KL update with an evolving teacher, TTT-Discover balances reward-driven discovery against KL-based retention, and OAPL preserves or reweights only the modes already present in its frozen reference while remaining geometrically local.

%Based on the quantification established in Section~\ref{sec:main_results}, we now study forgetting in two recently proposed algorithms, namely, TTT-Discover~\citep{yuksekgonul2026learning} and OAPL~\citep{ritter2026llms,brantley2025accelerating}. Both methods are \emph{near on-policy} in the sense that their update targets are constructed from distributions closely tied to the model’s own behavior, but they achieve this in different ways: TTT-Discover reweights samples drawn from the current policy using an entropic reward objective together with a KL anchor, whereas OAPL samples from a lagged reference policy and defines both its improvement target and regression objective relative to that same frozen reference. The following subsections below show that these design choices lead to distinct forgetting profiles in our mixture model.

\subsection{Mixture-Model Analysis of SDFT with Demonstrations and EMA Teachers}\label{sec:sdft_demo_ema_mixture}

Self-Distillation Fine-Tuning (SDFT)~\citep{shenfeld2026self} updates a \emph{student} policy on-policy by sampling from the student and minimizing a reverse-KL objective to a \emph{demonstration-conditioned} teacher. In the notation of~\cite{shenfeld2026self}, the student is \(\pi_{\theta_t}(\cdot\mid x)\), while the teacher is the same model additionally conditioned on an expert demonstration \(c\), yielding \(\pi_{\bar\theta_t}(\cdot\mid x,c)\). Thus each step minimizes
\[
D_{\mathrm{KL}}\!\big(\pi_{\theta_t}(\cdot\mid x)\,\|\,\pi_{\bar\theta_t}(\cdot\mid x,c)\big),
\]
taking gradients only with respect to the student parameters \(\theta_t\) while treating the teacher as fixed within that step. The key role of the demonstration \(c\) is to \emph{bias the teacher distribution toward a desirable behavior}: conditioning on \(c\) changes how much probability mass the teacher assigns to ``old'' behavior versus ``new'' behavior, and it shifts the teacher's preferred new behavior in a direction suggested by the demonstration. Across steps, \cite{shenfeld2026self} proposes an \emph{EMA teacher} (updating \(\bar\theta_t\) by exponential moving average of \(\theta_t\)) to stabilize the on-policy feedback loop and avoid chasing a rapidly moving target.

\paragraph{Mixture Abstraction of Demonstration-Guided SDFT.}
We model the SDFT process in our two-mode Gaussian mixture abstraction from Section~\ref{sec:mog}. As before, the \emph{old} behavior is represented by the fixed Gaussian $p_{\mathrm{o}}(y):=\gauss{y}{\mu_{\mathrm{o}}}{\Sigma}$. At each iteration \(t\), the teacher distribution is summarized by a \emph{teacher state}
$
\widetilde{\nu}_t = (\alpha_t,\nu_t)\in(0,1)\times\R^d$,
where \(\alpha_t\in(0,1)\)  denotes the amount of old-mode mass retained by the teacher and 
\(\nu_t\) is the current location of the teacher's new mode. 
Similarly, the student distribution is summarized by the \emph{student state}
$\widetilde{m}_t = (\beta_t,m_t)\in(0,1)\times\R^d$,
where \(\beta_t\) is the student's old-mode mass and \(m_t\) is the student's new-mode location. 
The per-step SDFT objective is then reverse-KL from student to the current teacher mixture.

To model the fact that the demonstration \(c\) provides a \emph{directional signal} (a ``vector closer to truth''), we introduce a fixed
\emph{demonstration anchor state}
\[
\widetilde{\nu}(c):=(\alpha_c,\nu_c)\in(0,1)\times\R^d,
\]
which should be interpreted as the mixture summary of the demonstration-conditioned teacher preference induced by \(c\).
We allow the teacher to be updated both by EMA smoothing toward the student and by an explicit pull toward \(\widetilde{\nu}(c)\).
A single scalar \(\lambda\in[0,1]\) controls the \emph{demonstrator strength}: \(\lambda=0\) means the teacher is purely an EMA of the student,
while larger \(\lambda\) makes the teacher track the demonstration anchor more strongly; see also Remark~\ref{rem:sdftassumption}.

\begin{definition}[EMA+demonstrator SDFT dynamics in the mixture model]\label{def:sdft_demo_ema}
Fix \(\Sigma\succ 0\) and \(p_{\mathrm{o}}(y)=\gauss{y}{\mu_{\mathrm{o}}}{\Sigma}\).
For teacher parameters \((\alpha,\nu)\in(0,1)\times\R^d\) define
\[
p_{\alpha,\nu}(y):=\alpha\,p_{\mathrm{o}}(y)+(1-\alpha)\,\gauss{y}{\nu}{\Sigma},
\]
and for student parameters \((\beta,m)\in(0,1)\times\R^d\) define
\[
q_{\beta,m}(y):=\beta\,p_{\mathrm{o}}(y)+(1-\beta)\,\gauss{y}{m}{\Sigma}.
\]
Define the phasewise reverse-KL loss
\[
L(\beta,m;\alpha,\nu):=\KL\!\big(q_{\beta,m}\,\|\,p_{\alpha,\nu}\big).
\]

Let \(\widetilde{\nu}(c)=(\alpha_c,\nu_c)\in(0,1)\times\R^d\) denote a fixed demonstration anchor.
We write the teacher and student states as
\[
\widetilde{\nu}_t:=(\alpha_t,\nu_t)\in(0,1)\times\R^d,
\qquad
\widetilde{m}_t:=(\beta_t,m_t)\in(0,1)\times\R^d.
\]
Given parameters \(\gamma>0\) (student step size), \(\zeta\in(0,1]\) (EMA rate), and \(\lambda\in[0,1]\) (demonstrator strength),
we say that \((\widetilde{\nu}_t,\widetilde{m}_t)_{t\ge 0}\) follow \emph{EMA+demonstrator SDFT dynamics} if, for all \(t\ge 0\),
\begin{align}
\widetilde{m}_{t+1}
&=
\widetilde{m}_t-\gamma\,\nabla_{(\beta,m)}L(\beta_t,m_t;\alpha_t,\nu_t),\label{eq:sdft_student_update}\\
\widetilde{\nu}_{t+1}
&=
(1-\zeta)\,\widetilde{\nu}_t
+\zeta\Big((1-\lambda)\,\widetilde{m}_{t+1}+\lambda\,\widetilde{\nu}(c)\Big).\label{eq:sdft_teacher_update}
\end{align}
The update \eqref{eq:sdft_student_update} is a reverse-KL tracking step of the student toward the current teacher,
and \eqref{eq:sdft_teacher_update} is an EMA teacher update that interpolates between the updated student and the demonstration anchor.
\end{definition}

\begin{theorem}\label{thm:sdft_demo_ema}
Assume the teacher iterates remain in the compact set
\[
K
:=
\Big\{(\alpha,\nu)\in(0,1)\times\R^d:\ 
\alpha\in[\underline{\alpha},\overline{\alpha}],\ 
\|\nu-\mu_{\mathrm{o}}\|\le R_\nu,\ 
\mnorm{\nu-\mu_{\mathrm{o}}}{\Sigma^{-1}}\ge \underline{\delta}
\Big\},
\]
for some \(0<\underline{\alpha}\le \overline{\alpha}<1\), \(R_\nu<\infty\), and \(\underline{\delta}>0\),
and assume the demonstration anchor \(\widetilde{\nu}(c)=(\alpha_c,\nu_c)\in K\).
For each \(y=(\alpha,\nu)\in K\) define
\[
F_y(x):=\KL\!\big(q_x\,\|\,p_y\big),
\qquad
x=(\beta,m),\quad q_x:=q_{\beta,m},\quad p_y:=p_{\alpha,\nu}.
\]
Then \(F_y\) is minimized at \(x=y\) and \(\nabla^2 F_y(y)\succ 0\).
Consequently, by continuity and compactness, there exist constants \(\mu>0\), \(L_H<\infty\), and \(r_0>0\) such that
\[
\lambda_{\min}\!\big(\nabla^2F_y(y)\big)\ge \mu\qquad\forall y\in K,
\]
and \(\nabla^2F_y(\cdot)\) is \(L_H\)-Lipschitz on \(B_{r_0}(y)\) uniformly in \(y\in K\).
Define
\[
\rho:=\min\!\left\{r_0,\frac{\mu}{2L_H}\right\},
\qquad
M:=\sup_{\substack{y\in K\\ \|x-y\|\le \rho}}\lambda_{\max}\!\big(\nabla^2F_y(x)\big)<\infty.
\]
Assume \(\|\widetilde{m}_0-\widetilde{\nu}_0\|\le \rho\) and \(\|\widetilde{\nu}_0-\widetilde{\nu}(c)\|\le \rho\), and run the
EMA+demonstrator SDFT dynamics of Definition~\ref{def:sdft_demo_ema} with step size \(0<\gamma\le 1/M\).
Then the following hold.

\medskip
\noindent\textbf{(A) Stability and geometric tracking.}
Let
\[
q:=1-\frac{\gamma\mu}{2}\in(0,1).
\]
Then, for all \(t\ge 0\),
\begin{equation}\label{eq:sdft_contract_to_current_teacher}
\|\widetilde{m}_{t+1}-\widetilde{\nu}_t\|\le q\,\|\widetilde{m}_t-\widetilde{\nu}_t\|.
\end{equation}
Moreover, the teacher error to the demonstration anchor satisfies
\begin{equation}\label{eq:sdft_teacher_to_demo_recursion}
\|\widetilde{\nu}_{t+1}-\widetilde{\nu}(c)\|
\le
(1-\zeta\lambda)\,\|\widetilde{\nu}_t-\widetilde{\nu}(c)\|
+
\zeta(1-\lambda)\,\|\widetilde{m}_{t+1}-\widetilde{\nu}_t\|.
\end{equation}
In particular, if \(\lambda>0\), then \(\widetilde{\nu}_t\to \widetilde{\nu}(c)\) and \(\widetilde{m}_t\to \widetilde{\nu}(c)\), and both sequences remain in the
\(\rho\)-neighborhood where the above curvature bounds apply.

\medskip
\noindent\textbf{(B) Accumlated Old-component drift.}
Let
\[
\widetilde L(\beta,m_{\mathrm{o}},m_{\mathrm{n}};\alpha,\nu)
:=
\KL\!\big(q_{\beta,m_{\mathrm{o}},m_{\mathrm{n}}}\,\|\,p_{\alpha,\nu}\big),
\qquad
q_{\beta,m_{\mathrm{o}},m_{\mathrm{n}}}(y)
:=
\beta\,\gauss{y}{m_{\mathrm{o}}}{\Sigma}
+
(1-\beta)\,\gauss{y}{m_{\mathrm{n}}}{\Sigma}.
\]
There exists a finite constant \(L_{\mathrm{old}}\), depending only on \(K,\Sigma,\rho\), such that along the trajectory,
\begin{equation}\label{eq:sdft_old_grad_lipschitz_track}
\big\|\nabla_{m_{\mathrm{o}}}\widetilde L(\beta_t,\mu_{\mathrm{o}},m_t;\alpha_t,\nu_t)\big\|
\le
L_{\mathrm{old}}\Big(\|\widetilde{m}_t-\widetilde{\nu}_t\|+\|\widetilde{\nu}_t-\widetilde{\nu}(c)\|\Big).
\end{equation}
If \(\lambda>0\), then the right-hand side is summable over \(t\), and in particular
\begin{equation}\label{eq:sdft_old_grad_summable}
\sum_{t=0}^\infty
\big\|\nabla_{m_{\mathrm{o}}}\widetilde L(\beta_t,\mu_{\mathrm{o}},m_t;\alpha_t,\nu_t)\big\|
<\infty,
\end{equation}
so the total update pressure on an already-correct old mean is finite (no accumulated drift).
Moreover, the student converges to the demonstration-induced equilibrium state:
\[
(\beta_t,m_t)\to (\alpha_c,\nu_c)\quad\text{as }t\to\infty.
\]
In particular, for any desired target state \(\widetilde{\nu}^\star=(\alpha^\star,\nu^\star)\in\R^{d+1}\)
(e.g.\ \(\widetilde{\nu}^\star=(\alpha,\mu_{\mathrm{n}})\)), one has the exact limit identity
\begin{equation}\label{eq:sdft_target_error_limit}
\lim_{t\to\infty}\big\|(\beta_t,m_t)-\widetilde{\nu}^\star\big\|
=
\big\|\widetilde{\nu}(c)-\widetilde{\nu}^\star\big\|.
\end{equation}
Consequently, if the demonstration anchor is \(\varepsilon_{\mathrm{demo}}\)-accurate in the sense that
\(\big\|\widetilde{\nu}(c)-\widetilde{\nu}^\star\big\|\le \varepsilon_{\mathrm{demo}}\),
then
\[
\limsup_{t\to\infty}\big\|(\beta_t,m_t)-\widetilde{\nu}^\star\big\|
\le
\varepsilon_{\mathrm{demo}}.
\]
\end{theorem}

\begin{remark}[Connection to SDFT assumptions]\label{rem:sdftassumption}
In \citet[Section~3.2]{shenfeld2026self}, the demonstration-conditioned teacher $\pi(\cdot\mid x,c)$ is argued to be a good target when
(i) it is approximately \emph{reward-optimal} and (ii) it has \emph{minimal deviation} from the current student in KL.
Concretely, the paper motivates conditions of the form
\[
\underbrace{\E_{y\sim \pi(\cdot\mid x,c)}[r(y,x)] \approx \E_{y\sim \pi_{k+1}^\star(\cdot\mid x)}[r(y,x)]}_{\text{(1) Optimality}},
\qquad
\underbrace{\KL\!\big(\pi(\cdot\mid x,c)\,\|\,\pi_k(\cdot\mid x)\big)\approx
\KL\!\big(\pi_{k+1}^\star(\cdot\mid x)\,\|\,\pi_k(\cdot\mid x)\big)}_{\text{(2) Minimal deviation}}.
\]

Our mixture-model theorem instantiates these two requirements as explicit, checkable assumptions on the teacher state and its dynamics.
Condition (1) corresponds to assuming the demonstration anchor $\widetilde{\nu}(c)=(\alpha_c,\nu_c)$ is \emph{close} to the desired target
$(\alpha,\mu_{\mathrm{n}})$ (up to a controllable approximation error): the teacher retains nontrivial old mass ($\alpha_c$ bounded away from $0$)
while pointing its ``new'' component toward the correct new behavior ($\nu_c$ near $\mu_{\mathrm{n}}$).
Condition (2) corresponds to our \emph{tracking regime}: the EMA recursion keeps the evolving teacher $\widetilde{\nu}_t$ close to the student
$\widetilde{m}_t$, and we assume the iterates remain in a uniform local neighborhood $\|\widetilde{m}_t-\widetilde{\nu}_t\|\le \rho$ where the
phasewise reverse-KL loss is uniformly well-conditioned (strongly convex/smooth).
In particular, local curvature implies that ``teacher close to student'' in parameter space is equivalent (up to constants) to ``small KL deviation''
in distribution space: $\KL\!\big(q_{\widetilde{m}_t}\,\|\,p_{\widetilde{\nu}_t}\big)=\Theta\!\big(\|\widetilde{m}_t-\widetilde{\nu}_t\|^2\big).$ Under these two ingredients, the student can improve toward the demonstrated behavior while avoiding accumulated old-mode drift, because the only
update channel acting on an already-correct old component is overlap-gated and becomes summable along the exponentially contracting
student--teacher lag.
\end{remark}

\begin{remark}[Exponential Dependence on Mode Separation]\label{rem:sdft_exp_separation}
If along the EMA+demonstrator trajectory the teacher remains uniformly separated from the old mode,
\[
\inf_{t\ge 0}\mnorm{\nu_t-\mu_{\mathrm{o}}}{\Sigma^{-1}}\ \ge\ \underline\delta>0,
\]
and the tracking tube is chosen small enough that the student new mean also stays separated, e.g.
\[
\sup_{t\ge 0}\mnorm{m_t-\nu_t}{\Sigma^{-1}}
\ \le\ 
\|\Sigma^{-1/2}\|_2\,\sup_{t\ge 0}\|\widetilde m_t-\widetilde \nu_t\|
\ \le\ \frac{\underline\delta}{2}
\quad\Rightarrow\quad
\mnorm{m_t-\mu_{\mathrm{o}}}{\Sigma^{-1}}\ge \underline\delta/2,
\]
and if $\beta_t,\alpha_t$ are bounded away from $(0,1)$ (as ensured by $K$ and the tracking assumption), then there is a constant
$C_{\mathrm{sep}}<\infty$ (depending only on $K,\Sigma$ and the tracking radius) such that for all $t$,
\begin{equation}\label{eq:sdft_old_grad_exp_separation}
\big\|\nabla_{m_{\mathrm{o}}}\widetilde L(\beta_t,\mu_{\mathrm{o}},m_t;\alpha_t,\nu_t)\big\|
\ \le\
C_{\mathrm{sep}}\,
\exp\!\left(-\frac{\underline\delta_{\mathrm{eff}}^2}{8}\right)\,
\|\widetilde m_t-\widetilde \nu_t\|,
\qquad
\underline\delta_{\mathrm{eff}}:=\underline\delta-\|\Sigma^{-1/2}\|_2\,\rho.
\end{equation}
Combining \eqref{eq:sdft_old_grad_exp_separation} with the geometric tracking bound
$\|\widetilde m_t-\widetilde \nu_t\|\le \kappa^t\|\widetilde m_0-\widetilde \nu_0\|$ yields the refined summability estimate
\[
\sum_{t=0}^\infty
\big\|\nabla_{m_{\mathrm{o}}}\widetilde L(\beta_t,\mu_{\mathrm{o}},m_t;\alpha_t,\nu_t)\big\|
\ \le\
\frac{C_{\mathrm{sep}}}{1-\kappa}\,
\exp\!\left(-\frac{\underline\delta_{\mathrm{eff}}^2}{8}\right)\,
\|\widetilde m_0-\widetilde \nu_0\|.
\]
Thus, in well-separated regimes, the total update pressure on an already-correct old mean is not only finite but exponentially small in the
Mahalanobis separation between the old mode and the (teacher/student) new mode.
\end{remark}

Theorem~\ref{thm:sdft_demo_ema} shows that SDFT’s two stabilizers: (a) reverse-KL tracking and (b) an EMA teacher anchored by demonstrations, jointly prevent \emph{mass forgetting} and control \emph{old-component drift}. 
Roughly, Part~(A) ensures convergence to the anchor state $(\alpha_c,\nu_c)$ and thus preserves nonzero old-mode mass, while Part~(B) shows that the gradient acting on an already-correct old mean is summable, ruling out accumulated drift. In particular,
\begin{itemize}
\item Part~(A) formalizes the on-policy effect: because each phasewise reverse-KL objective is locally strongly convex around the current teacher optimum, a single gradient step contracts the student toward the \emph{current} teacher by a uniform factor \(q<1\), echoing the mode-locality of our static reverse-KL analysis. 
The teacher recursion then mediates the evolution of the target itself: the EMA term smooths the teacher toward the updated student, while the anchor term pulls it toward the demonstrator summary \(\widetilde{\nu}(c)\), with \(\lambda\) quantifying the strength of this pull. 
When \(\lambda>0\), the target cannot wander indefinitely: the teacher is attracted to \(\widetilde{\nu}(c)\) and, since the student tracks the teacher, both sequences converge to the same demonstrated state \((\alpha_c,\nu_c)\).
\item  Part~(B) converts this tracking picture into a continual-learning guarantee: the update signal on an already-correct old mean is Lipschitz in the student--teacher lag (and the teacher--anchor mismatch), hence it decays along the trajectory and is summable, ruling out accumulated drift of the old distribution. 
Finally, the limit characterization \eqref{eq:sdft_target_error_limit} makes the consistency story explicit: the asymptotic distance to any desired target \(\widetilde{\nu}^\star\) is exactly the demonstrator’s mismatch \(\|\widetilde{\nu}(c)-\widetilde{\nu}^\star\|\), so an approximately correct demonstration yields an approximately correct limit.
\end{itemize}
\subsection{Mixture-Model Analysis of TTT-Discover}\label{sec:ttt_discover_mixture}

TTT-Discover~\citep{yuksekgonul2026learning} updates a model \emph{during test-time search} on a single hard instance, rather than freezing the policy and relying only on prompting or search heuristics.
Its key ingredient is an \emph{entropic} objective that reweights on-policy samples by $\exp(\eta r)$, thereby favoring high-reward outputs, together with an explicit KL penalty to a fixed reference policy that limits how far the policy can drift.
From the perspective of continual learning, this creates a natural tension between \emph{discovery} (shifting mass toward high-reward behaviors) and \emph{retention} (preserving previously learned behaviors encoded in the reference).

In our two-mode mixture model, this raises two concrete questions.
First, can the entropic objective itself cause \emph{strong forgetting}, i.e.\ collapse of the old-mode mass $\beta$?
Second, when the old mode is already correctly represented, does the objective exert a nontrivial drift signal on the old mean, or is this \emph{weak forgetting} suppressed by overlap?
The next lemma and theorem answer these questions first in a disjoint-support idealization, and then in the Gaussian mixture setting of the paper.

\paragraph{A KL-anchored Entropic Objective.}
Let $r:\R^d\to\R$ be a measurable reward and fix an entropic parameter $\eta>0$.
For any density $q$, define
\[
J_\eta(q)\;:=\;\log \E_{Y\sim q}\big[e^{\eta r(Y)}\big].
\]
Fix a reference density $q_0$ and a KL coefficient $\lambda_{\mathrm{ref}}\ge 0$, and define
\begin{equation}\label{eq:ttt_objective_def_streamlined}
\mathcal{L}_{\eta,\lambda_{\mathrm{ref}}}(q)
\;:=\;
J_\eta(q)\;-\;\lambda_{\mathrm{ref}}\,\KL(q\|q_0).
\end{equation}

\paragraph{Two-Mode Mixture Family.}
To analyze the objective~\eqref{eq:ttt_objective_def_streamlined}, we first restrict the learner to a two-mode mixture family
\[
q_\beta(y) := \beta\,p_{\mathrm{o}}(y) + (1-\beta)\,p_{\mathrm{n}}(y),
\qquad \beta \in [0,1].
\]
We take the reference density to be a fixed mixture $q_0 = q_{\beta_0}$ for some $\beta_0 \in (0,1)$, which serves as the KL anchor in~\eqref{eq:ttt_objective_def_streamlined}. 
\begin{lemma}[Disjoint-support Intuition for TTT-Discover]\label{lem:ttt_disjoint}
Fix $\eta>0$ and $\lambda_{\mathrm{ref}}\ge 0$.
Consider the disjoint-support setting from Definition~\ref{def:disjointsupp} and assume the reward is constant on each region:
\begin{align}\label{eq:rewdisjoint}
r(y)=u_{\mathrm{o}}\ \text{for }y\in A_{\mathrm{o}},
\qquad
r(y)=u_{\mathrm{n}}\ \text{for }y\in A_{\mathrm{n}},
\end{align}
for constants $u_{\mathrm{o}},u_{\mathrm{n}}\in\R$.
Then:
\begin{enumerate}
\item For every $\beta\in[0,1]$,
\begin{equation}\label{eq:ttt_disjoint_J_short}
J_\eta(q_\beta)=\log\!\Big(\beta e^{\eta u_{\mathrm{o}}}+(1-\beta)e^{\eta u_{\mathrm{n}}}\Big),
\qquad
\KL(q_\beta\|q_{\beta_0})=\beta\log\frac{\beta}{\beta_0}+(1-\beta)\log\frac{1-\beta}{1-\beta_0}.
\end{equation}
\item If $\lambda_{\mathrm{ref}}=0$ and $u_{\mathrm{n}}>u_{\mathrm{o}}$, then $\beta^\star=0$ is the unique maximizer of
$\beta\mapsto \mathcal{L}_{\eta,0}(q_\beta)$.
(Symmetrically, if $u_{\mathrm{o}}>u_{\mathrm{n}}$, then $\beta^\star=1$.)
\item If $\lambda_{\mathrm{ref}}>0$ and $u_{\mathrm{o}}\neq u_{\mathrm{n}}$, then
$\beta\mapsto \mathcal{L}_{\eta,\lambda_{\mathrm{ref}}}(q_\beta)$ is strictly concave on $(0,1)$ and has a unique maximizer
$\beta^\star\in(0,1)$.
If $u_{\mathrm{o}}=u_{\mathrm{n}}$, then the unique maximizer is $\beta^\star=\beta_0$.
\end{enumerate}
\end{lemma}

The disjoint-support case cleanly isolates the basic mechanism.
Without the KL anchor ($\lambda_{\mathrm{ref}}=0$), the entropic utility is purely \emph{mode-seeking}: it places all mass on whichever mode has higher reward, so if the new mode is preferred then strong forgetting occurs through $\beta^\star=0$.
Once $\lambda_{\mathrm{ref}}>0$, however, the KL term forces the optimizer to remain in the interior whenever the reference retains both modes, so mass collapse is ruled out in this idealized setting.
Thus, in the absence of overlap, retention is controlled entirely by whether the reference policy itself preserves old mass.

Now considering the two-mode Gaussian mixture setting from Section~\ref{sec:mog}, we return to the full parametric mixture family
\[
q_{\beta,m_{\mathrm{o}},m_{\mathrm{n}}}(y)
=
\beta\,\gauss{y}{m_{\mathrm{o}}}{\Sigma}
+
(1-\beta)\,\gauss{y}{m_{\mathrm{n}}}{\Sigma},
\]
in which the component means are also learned.

\begin{theorem}[TTT-Discover in the Gaussian mixture setting]\label{thm:ttt_gaussian}
Consider the target model in~\eqref{eq:truemodel}. Fix the learner model (see \eqref{eq:estmodel}) means at the true old/new means and define $q_\beta(y):=q_{\beta,\mu_{\mathrm{o}},\mu_{\mathrm{n}}}(y)=\beta\,p_{\mathrm{o}}(y)+(1-\beta)\,p_{\mathrm{n}}(y)
$, $\beta\in[0,1]$, and $q_0:=q_{\beta_0}$ for some $\beta_0\in(0,1)$. Let $\gamma:=\Phi\!(-{\delta}/{2})$, where $\Phi$ is the standard normal CDF and define $\kappa:=1-2\gamma\in(0,1)$. Define the Bayes partition
\begin{align}\label{eq:partitionforreward}
A_{\mathrm{n}}
:=
\Big\{y\in\R^d:\ (\mu_{\mathrm{n}}-\mu_{\mathrm{o}})^\top\Sigma^{-1}
\Big(y-\frac{\mu_{\mathrm{o}}+\mu_{\mathrm{n}}}{2}\Big)\ge 0\Big\},
\qquad
A_{\mathrm{o}}:=\R^d\setminus A_{\mathrm{n}},
\end{align}
and let the reward be the corresponding two-level step function
\begin{align}\label{stepreward}
r(y)=u_{\mathrm{o}}\ \text{for }y\in A_{\mathrm{o}},
\qquad
r(y)=u_{\mathrm{n}}\ \text{for }y\in A_{\mathrm{n}}.
\end{align}
Then the following hold.

\medskip
\noindent\textbf{(A)} Suppose $u_{\mathrm{n}}>u_{\mathrm{o}}$.
Define
\[
D(\beta):=\KL(q_\beta\|q_{\beta_0}),
\qquad
J_\eta(q_\beta)
=
\log\!\Big(
e^{\eta u_{\mathrm{o}}}\big(\gamma+\kappa\beta\big)
+
e^{\eta u_{\mathrm{n}}}\big(1-\gamma-\kappa\beta\big)
\Big),
\]
and let
\begin{equation}\label{eq:lambda_crit_new_streamlined}
\lambda_{\mathrm{crit}}^{(\mathrm{new})}
:=
\frac{-J_\eta'(q_\beta)\big|_{\beta=0}}{-D'(0)}
=
\frac{\kappa\big(e^{\eta u_{\mathrm{n}}}-e^{\eta u_{\mathrm{o}}}\big)}
{\big(\gamma e^{\eta u_{\mathrm{o}}}+(1-\gamma)e^{\eta u_{\mathrm{n}}}\big)\,(-D'(0))}
>0.
\end{equation}
Then:
\begin{itemize}
\item if $0\le \lambda_{\mathrm{ref}}\le \lambda_{\mathrm{crit}}^{(\mathrm{new})}$, the unique maximizer of
$\beta\mapsto \mathcal{L}_{\eta,\lambda_{\mathrm{ref}}}(q_\beta)$ is $\beta^\star=0$;
\item if $\lambda_{\mathrm{ref}}>\lambda_{\mathrm{crit}}^{(\mathrm{new})}$, the unique maximizer satisfies $\beta^\star\in(0,\beta_0)$.
\end{itemize}
Thus, unlike the disjoint-support setting, a positive KL anchor does \emph{not} automatically prevent strong forgetting:
the anchor must be sufficiently strong.

\medskip
\noindent\textbf{(B)} Consider now the learner family in~\eqref{eq:estmodel} and define
\[
J_\eta(\beta,m_{\mathrm{o}},m_{\mathrm{n}})
:=
\log \E_{Y\sim q_{\beta,m_{\mathrm{o}},m_{\mathrm{n}}}}\big[e^{\eta r(Y)}\big].
\]
At the correct old mean $m_{\mathrm{o}}=\mu_{\mathrm{o}}$, the gradient satisfies
\begin{equation}\label{eq:ttt_oldmean_grad_explicit_streamlined}
\nabla_{m_{\mathrm{o}}} J_\eta(\beta,\mu_{\mathrm{o}},m_{\mathrm{n}})
=
\beta\,(w_{\mathrm{n}}-w_{\mathrm{o}})\,
\frac{\varphi(\delta/2)}{\delta}\,
\Sigma^{-1}(\mu_{\mathrm{n}}-\mu_{\mathrm{o}}),
\end{equation}
where $\varphi(t)=(2\pi)^{-1/2}e^{-t^2/2}$ is the standard normal density and
\[
w_{\mathrm{o}}:=\frac{e^{\eta u_{\mathrm{o}}}}{\E_{Y\sim q_{\beta,\mu_{\mathrm{o}},m_{\mathrm{n}}}}[e^{\eta r(Y)}]},
\qquad
w_{\mathrm{n}}:=\frac{e^{\eta u_{\mathrm{n}}}}{\E_{Y\sim q_{\beta,\mu_{\mathrm{o}},m_{\mathrm{n}}}}[e^{\eta r(Y)}]}.
\]
If moreover $|u_{\mathrm{o}}|,|u_{\mathrm{n}}|\le R$, then
\begin{equation}\label{eq:ttt_oldmean_grad_bound_streamlined}
\big\|\nabla_{m_{\mathrm{o}}} J_\eta(\beta,\mu_{\mathrm{o}},m_{\mathrm{n}})\big\|
\ \le\
\beta\,(e^{2\eta R}-e^{-2\eta R})\,
\frac{1}{\sqrt{2\pi}}\,
\frac{e^{-\delta^2/8}}{\delta}\,
\big\|\Sigma^{-1}(\mu_{\mathrm{n}}-\mu_{\mathrm{o}})\big\|.
\end{equation}
In particular, the old-mean learning signal is exponentially small in the Mahalanobis separation $\delta$.
Furthermore, if the full KL-anchored objective is evaluated at a synchronized point $q_0=q_{\beta,\mu_{\mathrm{o}},m_{\mathrm{n}}}$, then the KL term has zero
gradient in $m_{\mathrm{o}}$, so the same bound applies to the full TTT-style objective at that point.
\end{theorem}

Theorem~\ref{thm:ttt_gaussian} shows that the disjoint-support intuition survives qualitatively in the overlapping Gaussian setting, but with an important refinement. As in the idealized case, the unanchored entropic utility is intrinsically mode-seeking and therefore reallocates mass toward the higher-reward mode, potentially causing strong forgetting through collapse of the old-mode weight. However, unlike the disjoint-support regime, a positive KL anchor does \emph{not} automatically rule out this collapse, because the Gaussian components have common support and the KL penalty remains finite at the boundary $\beta=0$. Instead, retention requires the anchor to be sufficiently strong, and the threshold \eqref{eq:lambda_crit_new_streamlined} makes this tradeoff explicit: the reward tilt promotes discovery, while the KL anchor must be large enough to counterbalance that tendency and preserve memory.

At the same time, part~(B) shows that the weak-forgetting story is much more favorable. Even in regimes where mass collapse is still possible, the learning signal on an already-correct old mean is highly localized: once the old mode is correctly positioned, only samples near the Bayes decision boundary generate nontrivial update pressure, and that pressure decays exponentially in the Mahalanobis separation $\delta$. This mirrors the reverse-KL locality results proved earlier in the paper, where overlap similarly controls the extent to which an old mode can be perturbed. Thus, in the Gaussian mixture model, TTT-style entropic objectives cleanly separate two effects: the reward term drives discovery by reallocating mixture mass, while the geometry of overlap controls the drift of already-correct old parameters.

We further provide an exact characterization of $\beta^*$ for both the disjoint case and the Gaussian case in Proposition~\label{prop:ttt_beta_star_characterization} in Section~\ref{sec:exactchar}.

\subsection{Mixture-Model Analysis of OAPL}\label{subsec:OAPL_consolidated}

We next consider the OAPL algorithm~\citep{brantley2025accelerating,ritter2026llms} which uses a \emph{frozen} reference policy $q_0$ (typically a lagged inference engine) both to generate samples and to define the update target.
In distribution space, the corresponding KL-regularized improvement problem has the closed-form optimizer
\[
q^*(y)=\frac{1}{Z}\,q_0(y)e^{r(y)/\tau},
\qquad
Z:=\E_{Y\sim q_0}\big[e^{r(Y)/\tau}\big],
\]
and OAPL fits this target by a squared ``advantage-matching'' regression under the same reference measure.
This makes OAPL technically off-policy relative to the current trainable parameters, but still self-consistent: the sampling distribution and the optimization target are defined from the same reference policy.

\paragraph{OAPL Setup (Single Phase).}
Fix a frozen reference mixture
\begin{equation}\label{eq:OAPL_ref_consolidated}
q_0(y):=q_{\beta_0,\mu_{\mathrm{o}},\mu_{\mathrm{n}}}(y)=\beta_0\,p_{\mathrm{o}}(y)+(1-\beta_0)\,p_{\mathrm{n}}(y),
\qquad \beta_0\in(0,1),
\end{equation}
where
\[
p_{\mathrm{o}}(y):=\gauss{y}{\mu_{\mathrm{o}}}{\Sigma},
\qquad
p_{\mathrm{n}}(y):=\gauss{y}{\mu_{\mathrm{n}}}{\Sigma},
\qquad
\mu_{\mathrm{o}}\neq \mu_{\mathrm{n}},\quad \Sigma\succ 0.
\]
Let $r:\R^d\to\R$ be measurable and $\tau>0$.
Define
\[
V^*:=\tau\log \E_{Y\sim q_0}[e^{r(Y)/\tau}],
\qquad
A^*(y):=r(y)-V^*,
\qquad
q^*(y):=q_0(y)e^{A^*(y)/\tau}.
\]
For the parametric family
\[
q_{\beta,m_{\mathrm{n}}}(y):=\beta\,p_{\mathrm{o}}(y)+(1-\beta)\,\gauss{y}{m_{\mathrm{n}}}{\Sigma},
\]
the OAPL population regression objective is
\[
J(\beta,m_{\mathrm{n}})
:=
\E_{Y\sim q_0}\!\left[\left(\tau\log\frac{q_{\beta,m_{\mathrm{n}}}(Y)}{q_0(Y)}-A^*(Y)\right)^2\right].
\]
We first analyze the disjoint-support case, before analyzing the Gaussian mixture model.

\begin{lemma}[Disjoint-support OAPL only Weights Existing Mode Mass]\label{lem:OAPL_disjoint}
Consider the disjoint-support setting from Definition~\ref{def:disjointsupp} and the step-wise constant rewards as in~\eqref{eq:rewdisjoint}. Then the OAPL target $q^*$ is again a two-component mixture with unchanged components,
\[
q^*(y)=\beta^*\,p_{\mathrm{o}}(y)+(1-\beta^*)\,p_{\mathrm{n}}(y),
\]
where
\begin{equation}\label{eq:OAPL_beta_star_disjoint}
\beta^*
=
\frac{\beta_0 e^{r_{\mathrm{o}}/\tau}}
{\beta_0 e^{r_{\mathrm{o}}/\tau}+(1-\beta_0)e^{r_{\mathrm{n}}/\tau}}.
\end{equation}
In particular, if $\beta_0\in(0,1)$ and $r_{\mathrm{o}},r_{\mathrm{n}}$ are finite, then $\beta^*\in(0,1)$.
\end{lemma}

In the disjoint-support idealization, OAPL cannot create or destroy modes; it can only \emph{reweight} the mass already present in the reference.
The explicit formula \eqref{eq:OAPL_beta_star_disjoint} shows that old-mode mass is retained whenever the reference already assigns nonzero mass to it.
Equally importantly, if the reference has already collapsed the old mode (e.g.\ $\beta_0=0$), then OAPL cannot recover it in this idealization.

\begin{theorem}[OAPL in the Gaussian Mixture Model]\label{thm:OAPL_gaussian}
Let the reference mixture $q_0$ be given by \eqref{eq:OAPL_ref_consolidated}, and assume the reward is the two-level step function as in~\eqref{stepreward} based on the Bayes partition in~\eqref{eq:partitionforreward}. Let $r_{\mathrm{o}}^{(0)}(y):={\beta_0\,p_{\mathrm{o}}(y)}/{q_0(y)}$ denote the \emph{old responsibility} under the frozen reference. Then the following hold.

\medskip
\noindent\textbf{(A)} The expected old responsibility of the target $q^*$ is
\begin{equation}\label{eq:OAPL_expected_old_resp}
\E_{Y\sim q^*}\big[r_{\mathrm{o}}^{(0)}(Y)\big]
=
\frac{
\beta_0\Big((1-\gamma)e^{r_{\mathrm{o}}/\tau}+\gamma e^{r_{\mathrm{n}}/\tau}\Big)
}{
\beta_0\Big((1-\gamma)e^{r_{\mathrm{o}}/\tau}+\gamma e^{r_{\mathrm{n}}/\tau}\Big)
+
(1-\beta_0)\Big(\gamma e^{r_{\mathrm{o}}/\tau}+(1-\gamma)e^{r_{\mathrm{n}}/\tau}\Big)
}.
\end{equation}
In particular, if $\beta_0\in(0,1)$ and $r_{\mathrm{o}},r_{\mathrm{n}}$ are finite, then $\E_{Y\sim q^*}\big[r_{\mathrm{o}}^{(0)}(Y)\big]\in(0,1),$ so the OAPL target cannot completely forget the old mode as long as the reference retains nonzero old mass.

\medskip
\noindent\textbf{(B)} For the parametric objective
\[
J(\beta,m_{\mathrm{n}})
=
\E_{Y\sim q_0}\!\left[\left(\tau\log\frac{q_{\beta,m_{\mathrm{n}}}(Y)}{q_0(Y)}-A^*(Y)\right)^2\right],
\]
the gradient with respect to the new mean satisfies
\begin{equation}\label{eq:OAPL_grad_m_identity_consolidated}
\nabla_{m_{\mathrm{n}}} J(\beta,m_{\mathrm{n}})
=
2\tau\,\E_{Y\sim q_0}\Big[\Delta_{\beta,m_{\mathrm{n}}}(Y)\;r_{\mathrm{n}}^{(\beta,m_{\mathrm{n}})}(Y)\;\Sigma^{-1}(Y-m_{\mathrm{n}})\Big],
\qquad
\Delta_{\beta,m_{\mathrm{n}}}(y):=\tau\log\tfrac{q_{\beta,m_{\mathrm{n}}}(y)}{q_0(y)}-A^*(y).
\end{equation}
Hence the influence of old-mode samples on the new-mean update is gated by the new-component responsibility
$r_{\mathrm{n}}^{(\beta,m_{\mathrm{n}})}$.
If moreover $|r(y)|\le R$ for all $y$, then at the synchronized point $(\beta,m_{\mathrm{n}})=(\beta_0,\mu_{\mathrm{n}})$,
\begin{equation}\label{eq:OAPL_oldmode_grad_bound_consolidated}
\Big\|
2\tau\,\beta_0\,
\E_{Y\sim p_{\mathrm{o}}}\big[A^*(Y)\,r_{\mathrm{n}}^{(\beta_0,\mu_{\mathrm{n}})}(Y)\,\Sigma^{-1}(Y-\mu_{\mathrm{n}})\big]
\Big\|
\ \le\
4\tau R\,\beta_0\,\sqrt{M_{\mathrm{o}\to\mathrm{n}}}\,\sqrt{\varepsilon_{\mathrm{o}\to\mathrm{n}}^{\mathrm{ref}}},
\end{equation}
where
\[
M_{\mathrm{o}\to\mathrm{n}}:=\E_{Y\sim p_{\mathrm{o}}}\big[\|\Sigma^{-1}(Y-\mu_{\mathrm{n}})\|^2\big]
=
\mathrm{tr}(\Sigma^{-1})+(\mu_{\mathrm{o}}-\mu_{\mathrm{n}})^\top\Sigma^{-2}(\mu_{\mathrm{o}}-\mu_{\mathrm{n}}),
\]
and
\[
\varepsilon_{\mathrm{o}\to\mathrm{n}}^{\mathrm{ref}}
:=
\E_{Y\sim p_{\mathrm{o}}}\big[r_{\mathrm{n}}^{(\beta_0,\mu_{\mathrm{n}})}(Y)\big]
\le
\frac12\sqrt{\frac{1-\beta_0}{\beta_0}}\exp\!\left(-\frac18\mnorm{\mu_{\mathrm{n}}-\mu_{\mathrm{o}}}{\Sigma^{-1}}^2\right).
\]
Thus the old-mode contribution to the update of the new mean is exponentially small in the separation.
\end{theorem}

Part~(A) is the Gaussian analogue of the disjoint-support mass-reweighting formula.
Because the Gaussian components overlap, the target $q^*$ is no longer exactly a mixture in the same family, but its \emph{expected old responsibility}
admits the explicit formula \eqref{eq:OAPL_expected_old_resp}.
This shows that, unlike purely new-only forward-KL SFT, OAPL cannot entirely discard the old mode unless the frozen reference has already done so.

Part~(B) shows that OAPL is also \emph{geometrically local}: the contribution of old-mode samples to the update of the new mean is suppressed by
the overlap factor \(\varepsilon_{\mathrm{o}\to\mathrm{n}}^{\mathrm{ref}}\), which decays exponentially with the Mahalanobis separation.
Thus OAPL combines two stabilizing features in this toy model: the reference anchor retains a positive amount of old mass, and the parametric update
remains localized because cross-mode influence is exponentially small.

\section{Conclusion}

Within a mixture-model framework, we analyze how two predominant continual-learning mechanisms, namely on-policy sampling and replay-based memory access to past behavior, mitigate catastrophic forgetting. In the absence of these mechanisms, the effective training distribution becomes dominated by new information, making forgetting difficult to avoid even when the model class can represent both old and new behaviors. Our analysis shows that forward-KL objectives trained on new-only data naturally induce mass forgetting of the old behavior, whereas reverse-KL–style updates aligned with on-policy targets retain old modes and produce only overlap-controlled drift. Replay interacts with these objectives in fundamentally different ways: under forward-KL it must modify the training distribution to prevent mass forgetting, while under reverse-KL it primarily stabilizes stochastic optimization by ensuring persistent visibility of past behavior. These insights extend to several modern near–on-policy post-training methods, including SDFT, TTT-Discover, and OAPL, whose updates can be interpreted through the same mixture-model perspective. 

Future work includes extending this analysis to high-dimensional generative models where behaviors correspond to richer semantic modes rather than mixture models. Another interesting promising direction is to design principled post-training algorithms that explicitly balance exploration of new behaviors with retention of past capabilities using theoretically grounded sampling or memory mechanisms, built upon the insights derived in this work.

\bibliographystyle{abbrvnat}
\bibliography{ref}

\clearpage
\appendix
\noindent\rule{\textwidth}{1pt}
\begin{center}
\vspace{7pt}
{\Large  Appendix}
\end{center}
\noindent\rule{\textwidth}{1pt}

\section{RL-based Post-training and Reverse-KL Minimization}\label{app:rltokl}

In KL-regularized RL (the standard trust-region formulation of RL), the policy update is posed as
\[
\max_{\pi}\;\; \E_{\pi}[r] \;-\;\tau\,\KL(\pi\|\pi_{\mathrm{ref}}),
\]
which penalizes deviation from a reference policy $\pi_{\mathrm{ref}}$ while improving reward.
The unique optimizer is the exponential tilt
\[
\pi^*(y)\;=\;\frac{1}{Z}\,\pi_{\mathrm{ref}}(y)\,e^{r(y)/\tau},
\qquad
Z:=\E_{Y\sim \pi_{\mathrm{ref}}}\big[e^{r(Y)/\tau}\big].
\]
Assuming 
$Z < \infty$, this distribution is well defined.
Moreover, for any $\pi\ll \pi_{\mathrm{ref}}$ one has the exact identity
\[
\tau\log Z \;-\;\Big(\E_{\pi}[r]-\tau\KL(\pi\|\pi_{\mathrm{ref}})\Big)
\;=\;\tau\,\KL(\pi\|\pi^*).
\]
Thus maximizing the KL-regularized RL objective is equivalent (up to an additive constant $\tau\log Z$) to minimizing the $\KL(\pi\|\pi^*)$ to the reward-tilted target $\pi^*$; see~\cite{korbak2022reinforcement} for more details.

A concrete example in our two-mode Gaussian setting (defined in Section~\ref{sec:mog}) is obtained by choosing
\[
\pi_{\mathrm{ref}}(y):=p_{\beta_0}(y)=\beta_0\,p_{\mathrm{o}}(y)+(1-\beta_0)\,p_{\mathrm{n}}(y),
\qquad \beta_0\in(0,1),
\]
and defining the reward
\[
r(y):=\tau\log\frac{p_\alpha(y)}{q_{\beta_0}(y)}
=
\tau\log\frac{\alpha\,p_{\mathrm{o}}(y)+(1-\alpha)\,p_{\mathrm{n}}(y)}
{\beta_0\,p_{\mathrm{o}}(y)+(1-\beta_0)\,p_{\mathrm{n}}(y)}.
\]
This reward can be interpreted as a \emph{log-density correction}: it assigns positive reward to outputs that are underweighted by the reference policy relative to the desired true target $p_\alpha$, and negative reward to outputs that are overweighted, so that the KL-regularized RL update exactly tilts the reference policy toward $p_\alpha$. In this case, we have
\[
Z
=
\E_{Y\sim \pi_{\mathrm{ref}}}\big[e^{r(Y)/\tau}\big]
=
\int q_{\beta_0}(y)\,\frac{p_\alpha(y)}{q_{\beta_0}(y)}\,dy
=
\int p_\alpha(y)\,dy
=
1,
\]
so the tilted optimizer becomes
\[
\pi^*(y)
=
\frac{1}{Z}\,\pi_{\mathrm{ref}}(y)e^{r(y)/\tau}
=
q_{\beta_0}(y)\,\frac{p_\alpha(y)}{q_{\beta_0}(y)}
=
p_\alpha(y)
=
\alpha\,p_{\mathrm{o}}(y)+(1-\alpha)\,p_{\mathrm{n}}(y).
\]
Thus, with this choice of reference policy and reward, the KL-regularized RL solution exactly recovers the Gaussian mixture target.

\section{Proofs} \label{app:proofs}
We state a few results used in the proof.

\begin{lemma}[Gaussian Stein identity~\citep{chen2010normal}]\label{lem:stein}
Let $Y\sim\mathcal{N}(\mu,\Sigma)$ with $\Sigma\succ 0$, and let $g:\R^d\to\R$ be continuously differentiable with $\E[\norm{\nabla g(Y)}]<\infty$.
Then
\[
\E\big[\Sigma^{-1}(Y-\mu)\,g(Y)\big] = \E\big[\nabla g(Y)\big].
\]
\end{lemma}

\begin{lemma}[Fourth moment of a Gaussian quadratic form]\label{lem:gauss_quad_4th} Let $Z\sim\mathcal{N}(0,I_d)$ and let $A\succeq 0$ be symmetric. Then \[ \E\big[(Z^\top A Z)^2\big] \;=\; 2\,\mathrm{tr}(A^2)+\big(\mathrm{tr}(A)\big)^2. \] \end{lemma} 
\begin{proof} 
This is standard; one way to establish the result is to diagonalize $A=U^\top \Lambda U$ with $\Lambda=\mathrm{diag}(\lambda_1,\dots,\lambda_d)$ and $UZ\overset{d}{=}Z$. Then $Z^\top A Z=\sum_i \lambda_i Z_i^2$ and \[ \E\Big[\Big(\sum_i\lambda_i Z_i^2\Big)^2\Big] =\sum_i\lambda_i^2\E[Z_i^4]+\sum_{i\neq j}\lambda_i\lambda_j\E[Z_i^2]\E[Z_j^2] =3\sum_i\lambda_i^2+\sum_{i\neq j}\lambda_i\lambda_j =2\sum_i\lambda_i^2+\Big(\sum_i\lambda_i\Big)^2. \]
establishing the result.
\end{proof}

% \subsection{Proof of Lemma~\ref{lem:disjoint-kl}}

\subsection{Proof of Lemma~\ref{lem:leakage}}
% \PLBC*
\begin{proof}[Proof of Lemma~\ref{lem:leakage}]
For the first bound,
\[
\E_{g}[r_f(Y)] = \int g(y)\frac{w f(y)}{w f(y)+(1-w)g(y)}\,dy
= \int \frac{w f(y)g(y)}{w f(y)+(1-w)g(y)}\,dy.
\]
Using $a+b\ge 2\sqrt{ab}$ with $a=w f(y)$ and $b=(1-w)g(y)$ yields
\[
w f(y)+(1-w)g(y)\ge 2\sqrt{w(1-w)f(y)g(y)},
\]
so the integrand is at most
\[
\frac{w f(y)g(y)}{2\sqrt{w(1-w)f(y)g(y)}}
= \frac12\sqrt{\frac{w}{1-w}}\;\sqrt{f(y)g(y)}.
\]
Integrating gives the stated inequality, and the second inequality follows by symmetry.
\end{proof}

\subsection{Proof of Theorem~\ref{thm:kl_main}}
% \fKL*
\begin{proof}[Proof of Theorem~\ref{thm:kl_main}]
Define the likelihood ratio $X(y):=\frac{p_{\mathrm{o}}(y)}{p_{\mathrm{n}}(y)}$.
Then for every $y$,
\[
q_\beta(y)=\beta p_{\mathrm{o}}(y)+(1-\beta)p_{\mathrm{n}}(y)=p_{\mathrm{n}}(y)\big((1-\beta)+\beta X(y)\big),
\]
so
\[
L_{\mathrm{SFT}}(\beta)=\KL(p_{\mathrm{n}}\|q_\beta)
=\E_{Y\sim p_{\mathrm{n}}}\!\left[\log\frac{p_{\mathrm{n}}(Y)}{q_\beta(Y)}\right]
=\E_{p_{\mathrm{n}}}\!\left[-\log\big((1-\beta)+\beta X(Y)\big)\right].
\]
Since $\E_{p_{\mathrm{n}}}[X(Y)]=\int p_{\mathrm{o}}=1$ and $-\log$ is strictly convex,
Jensen's inequality gives
\[
L_{\mathrm{SFT}}(\beta)\ \ge\ -\log\Big((1-\beta)+\beta\,\E[X(Y)]\Big)=-\log(1)=0,
\]
with strict inequality for $\beta>0$ because $X$ is not a.s.\ constant under $p_{\mathrm{n}}$ when $\mu_{\mathrm{o}}\neq\mu_{\mathrm{n}}$.
Thus $\beta=0$ is the unique minimizer.

Differentiate under the expectation (justified since $(1-\beta)+\beta X(Y)\ge 1-\beta>0$ and Gaussians have all moments):
\[
L_{\mathrm{SFT}}'(\beta)
=-\E_{p_{\mathrm{n}}}\!\left[\frac{X(Y)-1}{(1-\beta)+\beta X(Y)}\right],
\qquad
L_{\mathrm{SFT}}''(\beta)
=\E_{p_{\mathrm{n}}}\!\left[\frac{(X(Y)-1)^2}{\big((1-\beta)+\beta X(Y)\big)^2}\right].
\]
Since $X$ is not a.s.\ constant, $L_{\mathrm{SFT}}''(\beta)>0$ for all $\beta\in(0,1)$, hence $L_{\mathrm{SFT}}'$ is strictly increasing.
Also $L_{\mathrm{SFT}}'(0)=-\E[X(Y)-1]=0$, so $L_{\mathrm{SFT}}'(\beta)>0$ for all $\beta\in(0,1)$.
Therefore $L_{\mathrm{SFT}}$ is strictly increasing on $[0,1]$.

Now parameterize $\beta=\sigma(\phi)$. Since $\frac{d\beta}{d\phi}=\beta(1-\beta)>0$,
\[
\frac{d}{d\phi}L_{\mathrm{SFT}}(\sigma(\phi))
=\beta(1-\beta)\,L_{\mathrm{SFT}}'(\beta)\ >\ 0\qquad\forall \beta\in(0,1),
\]
so along logit gradient flow $\dot\phi=-\frac{d}{d\phi}L_{\mathrm{SFT}}(\sigma(\phi))$ we have $\dot\phi<0$ whenever $\beta\in(0,1)$.
Thus $\phi(t)$ is strictly decreasing and $\beta(t)$ is strictly decreasing, hence $\beta(t)\to \beta_\infty\in[0,1]$ exists.
If $\beta_\infty>0$, then $L_{\mathrm{SFT}}'(\beta_\infty)>0$, so $\frac{d}{d\phi}L_{\mathrm{SFT}}(\sigma(\phi(t)))$
stays bounded below by a positive constant for large $t$, contradicting $\phi(t)$ having a finite limit.
Hence $\beta_\infty=0$.

To obtain \eqref{eq:kl_sft_logit_grad}, note that $\partial_\phi\log q_\beta(y)=r_{\mathrm{o}}(y)-\beta$ for any two-component mixture
(with $\beta=\sigma(\phi)$). Since $L_{\mathrm{SFT}}(\beta)=\E_{p_{\mathrm{n}}}[\log p_{\mathrm{n}}-\log q_\beta]$,
\[
\frac{d}{d\phi}L_{\mathrm{SFT}}(\sigma(\phi))=-\E_{p_{\mathrm{n}}}\big[\partial_\phi \log q_\beta(Y)\big]
=\beta-\E_{p_{\mathrm{n}}}[r_{\mathrm{o}}(Y)],
\]
yielding \eqref{eq:kl_sft_logit_grad}.
Finally, \eqref{eq:kl_sft_leak_bound} is Lemma~\ref{lem:leakage} applied to the mixture $q_\beta$ with $w=\beta$, $f=p_{\mathrm{o}}$, $g=p_{\mathrm{n}}$,
combined with Lemma~\ref{lem:bc-gauss} and $\mathrm{BC}(p_{\mathrm{o}},p_{\mathrm{n}})=e^{-\delta^2/8}$.
\end{proof}

\subsection{Proof of Lemma~\ref{lem:replay_forwardKL_consolidated}}
% \ReplayFKL*

\begin{proof}[Proof of Lemma~\ref{lem:replay_forwardKL_consolidated}]
\textbf{(A)} Expand
\[
\tilde q_{\beta,\lambda}
=(1-\lambda)\big(\beta p_{\mathrm{o}}+(1-\beta)p_{\mathrm{n}}\big)+\lambda p_{\mathrm{o}}
=\big(\lambda+(1-\lambda)\beta\big)p_{\mathrm{o}}+(1-\lambda)(1-\beta)p_{\mathrm{n}}.
\]
Set $\tilde\beta:=\lambda+(1-\lambda)\beta$ so that $\tilde q_{\beta,\lambda}=q_{\tilde\beta}$.
Thus
\[
\min_{\beta\in[0,1]}\KL(p_{\mathrm{n}}\|\tilde q_{\beta,\lambda})
=
\min_{\tilde\beta\in[\lambda,1]}\KL(p_{\mathrm{n}}\|q_{\tilde\beta}).
\]
It remains to show $\gamma\mapsto F(\gamma):=\KL(p_{\mathrm{n}}\|q_\gamma)$ is strictly increasing on $(0,1)$.
Differentiate under the integral:
\[
F(\gamma)=\int p_{\mathrm{n}}(y)\log\frac{p_{\mathrm{n}}(y)}{(1-\gamma)p_{\mathrm{n}}(y)+\gamma p_{\mathrm{o}}(y)}\,dy,
\]
\[
F'(\gamma)
= -\int p_{\mathrm{n}}(y)\,\frac{p_{\mathrm{o}}(y)-p_{\mathrm{n}}(y)}{q_\gamma(y)}\,dy,
\qquad
F''(\gamma)
= \int p_{\mathrm{n}}(y)\,\frac{(p_{\mathrm{o}}(y)-p_{\mathrm{n}}(y))^2}{q_\gamma(y)^2}\,dy.
\]
Since $p_{\mathrm{o}}\not\equiv p_{\mathrm{n}}$ (distinct means), $F''(\gamma)>0$ on $(0,1)$, so $F'$ is strictly increasing.
Also $F'(0)=-\int(p_{\mathrm{o}}-p_{\mathrm{n}})=0$, hence $F'(\gamma)>0$ for all $\gamma\in(0,1)$ and $F$ is strictly increasing.
Therefore the minimizer over $[\lambda,1]$ is uniquely $\tilde\beta^\star=\lambda$, i.e.\ $\beta^\star=0$.

\textbf{(B)} Observe that $q_\lambda(y)=\lambda p_{\mathrm{o}}(y)+(1-\lambda)p_{\mathrm{n}}(y)=\tilde p_\lambda(y)$, hence
$\KL(\tilde p_\lambda\|q_\lambda)=0$ and $\beta=\lambda$ is a global minimizer.
If $\KL(\tilde p_\lambda\|q_\beta)=0$, then $q_\beta=\tilde p_\lambda$ a.e., i.e.
\[
(\beta-\lambda)\big(p_{\mathrm{o}}(y)-p_{\mathrm{n}}(y)\big)=0\quad\text{a.e.}
\]
Since $p_{\mathrm{o}}-p_{\mathrm{n}}$ is not zero a.e., this forces $\beta=\lambda$, proving uniqueness.
\end{proof}

\subsection{Proof of Theorem~\ref{thm:kl_stationary_beta_mn}}
% \RKLN*
\begin{proof}[Proof of Theorem~\ref{thm:kl_stationary_beta_mn}]
We start with a general derivative identity for $\KL(q_\theta\|p)$ when only $q_\theta$ depends on $\theta$. Let $p$ be a fixed strictly positive density, and let $\{q_\theta:\theta\in\Theta\}$ be a $C^1$ family of densities such that $\int q_\theta(y)\,dy=1$ for all $\theta$ and differentiation under the integral is justified. Then \begin{equation}\label{eq:general-kl-deriv} \nabla_\theta \KL(q_\theta\|p) =\int_{\mathbb{R}^d} (\nabla_\theta q_\theta(y))\,\log\frac{q_\theta(y)}{p(y)}\,dy. \end{equation} Indeed, writing $\KL(q_\theta\|p)=\int q_\theta\log(q_\theta/p)$ and differentiating, \[ \nabla_\theta \KL(q_\theta\|p) =\int (\nabla_\theta q_\theta)\log\frac{q_\theta}{p}\,dy +\int q_\theta\,\nabla_\theta\!\left(\log\frac{q_\theta}{p}\right)\,dy =\int (\nabla_\theta q_\theta)\log\frac{q_\theta}{p}\,dy + \int \nabla_\theta q_\theta\,dy. \] The last integral vanishes because $\int \nabla_\theta q_\theta\,dy = \nabla_\theta \int q_\theta\,dy=\nabla_\theta 1=0$, yielding \eqref{eq:general-kl-deriv}. For the present Gaussian-mixture family, the required interchange of derivative and integral holds by dominated convergence: $\nabla_\theta q_{\beta,m_{\mathrm{n}}}(y)$ is a Gaussian density times a polynomial in $y$, while $\log(q_{\beta,m_{\mathrm{n}}}(y)/p_\alpha(y))$ grows at most quadratically in $\|y\|$ because both numerator and denominator are mixtures of equal-covariance Gaussians; hence the integrand is dominated by an integrable function. 

\medskip
We next compute $\partial_\beta q_{\beta,m_{\mathrm{n}}}$ and $\nabla_{m_{\mathrm{n}}} q_{\beta,m_{\mathrm{n}}}$. By definition, $ q_{\beta,m_{\mathrm{n}}}(y)=\beta\,\gauss{y}{\mu_{\mathrm{o}}}{\Sigma}+(1-\beta)\,\gauss{y}{m_{\mathrm{n}}}{\Sigma},$ so \[ \frac{\partial}{\partial\beta}q_{\beta,m_{\mathrm{n}}}(y)=\gauss{y}{\mu_{\mathrm{o}}}{\Sigma}-\gauss{y}{m_{\mathrm{n}}}{\Sigma}. \] Also, using the standard Gaussian derivative identity $\nabla_{m}\gauss{y}{m}{\Sigma}=\gauss{y}{m}{\Sigma}\,\Sigma^{-1}(y-m)$, we have \[ \nabla_{m_{\mathrm{n}}} q_{\beta,m_{\mathrm{n}}}(y) =(1-\beta)\,\gauss{y}{m_{\mathrm{n}}}{\Sigma}\,\Sigma^{-1}(y-m_{\mathrm{n}}). \] Substituting these into \eqref{eq:general-kl-deriv} yields \eqref{eq:beta-grad-kl}--\eqref{eq:mn-grad-kl}. 

\medskip 
Evaluating the aforementioned partials at $(\beta,m_{\mathrm{n}})=(\alpha,\mu_{\mathrm{n}})$, we have pointwise equality of densities: \[ q_{\alpha,\mu_{\mathrm{n}}}(y) =\alpha\,\gauss{y}{\mu_{\mathrm{o}}}{\Sigma}+(1-\alpha)\,\gauss{y}{\mu_{\mathrm{n}}}{\Sigma} =p_\alpha(y) \quad\text{for all }y, \] hence $\log\!\big(q_{\alpha,\mu_{\mathrm{n}}}(y)/p_\alpha(y)\big)=\log 1=0$ for all $y$. Plugging this into \eqref{eq:beta-grad-kl}--\eqref{eq:mn-grad-kl} yields $\partial_\beta L(\alpha,\mu_{\mathrm{n}})=0$ and $\nabla_{m_{\mathrm{n}}}L(\alpha,\mu_{\mathrm{n}})=0$. Finally, $\KL(\cdot\|\cdot)\ge 0$ always and equals $0$ iff $q=p$ a.e., so $L(\alpha,\mu_{\mathrm{n}})=0$ and $(\alpha,\mu_{\mathrm{n}})$ is a global minimizer. 
\end{proof}

\subsection{Proof of Theorem~\ref{lem:fwdKL_oldmean_drift}}
% \RKL*
\begin{proof}[Proof of Theorem~\ref{lem:fwdKL_oldmean_drift}]

Write $q(y):=q_{\beta,m_{\mathrm{o}},m_{\mathrm{n}}}(y)$ and $p(y):=p_\alpha(y)$.
A standard identity (cf.\ \eqref{eq:general-kl-deriv} later in the paper) yields
\[
\nabla_{m_{\mathrm{o}}}\KL(q\|p)=\int (\nabla_{m_{\mathrm{o}}}q(y))\,\log\frac{q(y)}{p(y)}\,dy,
\]
since $\int \nabla_{m_{\mathrm{o}}}q=0$.
Moreover $\nabla_{m_{\mathrm{o}}}q(y)=\beta\,\gauss{y}{m_{\mathrm{o}}}{\Sigma}\,\Sigma^{-1}(y-m_{\mathrm{o}})$.
Evaluating at $m_{\mathrm{o}}=\mu_{\mathrm{o}}$ gives
\[
\nabla_{m_{\mathrm{o}}}L_{\mathrm{RL}}(\beta,\mu_{\mathrm{o}},m_{\mathrm{n}})
=\beta\,\Sigma^{-1}\E_{Y\sim p_{\mathrm{o}}}\!\left[(Y-\mu_{\mathrm{o}})\,\log\frac{q(Y)}{p(Y)}\right].
\]
Apply Stein's identity (Lemma~\ref{lem:stein}) with $g(y)=\log\frac{q(y)}{p(y)}$ to obtain
\[
\nabla_{m_{\mathrm{o}}}L_{\mathrm{RL}}(\beta,\mu_{\mathrm{o}},m_{\mathrm{n}})
=\beta\,\E_{p_{\mathrm{o}}}\!\left[\nabla_y\log\frac{q(Y)}{p(Y)}\right]
=\beta\,\E_{p_{\mathrm{o}}}\!\left[\nabla_y\log q(Y)-\nabla_y\log p(Y)\right].
\]
For equal-covariance Gaussian mixtures, the score is responsibility-weighted:
\[
\nabla_y\log q(y)=-\Sigma^{-1}\Big(y-\big(r_{\mathrm{o}}(y)m_{\mathrm{o}}+r_{\mathrm{n}}(y)m_{\mathrm{n}}\big)\Big),
\qquad
\nabla_y\log p(y)=-\Sigma^{-1}\Big(y-\big(s_{\mathrm{o}}(y)\mu_{\mathrm{o}}+s_{\mathrm{n}}(y)\mu_{\mathrm{n}}\big)\Big).
\]
Substituting $m_{\mathrm{o}}=\mu_{\mathrm{o}}$ and subtracting yields
\[
\nabla_y\log q(y)-\nabla_y\log p(y)
=
\Sigma^{-1}\Big((1-r_{\mathrm{o}}(y))(m_{\mathrm{n}}-\mu_{\mathrm{o}})-(1-s_{\mathrm{o}}(y))(\mu_{\mathrm{n}}-\mu_{\mathrm{o}})\Big).
\]
Taking $\E_{p_{\mathrm{o}}}$ gives \eqref{eq:kl_oldmean_exact} with $\varepsilon_q=\E_{p_{\mathrm{o}}}[1-r_{\mathrm{o}}]$ and $\varepsilon_p=\E_{p_{\mathrm{o}}}[1-s_{\mathrm{o}}]$.

Finally, \eqref{eq:kl_eps_bounds} follows from Lemma~\ref{lem:leakage} and Lemma~\ref{lem:bc-gauss}:
\begin{itemize}
\item For $\varepsilon_q=\E_{p_{\mathrm{o}}}[1-r_{\mathrm{o}}(Y)]$, apply Lemma~\ref{lem:leakage} to the model mixture
$q_{\beta,\mu_{\mathrm{o}},m_{\mathrm{n}}}$ with $w=\beta$, $f=\gauss{\cdot}{\mu_{\mathrm{o}}}{\Sigma}$, $g=\gauss{\cdot}{m_{\mathrm{n}}}{\Sigma}$.
\item For $\varepsilon_p=\E_{p_{\mathrm{o}}}[1-s_{\mathrm{o}}(Y)]$, apply Lemma~\ref{lem:leakage} to the target mixture
$p_\alpha$ with $w=\alpha$, $f=p_{\mathrm{o}}$, $g=p_{\mathrm{n}}$.
\end{itemize}

\end{proof}

\subsection{Proof of Theorem~\ref{thm:kl_RL_local_rate}}
% \PL*

We start by showing the Lipschitiz continuity of the Hessian.

\begin{lemma}[Local Hessian-Lipschitzness of the reverse-KL objective]\label{lem:hessian_lipschitz_reverseKL}
Let
\[
L(\phi,m):=\KL(q_{\phi,m}\|p_\alpha),
\qquad
q_{\phi,m}(y):=\beta(\phi)\,\phi_\Sigma(y;\mu_{\mathrm{o}})+(1-\beta(\phi))\,\phi_\Sigma(y;m),
\qquad
\beta(\phi)=\frac{1}{1+e^{-\phi}},
\]
with \(m_{\mathrm{o}}=\mu_{\mathrm{o}}\) fixed.
Fix any compact set
\[
K:=\{(\phi,m):\ |\phi-\phi^\star|\le r,\ \|m-m^\star\|\le r\}
\subset \R\times\R^d,
\]
where \(\phi^\star=\log\frac{\alpha}{1-\alpha}\) and \(m^\star=\mu_{\mathrm{n}}\).
Then \(L\) is \(C^3\) on \(K\).
Consequently, there exists a finite constant \(L_H(K)<\infty\) such that
\[
\|\nabla^2 L(\theta)-\nabla^2 L(\theta')\|_2
\le L_H(K)\,\|\theta-\theta'\|
\qquad
\forall \theta,\theta'\in K.
\]
In particular, the Hessian-Lipschitz assumption used to quantify the local PL region holds on every compact neighborhood bounded away from \(\beta\in\{0,1\}\).
\end{lemma}

\begin{proof}[Proof of Theorem~\ref{thm:kl_RL_local_rate}]
Write \(\theta=(\phi,m)\) and \(q_\theta=q_{\phi,m}\).
We first show that derivatives of the integrand in the KL objective admit a uniform integrable envelope on \(K\).

Since \(K\) is compact and \(\beta(\phi)\) is continuous, there exist constants
\[
0<\underline{\beta}\le \beta(\phi)\le \overline{\beta}<1
\qquad
\forall (\phi,m)\in K.
\]
Moreover, the set of means \(\{m:\ (\phi,m)\in K\}\) is compact.
For each multi-index \(\nu\) with \(|\nu|\le 3\), the derivatives \(\partial_\theta^\nu q_\theta(y)\) are finite linear combinations of terms of the form
\[
P_{\nu}(y,m)\,\gauss{y}{\mu_{\mathrm{o}}}{\Sigma}
\qquad\text{or}\qquad
Q_{\nu}(y,m)\,\gauss{y}{m}{\Sigma},
\]
where \(P_\nu,Q_\nu\) are polynomials in \(y\) whose coefficients depend continuously on \(m\).
Since \(m\) ranges over a compact set, there exist constants \(C_\nu,c_\nu>0\) such that
\begin{equation}\label{eq:q_deriv_envelope}
|\partial_\theta^\nu q_\theta(y)|
\le
C_\nu(1+\|y\|^{3})e^{-c_\nu \|y\|^2}
\qquad
\forall \theta\in K,\ \forall y\in\R^d,\ \forall |\nu|\le 3.
\end{equation}

We next control the logarithmic factor.
Because \(\underline{\beta}>0\), we have the pointwise lower bound
\[
q_\theta(y)\ge \underline{\beta}\,\gauss{y}{\mu_{\mathrm{o}}}{\Sigma}
\qquad\forall \theta\in K,\ \forall y\in\R^d.
\]
Similarly, \(p_\alpha(y)\ge \alpha\,\gauss{y}{\mu_{\mathrm{o}}}{\Sigma}\).
For equal-covariance Gaussian mixtures, the quadratic terms in \(\log q_\theta(y)\) and \(\log p_\alpha(y)\) cancel, and the remaining difference grows at most linearly in \(\|y\|\).
Hence there exists a constant \(C_{\log}>0\) such that
\begin{equation}\label{eq:log_envelope}
\sup_{\theta\in K}\left|\log\frac{q_\theta(y)}{p_\alpha(y)}\right|
\le
C_{\log}(1+\|y\|)
\qquad
\forall y\in\R^d.
\end{equation}

Now write the KL objective as
\[
L(\theta)=\int_{\R^d} q_\theta(y)\,\log\frac{q_\theta(y)}{p_\alpha(y)}\,dy.
\]
Differentiating with respect to \(\theta\) up to third order produces finite sums of products of derivatives of \(q_\theta\), powers of \(q_\theta^{-1}\), and the factor \(\log(q_\theta/p_\alpha)\).
Using the lower bound \(q_\theta(y)\ge \underline{\beta}\,\gauss{y}{\mu_{\mathrm{o}}}{\Sigma}\), the derivative envelope \eqref{eq:q_deriv_envelope}, and
the logarithmic bound \eqref{eq:log_envelope}, each derivative of the integrand up to order \(3\) is dominated by a function of the form
\[
C(1+\|y\|^M)e^{-c\|y\|^2}
\]
for some constants \(C,M,c>0\) independent of \(\theta\in K\).
This envelope is integrable on \(\R^d\).
Therefore differentiation under the integral sign is justified up to third order by dominated convergence, so \(L\in C^3(K)\).

Finally, since \(L\in C^3(K)\), the third derivative \(\nabla^3 L\) is continuous on the compact set \(K\), and hence bounded:
\[
M_3:=\sup_{\theta\in K}\|\nabla^3 L(\theta)\|_{\mathrm{op}}<\infty.
\]
The mean value theorem in Banach spaces then implies
\[
\|\nabla^2 L(\theta)-\nabla^2 L(\theta')\|_2
\le
M_3\,\|\theta-\theta'\|
\qquad
\forall \theta,\theta'\in K.
\]
Thus the Hessian is Lipschitz on \(K\) with \(L_H(K):=M_3\).
\end{proof}

\begin{proof}[Proof of Theorem~\ref{thm:kl_RL_local_rate}]

We first prove that $H_\star:=\nabla^2L(\theta^\star)\succ 0$ and $\mu_\star:=\lambda_{\min}(H_\star)>0.$ Write \(\theta=(\phi,m)\in\R\times\R^d\), \(q_\theta:=q_{\phi,m}\), and \(\theta^\star=(\phi^\star,m^\star)\), so that
\(q_{\theta^\star}=p_\alpha\).
Since $L(\theta)=\KL(q_\theta\|p_\alpha)=\KL(q_\theta\|q_{\theta^\star}),
$ we first derive the Fisher representation of the Hessian at \(\theta^\star\). Let
\[
\ell_\theta(y):=\log q_\theta(y),
\qquad
s_\theta(y):=\nabla_\theta \ell_\theta(y)=\nabla_\theta \log q_\theta(y).
\]
Then
\[
L(\theta)=\int q_\theta(y)\big(\ell_\theta(y)-\ell_{\theta^\star}(y)\big)\,dy.
\]
Differentiating with respect to \(\theta\), and using \(\nabla_\theta q_\theta=q_\theta s_\theta\), gives
\[
\nabla_\theta L(\theta)
=
\int q_\theta(y)s_\theta(y)\big(\ell_\theta(y)-\ell_{\theta^\star}(y)\big)\,dy
+
\int q_\theta(y)\nabla_\theta \ell_\theta(y)\,dy.
\]
Since \(q_\theta\nabla_\theta \ell_\theta=\nabla_\theta q_\theta\) and \(\int q_\theta(y)\,dy=1\), the second term vanishes:
\[
\int q_\theta(y)\nabla_\theta \ell_\theta(y)\,dy
=
\int \nabla_\theta q_\theta(y)\,dy
=
\nabla_\theta\int q_\theta(y)\,dy
=
\nabla_\theta 1
=
0.
\]
Hence
\[
\nabla_\theta L(\theta)
=
\int q_\theta(y)s_\theta(y)\big(\ell_\theta(y)-\ell_{\theta^\star}(y)\big)\,dy.
\]
In particular, at \(\theta=\theta^\star\), the logarithmic factor vanishes pointwise, so
\[
\nabla_\theta L(\theta^\star)=0.
\]
We next differentiate once more.
Set
\[
h_\theta(y):=\ell_\theta(y)-\ell_{\theta^\star}(y).
\]
Then
\[
\nabla_\theta L(\theta)=\int q_\theta(y)s_\theta(y)h_\theta(y)\,dy.
\]
Differentiating and evaluating at \(\theta=\theta^\star\), every term containing the factor \(h_\theta(y)\) vanishes because
\(h_{\theta^\star}(y)=0\).
The only surviving term comes from differentiating \(h_\theta\), and since
\[
\nabla_\theta h_\theta(y)=\nabla_\theta \ell_\theta(y)=s_\theta(y),
\]
we obtain
\[
\nabla_\theta^2L(\theta^\star)
=
\int q_{\theta^\star}(y)\,s_{\theta^\star}(y)s_{\theta^\star}(y)^\top\,dy
=
\E_{Y\sim q_{\theta^\star}}\!\big[s_{\theta^\star}(Y)s_{\theta^\star}(Y)^\top\big].
\]
Because \(q_{\theta^\star}=p_\alpha\), this yields the Fisher representation
\[
H_\star
=
\nabla^2L(\theta^\star)
=
\E_{Y\sim p_\alpha}\!\big[s(Y)s(Y)^\top\big],
\]
where \(s(Y)=s_{\theta^\star}(Y)\).

For the present two-component model, the score vector is
\[
s(Y)=
\begin{pmatrix}
r_{\mathrm{o}}^\star(Y)-\alpha\\[2pt]
r_{\mathrm{n}}^\star(Y)\,\Sigma^{-1}(Y-\mu_{\mathrm{n}})
\end{pmatrix},
\qquad
r_{\mathrm{o}}^\star(y):=\frac{\alpha\,\phi_\Sigma(y;\mu_{\mathrm{o}})}{p_\alpha(y)},
\qquad
r_{\mathrm{n}}^\star(y):=1-r_{\mathrm{o}}^\star(y).
\]

We now prove that \(H_\star\) is positive definite.
Let \(v=(u,a)\in\R\times\R^d\). Using the Fisher representation,
\[
v^\top H_\star v
=
\E_{Y\sim p_\alpha}\!\left[
\Big(
u\big(r_{\mathrm{o}}^\star(Y)-\alpha\big)
+
a^\top r_{\mathrm{n}}^\star(Y)\Sigma^{-1}(Y-\mu_{\mathrm{n}})
\Big)^2
\right].
\]
Define
\[
g_v(y)
:=
u\big(r_{\mathrm{o}}^\star(y)-\alpha\big)
+
a^\top r_{\mathrm{n}}^\star(y)\Sigma^{-1}(y-\mu_{\mathrm{n}}).
\]
Then
\[
v^\top H_\star v=\E_{Y\sim p_\alpha}[g_v(Y)^2]\ge 0.
\]

Suppose now that \(v^\top H_\star v=0\). Then \(g_v(Y)=0\) for \(p_\alpha\)-almost every \(Y\).
Since \(p_\alpha\) is a strictly positive continuous density on \(\R^d\), every nonempty open set has positive \(p_\alpha\)-measure.
Because \(g_v\) is continuous, if there existed \(y_0\in\R^d\) with \(g_v(y_0)\neq 0\), then by continuity there would exist an open neighborhood
\(U\ni y_0\) on which \(g_v\) is bounded away from \(0\), implying
\[
\E_{Y\sim p_\alpha}[g_v(Y)^2]
\ge
\int_U g_v(y)^2p_\alpha(y)\,dy
>
0,
\]
a contradiction.
Therefore
\[
g_v(y)=0
\qquad
\forall y\in\R^d.
\]

We first show that \(u=0\).
Evaluating at \(y=\mu_{\mathrm{n}}\), the second term vanishes, so
\[
0=g_v(\mu_{\mathrm{n}})=u\big(r_{\mathrm{o}}^\star(\mu_{\mathrm{n}})-\alpha\big).
\]
Now
\[
r_{\mathrm{o}}^\star(\mu_{\mathrm{n}})
=
\frac{\alpha\,\phi_\Sigma(\mu_{\mathrm{n}};\mu_{\mathrm{o}})}
{\alpha\,\phi_\Sigma(\mu_{\mathrm{n}};\mu_{\mathrm{o}})+(1-\alpha)\phi_\Sigma(\mu_{\mathrm{n}};\mu_{\mathrm{n}})}.
\]
Since \(\mu_{\mathrm{o}}\neq \mu_{\mathrm{n}}\),
\[
\phi_\Sigma(\mu_{\mathrm{n}};\mu_{\mathrm{o}})
<
\phi_\Sigma(\mu_{\mathrm{n}};\mu_{\mathrm{n}}),
\]
hence \(r_{\mathrm{o}}^\star(\mu_{\mathrm{n}})<\alpha\), so \(r_{\mathrm{o}}^\star(\mu_{\mathrm{n}})-\alpha\neq 0\).
Therefore
\[
u=0.
\]

With \(u=0\), the identity \(g_v(y)=0\) becomes
\[
a^\top r_{\mathrm{n}}^\star(y)\Sigma^{-1}(y-\mu_{\mathrm{n}})=0
\qquad
\forall y\in\R^d.
\]
Because \(r_{\mathrm{n}}^\star(y)>0\) for all \(y\in\R^d\) (both Gaussian components are strictly positive and \(1-\alpha>0\)),
we may divide by \(r_{\mathrm{n}}^\star(y)\) and obtain
\[
a^\top \Sigma^{-1}(y-\mu_{\mathrm{n}})=0
\qquad
\forall y\in\R^d.
\]
Taking \(y=\mu_{\mathrm{n}}+\Sigma a\), we get
\[
0=a^\top \Sigma^{-1}(\Sigma a)=a^\top a=\|a\|^2,
\]
so \(a=0\).

Thus \(v=(u,a)=0\) is the only vector satisfying \(v^\top H_\star v=0\). Therefore \(H_\star\) is positive definite:
\[
H_\star\succ 0.
\]
Consequently, $\mu_\star:=\lambda_{\min}(H_\star)>0.$

We next prove the explicit lower bound on the Hessian inside the ball \(B_\rho(\theta^\star)\).
Fix \(\theta\in B_\rho(\theta^\star)\).
By Weyl's inequality,
\[
\lambda_{\min}\!\big(\nabla^2L(\theta)\big)
\ge
\lambda_{\min}\!\big(\nabla^2L(\theta^\star)\big)
-
\|\nabla^2L(\theta)-\nabla^2L(\theta^\star)\|_2.
\]
Since \(\lambda_{\min}(\nabla^2L(\theta^\star))=\mu_\star\) and the Hessian is \(L_H\)-Lipschitz on \(K\), we obtain
\[
\lambda_{\min}\!\big(\nabla^2L(\theta)\big)
\ge
\mu_\star-L_H\|\theta-\theta^\star\|.
\]
Because \(\theta\in B_\rho(\theta^\star)\) and \(\rho\le \mu_\star/(2L_H)\),
\[
L_H\|\theta-\theta^\star\|\le L_H\rho\le \frac{\mu_\star}{2},
\]
hence
\[
\lambda_{\min}\!\big(\nabla^2L(\theta)\big)\ge \mu_\star-\frac{\mu_\star}{2}=\frac{\mu_\star}{2}.
\]
This proves \eqref{eq:local_hessian_lower}.

We next derive the two local inequalities in \eqref{eq:local_PL_final}.
Fix \(\theta\in B_\rho(\theta^\star)\) and write \(d:=\theta-\theta^\star\).
Since \(\nabla L(\theta^\star)=0\), Taylor's theorem with integral remainder gives
\[
L(\theta)-L(\theta^\star)
=
\int_0^1 (1-s)\, d^\top \nabla^2L(\theta^\star+s d)\,d\,ds.
\]
Because the whole line segment \(\theta^\star+s d\) lies in \(B_\rho(\theta^\star)\), the Hessian lower bound \eqref{eq:local_hessian_lower} applies throughout the segment, so
\[
L(\theta)
=
\int_0^1 (1-s)\, d^\top \nabla^2L(\theta^\star+s d)\,d\,ds
\ge
\int_0^1 (1-s)\,\frac{\mu_\star}{2}\|d\|^2\,ds
=
\frac{\mu_\star}{4}\|d\|^2.
\]
This proves the quadratic-growth inequality.

To prove the PL inequality, we use strong convexity in the form
\[
L(\theta^\star)\ge L(\theta)+\ip{\nabla L(\theta)}{\theta^\star-\theta}+\frac{\mu_\star}{4}\|\theta^\star-\theta\|^2,
\]
which holds because \(L\) is \(\mu_\star/2\)-strongly convex on \(B_\rho(\theta^\star)\).
Since \(L(\theta^\star)=0\), this becomes
\[
L(\theta)\le \ip{\nabla L(\theta)}{\theta-\theta^\star}-\frac{\mu_\star}{4}\|\theta-\theta^\star\|^2.
\]
By Cauchy--Schwarz,
\[
L(\theta)\le \|\nabla L(\theta)\|\,\|\theta-\theta^\star\|-\frac{\mu_\star}{4}\|\theta-\theta^\star\|^2.
\]
The right-hand side is a quadratic function of \(t:=\|\theta-\theta^\star\|\), namely
\[
\|\nabla L(\theta)\|\,t-\frac{\mu_\star}{4}t^2,
\]
whose maximum over \(t\ge 0\) is attained at \(t=2\|\nabla L(\theta)\|/\mu_\star\) and equals
\[
\frac{1}{\mu_\star}\|\nabla L(\theta)\|^2.
\]
Therefore
\[
L(\theta)\le \frac{1}{\mu_\star}\|\nabla L(\theta)\|^2,
\]
which is equivalent to
\[
\|\nabla L(\theta)\|^2\ge \mu_\star L(\theta).
\]
This proves the local PL bound.

We now turn to the gradient-flow estimates.
Let \(\theta(t)\) solve \(\dot\theta(t)=-\nabla L(\theta(t))\).
By the chain rule,
\[
\frac{d}{dt}L(\theta(t))
=
\ip{\nabla L(\theta(t))}{\dot\theta(t)}
=
-\|\nabla L(\theta(t))\|^2
\le 0.
\]
Thus \(L(\theta(t))\) is nonincreasing along the flow.

We next prove that the trajectory stays inside \(B_\rho(\theta^\star)\).
Assume \(\theta(0)\in B_\rho(\theta^\star)\) and \(L(\theta(0))\le \varepsilon_{\mathrm{loc}}=\mu_\star\rho^2/8\).
On the boundary of the ball, the quadratic-growth bound gives
\[
L(\theta)\ge \frac{\mu_\star}{4}\rho^2=2\varepsilon_{\mathrm{loc}}
\qquad
\text{whenever }\|\theta-\theta^\star\|=\rho.
\]
Since \(L(\theta(t))\le L(\theta(0))\le \varepsilon_{\mathrm{loc}}\) for all \(t\ge 0\), the trajectory can never reach a point with
\(\|\theta(t)-\theta^\star\|=\rho\).
Hence \(\theta(t)\in B_\rho(\theta^\star)\) for all \(t\ge 0\).

Because the entire trajectory remains in \(B_\rho(\theta^\star)\), the local PL inequality applies for all \(t\ge 0\).
Combining it with the energy identity yields
\[
\frac{d}{dt}L(\theta(t))
=
-\|\nabla L(\theta(t))\|^2
\le
-\mu_\star L(\theta(t)).
\]
Grönwall's inequality then gives
\[
L(\theta(t))\le L(\theta(0))\,e^{-\mu_\star t},
\]
which is \eqref{eq:exp_loss_final}.

Finally, the quadratic-growth inequality implies
\[
L(\theta(t))\ge \frac{\mu_\star}{4}\|\theta(t)-\theta^\star\|^2.
\]
Combining this with \eqref{eq:exp_loss_final} yields
\[
\frac{\mu_\star}{4}\|\theta(t)-\theta^\star\|^2
\le
L(\theta(0))\,e^{-\mu_\star t},
\]
and therefore
\[
\|\theta(t)-\theta^\star\|
\le
\frac{2}{\sqrt{\mu_\star}}\sqrt{L(\theta(0))}\,e^{-\mu_\star t/2},
\]
which is \eqref{eq:exp_param_final}.
\end{proof}

\subsection{Proof of Lemma~\ref{lem:replay_reverseKL_consolidated}}
% \ReplayRKL*
\begin{proof}[Proof of Lemma~\ref{lem:replay_reverseKL_consolidated}] We establish both parts of the statement in order. 

\paragraph{Part (A).}
First expand the replay mixture using $q_{\beta,m_{\mathrm{n}}}=\beta p_{\mathrm{o}}+(1-\beta)\gauss{\cdot}{m_{\mathrm{n}}}{\Sigma}$:
\[
b_{\lambda,\beta,m_{\mathrm{n}}}
=(1-\lambda)\big(\beta p_{\mathrm{o}}+(1-\beta)\gauss{\cdot}{m_{\mathrm{n}}}{\Sigma}\big)+\lambda p_{\mathrm{o}}
=\big(\lambda+(1-\lambda)\beta\big)p_{\mathrm{o}}+(1-\lambda)(1-\beta)\gauss{\cdot}{m_{\mathrm{n}}}{\Sigma}.
\]
Set $\tilde\beta:=\lambda+(1-\lambda)\beta\in(0,1)$, so that $1-\tilde\beta=(1-\lambda)(1-\beta)$, which proves
$b_{\lambda,\beta,m_{\mathrm{n}}}=q_{\tilde\beta,m_{\mathrm{n}}}$ and $\tilde\beta\ge\lambda$.

Next, observe that $b_{\lambda,\beta,m_{\mathrm{n}}}=(1-\lambda)q_{\beta,m_{\mathrm{n}}}+\lambda p_{\mathrm{o}}\ge (1-\lambda)q_{\beta,m_{\mathrm{n}}}$
pointwise. Therefore for all $y$,
\[
0\le w_\lambda(y)=\frac{q_{\beta,m_{\mathrm{n}}}(y)}{b_{\lambda,\beta,m_{\mathrm{n}}}(y)}
\le \frac{q_{\beta,m_{\mathrm{n}}}(y)}{(1-\lambda)q_{\beta,m_{\mathrm{n}}}(y)}=\frac{1}{1-\lambda}.
\]

For unbiasedness, let $h$ be integrable under $q_{\beta,m_{\mathrm{n}}}$ (equivalently, $w_\lambda h$ integrable under $b_{\lambda,\beta,m_{\mathrm{n}}}$).
Then
\[
\E_{b_{\lambda,\beta,m_{\mathrm{n}}}}\big[w_\lambda(Y)h(Y)\big]
=\int b_{\lambda,\beta,m_{\mathrm{n}}}(y)\frac{q_{\beta,m_{\mathrm{n}}}(y)}{b_{\lambda,\beta,m_{\mathrm{n}}}(y)}h(y)\,dy
=\int q_{\beta,m_{\mathrm{n}}}(y)h(y)\,dy
=\E_{q_{\beta,m_{\mathrm{n}}}}\big[h(Y)\big],
\]
which is \eqref{eq:unbiased_general_consolidated}.  Finally, if $\E_b[\|h(Y)\|^2]<\infty$, then using the uniform bound $w_\lambda^2\le (1-\lambda)^{-2}$,
\[
\E_{b_{\lambda,\beta,m_{\mathrm{n}}}}\big[\|w_\lambda(Y)h(Y)\|^2\big]
\le \frac{1}{(1-\lambda)^2}\,\E_{b_{\lambda,\beta,m_{\mathrm{n}}}}\big[\|h(Y)\|^2\big],
\]
which is \eqref{eq:second_moment_bound_consolidated}.

\paragraph{Part (B).}
Under $b_{\lambda,\beta,m_{\mathrm{n}}}=q_{\tilde\beta,m_{\mathrm{n}}}$, the standard mixture generative model yields latent indicators
$Z_i\overset{i.i.d.}{\sim}\mathrm{Bernoulli}(\tilde\beta)$ with $\tilde\beta\ge\lambda$.  Therefore
\[
\Pr\Big(\sum_{i=1}^N Z_i=0\Big)=(1-\tilde\beta)^N,
\]
and substituting $1-\tilde\beta=(1-\lambda)(1-\beta)$ gives \eqref{eq:no_old_prob_consolidated} and the upper bound $(1-\tilde\beta)^N\le (1-\lambda)^N$.

For \eqref{eq:chernoff_oldcount_consolidated}, let $S:=\sum_{i=1}^N Z_i\sim\mathrm{Binomial}(N,\tilde\beta)$, so $\E[S]=N\tilde\beta\ge N\lambda$.
The multiplicative Chernoff bound implies
\[
\Pr\big(S\le (1-\delta)\E[S]\big)\le \exp\!\Big(-\frac{\delta^2}{2}\E[S]\Big).
\]
Taking $\delta=1/2$ yields $\Pr(S\le \tfrac12 N\tilde\beta)\le \exp(-N\tilde\beta/8)\le \exp(-N\lambda/8)$.
Since $\tfrac{\lambda}{2}N\le \tfrac{\tilde\beta}{2}N$, we have
\[
\Pr\Big(S\le \tfrac{\lambda}{2}N\Big)\le \Pr\Big(S\le \tfrac{\tilde\beta}{2}N\Big)\le \exp\!\Big(-\tfrac{\lambda}{8}N\Big),
\]
which is \eqref{eq:chernoff_oldcount_consolidated}.
\end{proof}

\subsection{Proof of Theorem~\ref{thm:sdft_demo_ema}}
\begin{proof}[Proof of Theorem~\ref{thm:sdft_demo_ema}]
Fix \(y=(\alpha,\nu)\in K\) and write \(F_y(x)=\KL(q_x\|p_y)\) for \(x=(\beta,m)\).
Because \(q_y\equiv p_y\), we have \(F_y(y)=0\), and since KL is nonnegative, \(y\) is a global minimizer of \(F_y\).
As established earlier for equal-covariance Gaussian mixtures, \(F_y\) is \(C^2\) in a neighborhood of \(y\), and the Hessian at \(y\) equals the Fisher
information of the parameterization \(x\mapsto q_x\) under \(Y\sim p_y\), hence \(\nabla^2F_y(y)\succ 0\).
By continuity of \(y\mapsto \nabla^2F_y(y)\) and compactness of \(K\), there exists \(\mu>0\) with
\(\lambda_{\min}(\nabla^2F_y(y))\ge \mu\) uniformly over \(y\in K\).
Similarly, by smooth dependence of \(F_y\) on \((x,y)\) and compactness, there exist \(r_0>0\) and \(L_H<\infty\) such that the Hessian is \(L_H\)-Lipschitz on
\(B_{r_0}(y)\) uniformly in \(y\in K\).
Define \(\rho=\min\{r_0,\mu/(2L_H)\}\).
Then for any \(y\in K\) and any \(x\in B_\rho(y)\), Weyl's inequality gives
\[
\lambda_{\min}\!\big(\nabla^2F_y(x)\big)
\ge
\lambda_{\min}\!\big(\nabla^2F_y(y)\big)
-
\|\nabla^2F_y(x)-\nabla^2F_y(y)\|_2
\ge
\mu-L_H\|x-y\|
\ge
\frac{\mu}{2}.
\]
Since the set \(\{(x,y): y\in K,\ \|x-y\|\le \rho\}\) is compact and \((x,y)\mapsto \nabla^2F_y(x)\) is continuous, there exists
\(M<\infty\) such that \(\lambda_{\max}(\nabla^2F_y(x))\le M\) uniformly on that set.

We now prove Part (A).
Fix \(t\ge 0\) and write \(y_t=\widetilde{\nu}_t\) and \(x_t=\widetilde{m}_t\).
Assume inductively that \(x_t\in B_\rho(y_t)\).
Since \(y_t\) is the minimizer of \(F_{y_t}\), we have \(\nabla F_{y_t}(y_t)=0\).
The mean-value formula for gradients yields
\[
\nabla F_{y_t}(x_t)-\nabla F_{y_t}(y_t)
=
\left(\int_0^1 \nabla^2F_{y_t}\big(y_t+s(x_t-y_t)\big)\,ds\right)(x_t-y_t).
\]
Define
\[
A_t:=\int_0^1 \nabla^2F_{y_t}\big(y_t+s(x_t-y_t)\big)\,ds.
\]
Because the segment \(\{y_t+s(x_t-y_t): s\in[0,1]\}\subset B_\rho(y_t)\), we have
\[
\lambda_{\min}(A_t)\ge \frac{\mu}{2},
\qquad
\lambda_{\max}(A_t)\le M.
\]
Using the student update \(\,x_{t+1}=x_t-\gamma\nabla F_{y_t}(x_t)\,\) with \(0<\gamma\le 1/M\) and \(\nabla F_{y_t}(y_t)=0\),
\[
x_{t+1}-y_t
=
x_t-y_t-\gamma(\nabla F_{y_t}(x_t)-\nabla F_{y_t}(y_t))
=
(I-\gamma A_t)(x_t-y_t).
\]
Since \(A_t\) is symmetric with spectrum in \([\mu/2,M]\), we have
\[
\|I-\gamma A_t\|_2=\max_{\lambda\in[\mu/2,M]}|1-\gamma\lambda|
\le 1-\frac{\gamma\mu}{2}
=:q,
\]
which proves \eqref{eq:sdft_contract_to_current_teacher}.

We next derive \eqref{eq:sdft_teacher_to_demo_recursion} from the teacher update.
Let \(\widetilde{\nu}(c)\) be fixed.
From \eqref{eq:sdft_teacher_update},
\[
\widetilde{\nu}_{t+1}-\widetilde{\nu}(c)
=
(1-\zeta)\big(\widetilde{\nu}_t-\widetilde{\nu}(c)\big)
+\zeta(1-\lambda)\big(\widetilde{m}_{t+1}-\widetilde{\nu}(c)\big).
\]
Rewrite \(\widetilde{m}_{t+1}-\widetilde{\nu}(c)=(\widetilde{m}_{t+1}-\widetilde{\nu}_t)+(\widetilde{\nu}_t-\widetilde{\nu}(c))\) to get
\[
\widetilde{\nu}_{t+1}-\widetilde{\nu}(c)
=
(1-\zeta\lambda)\big(\widetilde{\nu}_t-\widetilde{\nu}(c)\big)
+\zeta(1-\lambda)\big(\widetilde{m}_{t+1}-\widetilde{\nu}_t\big).
\]
Taking norms yields \eqref{eq:sdft_teacher_to_demo_recursion}.
If \(\lambda>0\), then \(1-\zeta\lambda\in(0,1)\) and \eqref{eq:sdft_contract_to_current_teacher} implies
\(\|\widetilde{m}_{t+1}-\widetilde{\nu}_t\|\to 0\); thus the recursion \eqref{eq:sdft_teacher_to_demo_recursion} implies
\(\|\widetilde{\nu}_{t}-\widetilde{\nu}(c)\|\to 0\), and hence also \(\|\widetilde{m}_t-\widetilde{\nu}(c)\|\to 0\).

We now prove Part (B).
Define
\[
G(\beta,m;\alpha,\nu)
:=
\nabla_{m_{\mathrm{o}}}\widetilde L(\beta,\mu_{\mathrm{o}},m;\alpha,\nu).
\]
At the teacher-matched point \((\beta,m)=(\alpha,\nu)\), we have
\(q_{\alpha,\mu_{\mathrm{o}},\nu}\equiv p_{\alpha,\nu}\), hence \(\widetilde L(\alpha,\mu_{\mathrm{o}},\nu;\alpha,\nu)=0\), and differentiating shows
\(G(\alpha,\nu;\alpha,\nu)=0\).
The map \((\beta,m,\alpha,\nu)\mapsto G(\beta,m;\alpha,\nu)\) is continuous and \(C^1\) on the compact set
\[
\mathcal{C}:=\Big\{(\beta,m,\alpha,\nu):\ (\alpha,\nu)\in K,\ \|(\beta,m)-(\alpha,\nu)\|\le \rho\Big\},
\]
so its Jacobian with respect to \((\beta,m,\alpha,\nu)\) is bounded on \(\mathcal{C}\).
Thus there exists \(L_{\mathrm{old}}<\infty\) such that for all \((\beta,m,\alpha,\nu)\in\mathcal{C}\),
\[
\|G(\beta,m;\alpha,\nu)-G(\alpha,\nu;\alpha,\nu)\|
\le
L_{\mathrm{old}}\Big(\|(\beta,m)-(\alpha,\nu)\|+\|(\alpha,\nu)-\widetilde{\nu}(c)\|\Big).
\]
Using \(G(\alpha,\nu;\alpha,\nu)=0\) and substituting \((\beta,m,\alpha,\nu)=(\beta_t,m_t,\alpha_t,\nu_t)\) yields
\eqref{eq:sdft_old_grad_lipschitz_track}.
If \(\lambda>0\), then \(\|\widetilde{m}_t-\widetilde{\nu}_t\|\to 0\) and \(\|\widetilde{\nu}_t-\widetilde{\nu}(c)\|\to 0\), so the right-hand side of
\eqref{eq:sdft_old_grad_lipschitz_track} is summable, proving \eqref{eq:sdft_old_grad_summable}.

Finally, since \(\widetilde{m}_t=(\beta_t,m_t)\to \widetilde{\nu}(c)=(\alpha_c,\nu_c)\), let \(\widetilde{\nu}^\star\in\R^{d+1}\) be any target state and note that
for all \(t\),
\[
\big\|\widetilde{m}_t-\widetilde{\nu}^\star\big\|
\le
\big\|\widetilde{m}_t-\widetilde{\nu}(c)\big\|
+
\big\|\widetilde{\nu}(c)-\widetilde{\nu}^\star\big\|,
\]
and also, by the reverse triangle inequality,
\[
\big\|\widetilde{m}_t-\widetilde{\nu}^\star\big\|
\ge
\big\|\widetilde{\nu}(c)-\widetilde{\nu}^\star\big\|
-
\big\|\widetilde{m}_t-\widetilde{\nu}(c)\big\|.
\]
Taking \(\limsup\) in the first inequality and \(\liminf\) in the second, and using
\(\|\widetilde{m}_t-\widetilde{\nu}(c)\|\to 0\), yields the limit identity \eqref{eq:sdft_target_error_limit}.
\end{proof}

\subsection{Proof of Lemma~\ref{lem:ttt_disjoint}}

\begin{proof}[Proof of Lemma~\ref{lem:ttt_disjoint}]
We prove the three statements in order.

\paragraph{Case 1: Compute $J_\eta(q_\beta)$ and $\KL(q_\beta\|q_{\beta_0})$.}
Under the disjoint-support assumption, $q_\beta=\beta p_{\mathrm{o}}$ on $A_{\mathrm{o}}$ and $q_\beta=(1-\beta)p_{\mathrm{n}}$ on $A_{\mathrm{n}}$.
Since the reward is constant on each region,
\begin{align*}
\E_{Y\sim q_\beta}[e^{\eta r(Y)}]
&=
\int_{A_{\mathrm{o}}} q_\beta(y)e^{\eta u_{\mathrm{o}}}\,dy
+
\int_{A_{\mathrm{n}}} q_\beta(y)e^{\eta u_{\mathrm{n}}}\,dy\\
&=
\beta e^{\eta u_{\mathrm{o}}}\int_{A_{\mathrm{o}}}p_{\mathrm{o}}(y)\,dy
+
(1-\beta)e^{\eta u_{\mathrm{n}}}\int_{A_{\mathrm{n}}}p_{\mathrm{n}}(y)\,dy\\
&=
\beta e^{\eta u_{\mathrm{o}}}+(1-\beta)e^{\eta u_{\mathrm{n}}},
\end{align*}
which proves the first identity in \eqref{eq:ttt_disjoint_J_short} after taking logarithms.

For the KL term, on $A_{\mathrm{o}}$ we have
\[
\log\frac{q_\beta(y)}{q_{\beta_0}(y)}
=
\log\frac{\beta p_{\mathrm{o}}(y)}{\beta_0 p_{\mathrm{o}}(y)}
=
\log\frac{\beta}{\beta_0},
\]
and on $A_{\mathrm{n}}$ we have
\[
\log\frac{q_\beta(y)}{q_{\beta_0}(y)}
=
\log\frac{(1-\beta)p_{\mathrm{n}}(y)}{(1-\beta_0)p_{\mathrm{n}}(y)}
=
\log\frac{1-\beta}{1-\beta_0}.
\]
Therefore
\[
\KL(q_\beta\|q_{\beta_0})
=
\int_{A_{\mathrm{o}}}\beta p_{\mathrm{o}}(y)\log\frac{\beta}{\beta_0}\,dy
+
\int_{A_{\mathrm{n}}}(1-\beta)p_{\mathrm{n}}(y)\log\frac{1-\beta}{1-\beta_0}\,dy,
\]
which simplifies to the second identity in \eqref{eq:ttt_disjoint_J_short}.

\paragraph{Case 2: Unanchored case $\lambda_{\mathrm{ref}}=0$.}
Set
\[
a:=e^{\eta u_{\mathrm{o}}},\qquad b:=e^{\eta u_{\mathrm{n}}}.
\]
Then
\[
J_\eta(q_\beta)=\log\!\big(b+\beta(a-b)\big).
\]
This is the logarithm of an affine function of $\beta$, so its monotonicity is determined by the sign of $a-b$:
if $u_{\mathrm{n}}>u_{\mathrm{o}}$, then $b>a$, so $a-b<0$ and the affine term is strictly decreasing in $\beta$, hence the unique maximizer is $\beta^\star=0$;
if $u_{\mathrm{o}}>u_{\mathrm{n}}$, the same argument gives $\beta^\star=1$; if $u_{\mathrm{o}}=u_{\mathrm{n}}$, then $a=b$ and $J_\eta$ is constant.

\paragraph{Case 3: Anchored case $\lambda_{\mathrm{ref}}>0$.}
Let
\[
F(\beta):=J_\eta(q_\beta)=\log\!\big(b+\beta(a-b)\big),
\qquad
G(\beta):=\KL(q_\beta\|q_{\beta_0}),
\qquad
H(\beta):=F(\beta)-\lambda_{\mathrm{ref}}G(\beta).
\]
For $\beta\in(0,1)$,
\[
F'(\beta)=\frac{a-b}{b+\beta(a-b)},
\qquad
F''(\beta)=-\frac{(a-b)^2}{(b+\beta(a-b))^2},
\]
and
\[
G'(\beta)=\log\frac{\beta}{\beta_0}-\log\frac{1-\beta}{1-\beta_0},
\qquad
G''(\beta)=\frac{1}{\beta}+\frac{1}{1-\beta}.
\]
If $u_{\mathrm{o}}\neq u_{\mathrm{n}}$, then $a\neq b$, so $F''(\beta)<0$ on $(0,1)$, while $G''(\beta)>0$ on $(0,1)$.
Thus
\[
H''(\beta)=F''(\beta)-\lambda_{\mathrm{ref}}G''(\beta)<0
\qquad\forall \beta\in(0,1),
\]
so $H$ is strictly concave.

To show the maximizer is interior, note that $F'(\beta)$ remains finite on $[0,1]$, whereas
\[
\lim_{\beta\downarrow 0}G'(\beta)=-\infty,
\qquad
\lim_{\beta\uparrow 1}G'(\beta)=+\infty.
\]
Hence
\[
\lim_{\beta\downarrow 0}H'(\beta)=+\infty,
\qquad
\lim_{\beta\uparrow 1}H'(\beta)=-\infty.
\]
By continuity of $H'$, there exists $\beta^\star\in(0,1)$ with $H'(\beta^\star)=0$, and by strict concavity this $\beta^\star$ is unique.

Finally, if $u_{\mathrm{o}}=u_{\mathrm{n}}$, then $F$ is constant in $\beta$, so maximizing $H$ is equivalent to minimizing $G$.
Since $G$ is strictly convex and $G'(\beta_0)=0$, its unique minimizer is $\beta_0$.
\end{proof}

\subsection{Proof of Theorem~\ref{thm:ttt_gaussian}}

\begin{proof}[Proof of Theorem~\ref{thm:ttt_gaussian}]
We prove parts (A) and (B) separately.

\paragraph{Proof of Part (A):} We start by computing the region probabilities under $p_{\mathrm{o}}$ and $p_{\mathrm{n}}$. Let $\Delta:=\mu_{\mathrm{n}}-\mu_{\mathrm{o}}$ and define
\[
T(Y):=\Delta^\top\Sigma^{-1}\Big(Y-\frac{\mu_{\mathrm{o}}+\mu_{\mathrm{n}}}{2}\Big).
\]
By definition, $A_{\mathrm{n}}=\{T(Y)\ge 0\}$ and $A_{\mathrm{o}}=\{T(Y)<0\}$.

If $Y\sim p_{\mathrm{o}}=\mathcal{N}(\mu_{\mathrm{o}},\Sigma)$, then
\[
\E[T(Y)]
=
\Delta^\top\Sigma^{-1}\Big(\mu_{\mathrm{o}}-\frac{\mu_{\mathrm{o}}+\mu_{\mathrm{n}}}{2}\Big)
=
-\frac12\,\Delta^\top\Sigma^{-1}\Delta
=
-\frac{\delta^2}{2},
\]
and
\[
\mathrm{Var}(T(Y))
=
\Delta^\top\Sigma^{-1}\Sigma\Sigma^{-1}\Delta
=
\delta^2.
\]
Therefore
\[
\Pr_{Y\sim p_{\mathrm{o}}}(A_{\mathrm{n}})
=
\Pr\Big(\mathcal{N}(-\delta^2/2,\delta^2)\ge 0\Big)
=
\Phi\!\Big(-\frac{\delta}{2}\Big)
=
\gamma.
\]
Hence
\[
\Pr_{p_{\mathrm{o}}}(A_{\mathrm{o}})=1-\gamma.
\]

Similarly, if $Y\sim p_{\mathrm{n}}=\mathcal{N}(\mu_{\mathrm{n}},\Sigma)$, then
\[
\E[T(Y)]=+\frac{\delta^2}{2},
\qquad
\mathrm{Var}(T(Y))=\delta^2,
\]
so
\[
\Pr_{p_{\mathrm{n}}}(A_{\mathrm{n}})
=
\Phi\!\Big(\frac{\delta}{2}\Big)
=
1-\gamma,
\qquad
\Pr_{p_{\mathrm{n}}}(A_{\mathrm{o}})=\gamma.
\]

We next compute $J_\eta(q_\beta)$ and its derivatives. Since $q_\beta=\beta p_{\mathrm{o}}+(1-\beta)p_{\mathrm{n}}$, the probability of $A_{\mathrm{o}}$ under $q_\beta$ is
\[
\Pr_{q_\beta}(A_{\mathrm{o}})
=
\beta(1-\gamma)+(1-\beta)\gamma
=
\gamma+\beta(1-2\gamma)
=
\gamma+\kappa\beta.
\]
Thus
\[
\Pr_{q_\beta}(A_{\mathrm{n}})=1-\gamma-\kappa\beta.
\]
Because $r(y)=u_{\mathrm{o}}$ on $A_{\mathrm{o}}$ and $r(y)=u_{\mathrm{n}}$ on $A_{\mathrm{n}}$,
\[
\E_{Y\sim q_\beta}[e^{\eta r(Y)}]
=
e^{\eta u_{\mathrm{o}}}(\gamma+\kappa\beta)
+
e^{\eta u_{\mathrm{n}}}(1-\gamma-\kappa\beta),
\]
which proves the formula for $J_\eta(q_\beta)$.

Differentiating gives
\[
J_\eta'(q_\beta)
=
\frac{\kappa(e^{\eta u_{\mathrm{o}}}-e^{\eta u_{\mathrm{n}}})}
{e^{\eta u_{\mathrm{o}}}(\gamma+\kappa\beta)+e^{\eta u_{\mathrm{n}}}(1-\gamma-\kappa\beta)},
\]
and
\[
J_\eta''(q_\beta)
=
-\frac{\kappa^2(e^{\eta u_{\mathrm{o}}}-e^{\eta u_{\mathrm{n}}})^2}
{\Big(e^{\eta u_{\mathrm{o}}}(\gamma+\kappa\beta)+e^{\eta u_{\mathrm{n}}}(1-\gamma-\kappa\beta)\Big)^2}
<0
\]
whenever $u_{\mathrm{o}}\neq u_{\mathrm{n}}$.

Next, we analyze the KL anchor. Let
\[
h(y):=p_{\mathrm{o}}(y)-p_{\mathrm{n}}(y).
\]
Then $q_\beta(y)=q_{\beta_0}(y)+(\beta-\beta_0)h(y)$ and $q_\beta'(y)=h(y)$.
Define
\[
D(\beta)=\KL(q_\beta\|q_{\beta_0})=\int q_\beta(y)\log\frac{q_\beta(y)}{q_{\beta_0}(y)}\,dy.
\]
Using $\int h(y)\,dy=0$ and differentiating under the integral sign,
\[
D'(\beta)=\int h(y)\log\frac{q_\beta(y)}{q_{\beta_0}(y)}\,dy,
\qquad
D''(\beta)=\int \frac{h(y)^2}{q_\beta(y)}\,dy.
\]
Because $q_\beta(y)>0$ for all $y$ and $h\not\equiv 0$, we have $D''(\beta)>0$ for $\beta\in(0,1)$.
Thus $D$ is strictly convex.
Also $D'(\beta_0)=0$, so strict convexity implies
\[
D'(\beta)<0\quad\text{for }\beta<\beta_0,
\qquad
D'(\beta)>0\quad\text{for }\beta>\beta_0.
\]
In particular, $D'(0)<0$.

Finally, we characterize the maximizer of the anchored objective. Assume $u_{\mathrm{n}}>u_{\mathrm{o}}$, so that $J_\eta'(q_\beta)<0$ for all $\beta\in[0,1]$.
Define
\[
H(\beta):=\mathcal{L}_{\eta,\lambda_{\mathrm{ref}}}(q_\beta)=J_\eta(q_\beta)-\lambda_{\mathrm{ref}}D(\beta).
\]
Since $J_\eta''<0$ and $D''>0$, we have
\[
H''(\beta)=J_\eta''(q_\beta)-\lambda_{\mathrm{ref}}D''(\beta)<0
\qquad\forall\beta\in(0,1),
\]
so $H$ is strictly concave.

Now
\[
H'(0)=J_\eta'(q_\beta)\big|_{\beta=0}-\lambda_{\mathrm{ref}}D'(0).
\]
Because $J_\eta'(q_\beta)|_{\beta=0}<0$ and $D'(0)<0$, define
\[
\lambda_{\mathrm{crit}}^{(\mathrm{new})}:=\frac{-J_\eta'(q_\beta)|_{\beta=0}}{-D'(0)}>0.
\]
If $0\le \lambda_{\mathrm{ref}}\le \lambda_{\mathrm{crit}}^{(\mathrm{new})}$, then $H'(0)\le 0$.
Since $H'$ is strictly decreasing, we have $H'(\beta)<0$ for all $\beta\in(0,1)$, so the unique maximizer is $\beta^\star=0$.

If $\lambda_{\mathrm{ref}}>\lambda_{\mathrm{crit}}^{(\mathrm{new})}$, then $H'(0)>0$.
At $\beta=\beta_0$ we have $D'(\beta_0)=0$, so
\[
H'(\beta_0)=J_\eta'(q_{\beta_0})<0.
\]
By continuity of $H'$ and strict concavity of $H$, there is a unique root of $H'$ in $(0,\beta_0)$, which is the unique maximizer of $H$.
This proves part (A).

\paragraph{Proof of Part (B):} Consider the full family
\[
q_{\beta,m_{\mathrm{o}},m_{\mathrm{n}}}(y)
=
\beta\,\gauss{y}{m_{\mathrm{o}}}{\Sigma}+(1-\beta)\,\gauss{y}{m_{\mathrm{n}}}{\Sigma},
\]
and define
\[
J_\eta(\beta,m_{\mathrm{o}},m_{\mathrm{n}})
=
\log\E_{Y\sim q_{\beta,m_{\mathrm{o}},m_{\mathrm{n}}}}[e^{\eta r(Y)}].
\]
Let
\[
M:=\E_{Y\sim q_{\beta,m_{\mathrm{o}},m_{\mathrm{n}}}}[e^{\eta r(Y)}],
\qquad
w_\eta^{(q)}(y):=\frac{e^{\eta r(y)}}{M}.
\]
By the standard score identity,
\[
\nabla_\theta J_\eta(q_\theta)
=
\E_{Y\sim q_\theta}\big[w_\eta^{(q_\theta)}(Y)\,\nabla_\theta\log q_\theta(Y)\big].
\]
For the old mean parameter,
\[
\nabla_{m_{\mathrm{o}}}\log q_{\beta,m_{\mathrm{o}},m_{\mathrm{n}}}(y)
=
r_{\mathrm{o}}(y)\,\Sigma^{-1}(y-m_{\mathrm{o}}),
\qquad
r_{\mathrm{o}}(y):=\frac{\beta\,\gauss{y}{m_{\mathrm{o}}}{\Sigma}}{q_{\beta,m_{\mathrm{o}},m_{\mathrm{n}}}(y)}.
\]
Hence, at $m_{\mathrm{o}}=\mu_{\mathrm{o}}$,
\[
\nabla_{m_{\mathrm{o}}}J_\eta(\beta,\mu_{\mathrm{o}},m_{\mathrm{n}})
=
\E_{Y\sim q_{\beta,\mu_{\mathrm{o}},m_{\mathrm{n}}}}\!\Big[w_\eta^{(q)}(Y)\,r_{\mathrm{o}}(Y)\,\Sigma^{-1}(Y-\mu_{\mathrm{o}})\Big].
\]
Using
\[
q_{\beta,\mu_{\mathrm{o}},m_{\mathrm{n}}}(y)\,r_{\mathrm{o}}(y)=\beta\,\gauss{y}{\mu_{\mathrm{o}}}{\Sigma}=\beta\,p_{\mathrm{o}}(y),
\]
this becomes
\[
\nabla_{m_{\mathrm{o}}}J_\eta(\beta,\mu_{\mathrm{o}},m_{\mathrm{n}})
=
\beta\,\E_{Y\sim p_{\mathrm{o}}}\!\Big[w_\eta^{(q)}(Y)\,\Sigma^{-1}(Y-\mu_{\mathrm{o}})\Big].
\]
Since $r$ is constant on $A_{\mathrm{o}}$ and $A_{\mathrm{n}}$, the weight is constant on each region:
\[
w_\eta^{(q)}(y)=w_{\mathrm{o}}\ \text{on }A_{\mathrm{o}},
\qquad
w_\eta^{(q)}(y)=w_{\mathrm{n}}\ \text{on }A_{\mathrm{n}}.
\]
Thus
\[
\nabla_{m_{\mathrm{o}}}J_\eta(\beta,\mu_{\mathrm{o}},m_{\mathrm{n}})
=
\beta\Big(
w_{\mathrm{o}}\,\E_{p_{\mathrm{o}}}[\Sigma^{-1}(Y-\mu_{\mathrm{o}})\mathbf{1}\{Y\in A_{\mathrm{o}}\}]
+
w_{\mathrm{n}}\,\E_{p_{\mathrm{o}}}[\Sigma^{-1}(Y-\mu_{\mathrm{o}})\mathbf{1}\{Y\in A_{\mathrm{n}}\}]
\Big).
\]
Since $\E_{p_{\mathrm{o}}}[\Sigma^{-1}(Y-\mu_{\mathrm{o}})]=0$, the two truncated expectations are negatives of one another, giving
\[
\nabla_{m_{\mathrm{o}}}J_\eta(\beta,\mu_{\mathrm{o}},m_{\mathrm{n}})
=
\beta(w_{\mathrm{n}}-w_{\mathrm{o}})\,
\E_{p_{\mathrm{o}}}\!\big[\Sigma^{-1}(Y-\mu_{\mathrm{o}})\mathbf{1}\{Y\in A_{\mathrm{n}}\}\big].
\]
By Lemma~\ref{lem:trunc_gauss_bayes_halfspace}, we have
\[
\E_{p_{\mathrm{o}}}\!\big[\Sigma^{-1}(Y-\mu_{\mathrm{o}})\mathbf{1}\{Y\in A_{\mathrm{n}}\}\big]
=
\frac{\varphi(\delta/2)}{\delta}\,\Sigma^{-1}(\mu_{\mathrm{n}}-\mu_{\mathrm{o}}),
\]
where $\varphi(t)=(2\pi)^{-1/2}e^{-t^2/2}$.
Substituting proves \eqref{eq:ttt_oldmean_grad_explicit_streamlined}.

Finally, if $|u_{\mathrm{o}}|,|u_{\mathrm{n}}|\le R$, then
$e^{\eta r(y)}\in[e^{-\eta R},e^{\eta R}]$ pointwise, so
\[
M\in[e^{-\eta R},e^{\eta R}],
\qquad
w_{\mathrm{o}},w_{\mathrm{n}}\in[e^{-2\eta R},e^{2\eta R}],
\]
and therefore
\[
|w_{\mathrm{n}}-w_{\mathrm{o}}|\le e^{2\eta R}-e^{-2\eta R}.
\]
Using $\varphi(\delta/2)=(2\pi)^{-1/2}e^{-\delta^2/8}$ in \eqref{eq:ttt_oldmean_grad_explicit_streamlined} yields
\eqref{eq:ttt_oldmean_grad_bound_streamlined}.

If the full objective is evaluated at a synchronized point $q_0=q_{\beta,\mu_{\mathrm{o}},m_{\mathrm{n}}}$, then
\[
\nabla_{m_{\mathrm{o}}}\KL(q_{\beta,m_{\mathrm{o}},m_{\mathrm{n}}}\|q_0)\Big|_{m_{\mathrm{o}}=\mu_{\mathrm{o}}}=0,
\]
because $\log(q/q_0)\equiv 0$ at that point and $\int \nabla_{m_{\mathrm{o}}}q=0$.
Hence the same bound applies to the full objective.
\end{proof}

\begin{lemma}[Truncated Gaussian moment along the Bayes halfspace]\label{lem:trunc_gauss_bayes_halfspace}
In the setting of Theorem~\ref{thm:ttt_gaussian}, we have
\[
\E\Big[\Sigma^{-1}(Y-\mu_{\mathrm{o}})\,\mathbf{1}\{Y\in A_{\mathrm{n}}\}\Big]
=
\frac{\varphi(\delta/2)}{\delta}\,\Sigma^{-1}(\mu_{\mathrm{n}}-\mu_{\mathrm{o}}),
\]
where $\varphi(t):=(2\pi)^{-1/2}e^{-t^2/2}$ is the standard normal density.
\end{lemma}

\begin{proof}
The proof is based on routine moment computations. Write $\Delta:=\mu_{\mathrm{n}}-\mu_{\mathrm{o}}$ and $X:=Y-\mu_{\mathrm{o}}$. Then $X\sim\mathcal{N}(0,\Sigma)$ and
\[
\Sigma^{-1}(Y-\mu_{\mathrm{o}})\mathbf{1}\{Y\in A_{\mathrm{n}}\}
=
\Sigma^{-1}X\,\mathbf{1}\{Y\in A_{\mathrm{n}}\}.
\]
We start by rewriteing the truncation event. By definition of $A_{\mathrm{n}}$, we have
\begin{align*}
Y\in A_{\mathrm{n}}
&\Longleftrightarrow
\Delta^\top\Sigma^{-1}\left(Y-\frac{\mu_{\mathrm{o}}+\mu_{\mathrm{n}}}{2}\right)\ge 0\\
&\Longleftrightarrow
\Delta^\top\Sigma^{-1}\left(\mu_{\mathrm{o}}+X-\frac{\mu_{\mathrm{o}}+\mu_{\mathrm{n}}}{2}\right)\ge 0\\
&\Longleftrightarrow
\Delta^\top\Sigma^{-1}\left(X-\frac{\Delta}{2}\right)\ge 0\\
&\Longleftrightarrow
\Delta^\top\Sigma^{-1}X\ge \frac{1}{2}\Delta^\top\Sigma^{-1}\Delta
=\frac{\delta^2}{2}.
\end{align*}
Next we whiten the Gaussian. Let $Z:=\Sigma^{-1/2}X$. Then $Z\sim\mathcal{N}(0,I_d)$ and $X=\Sigma^{1/2}Z$. Hence $\Sigma^{-1}X=\Sigma^{-1}\Sigma^{1/2}Z=\Sigma^{-1/2}Z.$ Also define $b:=\Sigma^{-1/2}\Delta\in\mathbb{R}^d.$
Then $\|b\|=\sqrt{b^\top b}=\sqrt{\Delta^\top\Sigma^{-1}\Delta}=\delta$, and
\[
\Delta^\top\Sigma^{-1}X
=
\Delta^\top\Sigma^{-1}\Sigma^{1/2}Z
=
(\Sigma^{-1/2}\Delta)^\top Z
=
b^\top Z.
\]
Therefore, we have the equivalence:
\[
Y\in A_{\mathrm{n}}
\quad\Longleftrightarrow\quad
b^\top Z\ge \frac{\delta^2}{2}.
\]
Combining these identities gives
\begin{equation}\label{eq:reduce_to_Z}
\E\Big[\Sigma^{-1}(Y-\mu_{\mathrm{o}})\,\mathbf{1}\{Y\in A_{\mathrm{n}}\}\Big]
=
\Sigma^{-1/2}\,\E\Big[Z\,\mathbf{1}\{b^\top Z\ge \delta^2/2\}\Big].
\end{equation}

We now reduce the above to a one-dimensional truncated normal moment. Let $u:=b/\delta$, so $\|u\|=1$ and $b^\top Z=\delta\,u^\top Z$.
Define the scalar random variable $U:=u^\top Z.$ Since $Z\sim\mathcal{N}(0,I_d)$ and $\|u\|=1$, we have $U\sim\mathcal{N}(0,1)$.
Moreover,
\[
\{b^\top Z\ge \delta^2/2\}
=
\{\delta U\ge \delta^2/2\}
=
\{U\ge \delta/2\}.
\]
Now decompose $Z$ into its component along $u$ and its orthogonal remainder:
\[
Z = uU + V,
\qquad
V:=Z-uU=(I-uu^\top)Z.
\]
We claim $V$ is independent of $U$ and satisfies $\E[V]=0$.
Indeed, $(U,V)$ is jointly Gaussian (as an affine image of the Gaussian vector $Z$), and
\begin{align*}
\textrm{Cov}(U,V)
=
\E[U V^\top]
=
\E\big[(u^\top Z)\,Z^\top(I-uu^\top)\big]
=
u^\top \E[ZZ^\top](I-uu^\top)\\
=
u^\top I_d (I-uu^\top)
=
u^\top - u^\top uu^\top
=
u^\top-u^\top
=
0.
\end{align*}
For jointly Gaussian random variables, zero covariance implies independence, hence $U$ and $V$ are independent.
Also $\E[V]=(I-uu^\top)\E[Z]=0$.
Therefore,
\begin{align*}
\E\Big[Z\,\mathbf{1}\{U\ge \delta/2\}\Big]
&=
\E\Big[(uU+V)\,\mathbf{1}\{U\ge \delta/2\}\Big]\\
&=
u\,\E\Big[U\,\mathbf{1}\{U\ge \delta/2\}\Big]
\;+\;
\E\Big[V\,\mathbf{1}\{U\ge \delta/2\}\Big].
\end{align*}
Using independence of $V$ and $U$ and $\E[V]=0$,
\begin{align*}
\E\Big[V\,\mathbf{1}\{U\ge \delta/2\}\Big]
=
\E\Big[\E[V\,\mathbf{1}\{U\ge \delta/2\}\mid U]\Big]
=
\E\Big[\mathbf{1}\{U\ge \delta/2\}\,\E[V\mid U]\Big]\\
=
\E\Big[\mathbf{1}\{U\ge \delta/2\}\,\E[V]\Big]
=
0.
\end{align*}
Thus
\begin{equation}\label{eq:Z_trunc_reduction}
\E\Big[Z\,\mathbf{1}\{U\ge \delta/2\}\Big]
=
u\,\E\Big[U\,\mathbf{1}\{U\ge \delta/2\}\Big].
\end{equation}

We now compute the scalar truncated moment. Since $U\sim\mathcal{N}(0,1)$ with density $\varphi$, we have
\[
\E\Big[U\,\mathbf{1}\{U\ge a\}\Big]
=
\int_{a}^{\infty} u\,\varphi(u)\,du
\qquad\text{for any }a\in\mathbb{R}.
\]
We now compute this integral explicitly.
Recall $\varphi(u)=(2\pi)^{-1/2}e^{-u^2/2}$, so
\[
\frac{d}{du}\varphi(u)
=
(2\pi)^{-1/2}\frac{d}{du}\big(e^{-u^2/2}\big)
=
(2\pi)^{-1/2}\big(-u\big)e^{-u^2/2}
=
-u\,\varphi(u).
\]
Hence $u\varphi(u)=-\varphi'(u)$, and therefore
\[
\int_{a}^{\infty} u\,\varphi(u)\,du
=
-\int_{a}^{\infty}\varphi'(u)\,du
=
-\Big(\lim_{u\to\infty}\varphi(u)-\varphi(a)\Big)
=
\varphi(a),
\]
since $\lim_{u\to\infty}\varphi(u)=0$.
Taking $a=\delta/2$ yields
\begin{equation}\label{eq:scalar_trunc_moment}
\E\Big[U\,\mathbf{1}\{U\ge \delta/2\}\Big]
=
\varphi(\delta/2).
\end{equation}
Substituting \eqref{eq:scalar_trunc_moment} into \eqref{eq:Z_trunc_reduction} gives
\[
\E\Big[Z\,\mathbf{1}\{b^\top Z\ge \delta^2/2\}\Big]
=
\E\Big[Z\,\mathbf{1}\{U\ge \delta/2\}\Big]
=
u\,\varphi(\delta/2)
=
\frac{b}{\delta}\,\varphi(\delta/2).
\]
Plugging this into \eqref{eq:reduce_to_Z} yields
\begin{align*}
\E\Big[\Sigma^{-1}(Y-\mu_{\mathrm{o}})\,\mathbf{1}\{Y\in A_{\mathrm{n}}\}\Big]
=
\Sigma^{-1/2}\left(\frac{b}{\delta}\,\varphi(\delta/2)\right)
=
\frac{\varphi(\delta/2)}{\delta}\,\Sigma^{-1/2}b\\
=
\frac{\varphi(\delta/2)}{\delta}\,\Sigma^{-1/2}\Sigma^{-1/2}\Delta
=
\frac{\varphi(\delta/2)}{\delta}\,\Sigma^{-1}\Delta.
\end{align*}
Recalling that  $\Delta=\mu_{\mathrm{n}}-\mu_{\mathrm{o}}$, completes the proof.
\end{proof}

\subsection{Exact Characterization of the Optimal Mixture Weight for the TTT-Discover}\label{sec:exactchar}
\begin{proposition}[Exact characterization of the optimal mixture weight for the TTT-style objective]\label{prop:ttt_beta_star_characterization1}
Fix $\eta>0$, $\lambda_{\mathrm{ref}}\ge 0$, and a reference weight $\beta_0\in(0,1)$.
Let
\[
a:=e^{\eta u_{\mathrm{o}}},\qquad b:=e^{\eta u_{\mathrm{n}}},
\]
where $u_{\mathrm{o}},u_{\mathrm{n}}\in\R$ are the old- and new-side reward levels.

\medskip
\noindent\textbf{(A) Disjoint-support case.}
Assume the disjoint-support setting of Lemma~\ref{lem:ttt_disjoint}. Define
\[
H_{\mathrm{disc}}(\beta)
:=
\mathcal{L}_{\eta,\lambda_{\mathrm{ref}}}(q_\beta)
=
\log\!\big(\beta a+(1-\beta)b\big)
-
\lambda_{\mathrm{ref}}\!\left[
\beta\log\frac{\beta}{\beta_0}
+
(1-\beta)\log\frac{1-\beta}{1-\beta_0}
\right].
\]
Then the maximizer $\beta^\star\in[0,1]$ is characterized as follows:
\[
\beta^\star=
\begin{cases}
0, & \lambda_{\mathrm{ref}}=0,\; b>a,\\[3pt]
1, & \lambda_{\mathrm{ref}}=0,\; a>b,\\[3pt]
\text{any }\beta\in[0,1], & \lambda_{\mathrm{ref}}=0,\; a=b,\\[3pt]
\text{unique solution in }(0,1)\text{ of }
\displaystyle
\frac{a-b}{\beta a+(1-\beta)b}
=
\lambda_{\mathrm{ref}}
\left(
\log\frac{\beta}{\beta_0}
-
\log\frac{1-\beta}{1-\beta_0}
\right),
& \lambda_{\mathrm{ref}}>0.
\end{cases}
\]
In particular, when $\lambda_{\mathrm{ref}}>0$ and $a=b$, the unique maximizer is $\beta^\star=\beta_0$.

\medskip
\noindent\textbf{(B) Gaussian case.}
Assume the Gaussian setting of Theorem~\ref{thm:ttt_gaussian}. Define
\[
\gamma:=\Phi\!\Big(-\frac{\delta}{2}\Big),\qquad
\kappa:=1-2\gamma,
\]
and
\[
H_{\mathrm{gauss}}(\beta)
:=
\mathcal{L}_{\eta,\lambda_{\mathrm{ref}}}(q_\beta)
=
\log\!\Big(
a(\gamma+\kappa\beta)+b(1-\gamma-\kappa\beta)
\Big)
-\lambda_{\mathrm{ref}}\,D(\beta),
\]
where
\[
D(\beta):=\KL(q_\beta\|q_{\beta_0}).
\]
Then the maximizer $\beta^\star\in[0,1]$ is characterized as follows:

\begin{itemize}
\item If $\lambda_{\mathrm{ref}}=0$, then
\[
\beta^\star=
\begin{cases}
0, & b>a,\\
1, & a>b,\\
\text{any }\beta\in[0,1], & a=b.
\end{cases}
\]

\item If $\lambda_{\mathrm{ref}}>0$ and $a=b$, then $\beta^\star=\beta_0$.

\item If $b>a$, define
\[
\lambda_{\mathrm{crit}}^{(\mathrm{new})}
:=
\frac{\kappa(b-a)}
{\big(a\gamma+b(1-\gamma)\big)\,(-D'(0))},
\qquad
D'(\beta)=\int_{\R^d}\big(p_{\mathrm{o}}(y)-p_{\mathrm{n}}(y)\big)\log\frac{q_\beta(y)}{q_{\beta_0}(y)}\,dy.
\]
Then
\[
\beta^\star=
\begin{cases}
0, & 0\le \lambda_{\mathrm{ref}}\le \lambda_{\mathrm{crit}}^{(\mathrm{new})},\\[4pt]
\text{unique solution in }(0,\beta_0)\text{ of }
\displaystyle
\frac{\kappa(a-b)}
{a(\gamma+\kappa\beta)+b(1-\gamma-\kappa\beta)}
=
\lambda_{\mathrm{ref}}\,D'(\beta),
& \lambda_{\mathrm{ref}}>\lambda_{\mathrm{crit}}^{(\mathrm{new})}.
\end{cases}
\]

\item If $a>b$, define
\[
\lambda_{\mathrm{crit}}^{(\mathrm{old})}
:=
\frac{\kappa(a-b)}
{\big(a(1-\gamma)+b\gamma\big)\,D'(1)}.
\]
Then
\[
\beta^\star=
\begin{cases}
1, & 0\le \lambda_{\mathrm{ref}}\le \lambda_{\mathrm{crit}}^{(\mathrm{old})},\\[4pt]
\text{unique solution in }(\beta_0,1)\text{ of }
\displaystyle
\frac{\kappa(a-b)}
{a(\gamma+\kappa\beta)+b(1-\gamma-\kappa\beta)}
=
\lambda_{\mathrm{ref}}\,D'(\beta),
& \lambda_{\mathrm{ref}}>\lambda_{\mathrm{crit}}^{(\mathrm{old})}.
\end{cases}
\]
\end{itemize}
\end{proposition}

\begin{proof}
We treat the disjoint-support and Gaussian cases separately.

\medskip
\noindent\textbf{Proof of Part (A):} By Lemma~\ref{lem:ttt_disjoint}, the objective is exactly
\[
H_{\mathrm{disc}}(\beta)
=
\log\!\big(\beta a+(1-\beta)b\big)
-
\lambda_{\mathrm{ref}}\!\left[
\beta\log\frac{\beta}{\beta_0}
+
(1-\beta)\log\frac{1-\beta}{1-\beta_0}
\right].
\]

First consider $\lambda_{\mathrm{ref}}=0$.
Then $H_{\mathrm{disc}}(\beta)=\log(\beta a+(1-\beta)b)$, the logarithm of an affine function of $\beta$.
Hence:
\begin{itemize}
\item if $b>a$, it is strictly decreasing, so $\beta^\star=0$;
\item if $a>b$, it is strictly increasing, so $\beta^\star=1$;
\item if $a=b$, it is constant, so every $\beta\in[0,1]$ is optimal.
\end{itemize}

Now assume $\lambda_{\mathrm{ref}}>0$.
Differentiate on $(0,1)$:
\[
H_{\mathrm{disc}}'(\beta)
=
\frac{a-b}{\beta a+(1-\beta)b}
-
\lambda_{\mathrm{ref}}
\left(
\log\frac{\beta}{\beta_0}
-
\log\frac{1-\beta}{1-\beta_0}
\right).
\]
Differentiating once more gives
\[
H_{\mathrm{disc}}''(\beta)
=
-\frac{(a-b)^2}{(\beta a+(1-\beta)b)^2}
-\lambda_{\mathrm{ref}}\left(\frac{1}{\beta}+\frac{1}{1-\beta}\right)
<0
\qquad\forall \beta\in(0,1),
\]
so $H_{\mathrm{disc}}$ is strictly concave on $(0,1)$.

Moreover,
\[
\lim_{\beta\downarrow 0}H_{\mathrm{disc}}'(\beta)=+\infty,
\qquad
\lim_{\beta\uparrow 1}H_{\mathrm{disc}}'(\beta)=-\infty,
\]
because the logarithmic term diverges while the first term remains finite at the endpoints.
Hence, by continuity of $H_{\mathrm{disc}}'$, there exists a unique $\beta^\star\in(0,1)$ with
$H_{\mathrm{disc}}'(\beta^\star)=0$.
This gives the stated first-order equation.
If $a=b$, then the first term in $H_{\mathrm{disc}}'$ vanishes identically, so the unique solution is
\[
\log\frac{\beta}{\beta_0}
=
\log\frac{1-\beta}{1-\beta_0},
\]
which is equivalent to $\beta=\beta_0$.

\medskip
\noindent\textbf{Proof of Part (B):} By Theorem~\ref{thm:ttt_gaussian}, the Gaussian objective can be written as
\[
H_{\mathrm{gauss}}(\beta)
=
\log\!\Big(
a(\gamma+\kappa\beta)+b(1-\gamma-\kappa\beta)
\Big)
-\lambda_{\mathrm{ref}}D(\beta),
\]
with
\[
D'(\beta)=\int_{\R^d}\big(p_{\mathrm{o}}(y)-p_{\mathrm{n}}(y)\big)\log\frac{q_\beta(y)}{q_{\beta_0}(y)}\,dy,
\qquad
D''(\beta)=\int_{\R^d}\frac{(p_{\mathrm{o}}(y)-p_{\mathrm{n}}(y))^2}{q_\beta(y)}\,dy>0.
\]
Therefore $D$ is strictly convex, $D'(\beta_0)=0$, $D'(0)<0$, and $D'(1)>0$.

First consider $\lambda_{\mathrm{ref}}=0$.
Then
\[
H_{\mathrm{gauss}}(\beta)=\log\!\Big(a(\gamma+\kappa\beta)+b(1-\gamma-\kappa\beta)\Big),
\]
whose derivative is
\[
H_{\mathrm{gauss}}'(\beta)
=
\frac{\kappa(a-b)}{a(\gamma+\kappa\beta)+b(1-\gamma-\kappa\beta)}.
\]
Since $\kappa>0$, the sign is the sign of $a-b$.
Thus:
\begin{itemize}
\item if $b>a$, the derivative is strictly negative and $\beta^\star=0$;
\item if $a>b$, the derivative is strictly positive and $\beta^\star=1$;
\item if $a=b$, the derivative is zero and every $\beta$ is optimal.
\end{itemize}

Now assume $\lambda_{\mathrm{ref}}>0$.
If $a=b$, then the logarithmic term is constant in $\beta$, so maximizing $H_{\mathrm{gauss}}$ is equivalent to minimizing $D$.
Since $D$ is strictly convex and $D'(\beta_0)=0$, its unique minimizer is $\beta_0$; hence $\beta^\star=\beta_0$.

Next suppose $b>a$.
Then
\[
H_{\mathrm{gauss}}'(\beta)
=
\frac{\kappa(a-b)}{a(\gamma+\kappa\beta)+b(1-\gamma-\kappa\beta)}
-\lambda_{\mathrm{ref}}D'(\beta),
\]
and
\[
H_{\mathrm{gauss}}''(\beta)
=
-\frac{\kappa^2(a-b)^2}{\big(a(\gamma+\kappa\beta)+b(1-\gamma-\kappa\beta)\big)^2}
-\lambda_{\mathrm{ref}}D''(\beta)
<0.
\]
So $H_{\mathrm{gauss}}$ is strictly concave, hence has at most one maximizer in $(0,1)$.

Evaluate the derivative at $\beta=0$:
\[
H_{\mathrm{gauss}}'(0)
=
\frac{\kappa(a-b)}{a\gamma+b(1-\gamma)}
-\lambda_{\mathrm{ref}}D'(0).
\]
Since $a-b<0$ and $D'(0)<0$, this equals
\[
H_{\mathrm{gauss}}'(0)
=
-\frac{\kappa(b-a)}{a\gamma+b(1-\gamma)}
+\lambda_{\mathrm{ref}}(-D'(0)).
\]
Therefore $H_{\mathrm{gauss}}'(0)\le 0$ exactly when
\[
\lambda_{\mathrm{ref}}
\le
\frac{\kappa(b-a)}{(a\gamma+b(1-\gamma))(-D'(0))}
=
\lambda_{\mathrm{crit}}^{(\mathrm{new})}.
\]
If this holds, then because $H_{\mathrm{gauss}}'$ is strictly decreasing on $(0,1)$, we have
$H_{\mathrm{gauss}}'(\beta)<0$ for all $\beta\in(0,1)$, so the unique maximizer is $\beta^\star=0$.

If instead $\lambda_{\mathrm{ref}}>\lambda_{\mathrm{crit}}^{(\mathrm{new})}$, then $H_{\mathrm{gauss}}'(0)>0$.
At $\beta=\beta_0$, since $D'(\beta_0)=0$,
\[
H_{\mathrm{gauss}}'(\beta_0)
=
\frac{\kappa(a-b)}{a(\gamma+\kappa\beta_0)+b(1-\gamma-\kappa\beta_0)}
<0.
\]
By continuity and strict monotonicity of $H_{\mathrm{gauss}}'$, there is a unique root in $(0,\beta_0)$, and that root is the unique maximizer.
It satisfies exactly the first-order equation
\[
\frac{\kappa(a-b)}{a(\gamma+\kappa\beta)+b(1-\gamma-\kappa\beta)}
=
\lambda_{\mathrm{ref}}D'(\beta).
\]

The case $a>b$ is symmetric.
Now the entropic derivative is positive, and the threshold is determined by the right boundary:
\[
\lambda_{\mathrm{crit}}^{(\mathrm{old})}
=
\frac{H_{\mathrm{gauss}}'(\beta)\big|_{\lambda_{\mathrm{ref}}=0,\ \beta=1}}{D'(1)}
=
\frac{\kappa(a-b)}{(a(1-\gamma)+b\gamma)D'(1)}.
\]
If $\lambda_{\mathrm{ref}}\le \lambda_{\mathrm{crit}}^{(\mathrm{old})}$, then $H_{\mathrm{gauss}}'(1)\ge 0$, and since
$H_{\mathrm{gauss}}'$ is strictly decreasing, $H_{\mathrm{gauss}}'(\beta)>0$ on $(0,1)$, so the unique maximizer is $\beta^\star=1$.
If $\lambda_{\mathrm{ref}}>\lambda_{\mathrm{crit}}^{(\mathrm{old})}$, then $H_{\mathrm{gauss}}'(1)<0$ while
\[
H_{\mathrm{gauss}}'(\beta_0)
=
\frac{\kappa(a-b)}{a(\gamma+\kappa\beta_0)+b(1-\gamma-\kappa\beta_0)}
>0,
\]
so the unique maximizer lies in $(\beta_0,1)$ and is characterized by the same first-order equation.
This completes the proof.
\end{proof}

\subsection{Proof of Lemma~\ref{lem:OAPL_disjoint}}

\begin{proof}[Proof of Lemma~\ref{lem:OAPL_disjoint}]
Under disjoint-support, we have
\[
q_0(y)=\beta_0 p_{\mathrm{o}}(y)\quad \text{for }y\in A_{\mathrm{o}},
\qquad
q_0(y)=(1-\beta_0)p_{\mathrm{n}}(y)\quad \text{for }y\in A_{\mathrm{n}}.
\]
Since the reward is constant on each region,
\[
q^*(y)=\frac{1}{Z}q_0(y)e^{r(y)/\tau}
=
\begin{cases}
\displaystyle \frac{1}{Z}\,\beta_0 e^{r_{\mathrm{o}}/\tau}p_{\mathrm{o}}(y), & y\in A_{\mathrm{o}},\\[8pt]
\displaystyle \frac{1}{Z}\,(1-\beta_0)e^{r_{\mathrm{n}}/\tau}p_{\mathrm{n}}(y), & y\in A_{\mathrm{n}}.
\end{cases}
\]
Therefore \(q^*\) is again a two-component mixture with the same components:
\[
q^*(y)=\beta^* p_{\mathrm{o}}(y)+(1-\beta^*)p_{\mathrm{n}}(y),
\]
where
\[
\beta^*
=
\frac{\beta_0 e^{r_{\mathrm{o}}/\tau}}{Z},
\qquad
1-\beta^*=\frac{(1-\beta_0)e^{r_{\mathrm{n}}/\tau}}{Z}.
\]
Because \(q^*\) is a probability density,
\[
1=\int q^*(y)\,dy=\frac{\beta_0 e^{r_{\mathrm{o}}/\tau}}{Z}+\frac{(1-\beta_0)e^{r_{\mathrm{n}}/\tau}}{Z},
\]
so
\[
Z=\beta_0 e^{r_{\mathrm{o}}/\tau}+(1-\beta_0)e^{r_{\mathrm{n}}/\tau}.
\]
Substituting into the expression for \(\beta^*\) gives
\[
\beta^*=
\frac{\beta_0 e^{r_{\mathrm{o}}/\tau}}
{\beta_0 e^{r_{\mathrm{o}}/\tau}+(1-\beta_0)e^{r_{\mathrm{n}}/\tau}},
\]
which is \eqref{eq:OAPL_beta_star_disjoint}.
Since \(\beta_0\in(0,1)\) and \(r_{\mathrm{o}},r_{\mathrm{n}}\) are finite, all factors are strictly positive, so \(\beta^*\in(0,1)\).
\end{proof}

\subsection{Proof of Theorem~\ref{thm:OAPL_gaussian}}

\begin{proof}[Proof of Theorem~\ref{thm:OAPL_gaussian}]
We prove parts (A) and (B) separately.

\paragraph{Proof of Part (A):} We start by computing the exact expected old responsibility under the target \(q^*\). Recall
\[
q^*(y)=\frac{1}{Z}q_0(y)e^{r(y)/\tau},
\qquad
r_{\mathrm{o}}^{(0)}(y)=\frac{\beta_0 p_{\mathrm{o}}(y)}{q_0(y)}.
\]
Then
\begin{align}
\E_{Y\sim q^*}\big[r_{\mathrm{o}}^{(0)}(Y)\big]
&=
\int q^*(y)\,\frac{\beta_0 p_{\mathrm{o}}(y)}{q_0(y)}\,dy \notag\\
&=
\frac{\beta_0}{Z}\int p_{\mathrm{o}}(y)e^{r(y)/\tau}\,dy. \label{eq:old_resp_start}
\end{align}
We therefore compute the two Gaussian integrals
\[
I_{\mathrm{o}}:=\int p_{\mathrm{o}}(y)e^{r(y)/\tau}\,dy,
\qquad
I_{\mathrm{n}}:=\int p_{\mathrm{n}}(y)e^{r(y)/\tau}\,dy.
\]

Under \(Y\sim p_{\mathrm{o}}=\mathcal{N}(\mu_{\mathrm{o}},\Sigma)\), the Bayes region \(A_{\mathrm{n}}\) is entered with probability
\[
\Pr_{p_{\mathrm{o}}}(A_{\mathrm{n}})
=
\Phi\!\Big(-\frac{\delta}{2}\Big)
=
\gamma,
\]
so \(\Pr_{p_{\mathrm{o}}}(A_{\mathrm{o}})=1-\gamma\).
Hence
\[
I_{\mathrm{o}}
=
e^{r_{\mathrm{o}}/\tau}\Pr_{p_{\mathrm{o}}}(A_{\mathrm{o}})
+
e^{r_{\mathrm{n}}/\tau}\Pr_{p_{\mathrm{o}}}(A_{\mathrm{n}})
=
(1-\gamma)e^{r_{\mathrm{o}}/\tau}+\gamma e^{r_{\mathrm{n}}/\tau}.
\]
Similarly, under \(Y\sim p_{\mathrm{n}}=\mathcal{N}(\mu_{\mathrm{n}},\Sigma)\),
\[
\Pr_{p_{\mathrm{n}}}(A_{\mathrm{o}})=\gamma,
\qquad
\Pr_{p_{\mathrm{n}}}(A_{\mathrm{n}})=1-\gamma,
\]
so
\[
I_{\mathrm{n}}
=
\gamma e^{r_{\mathrm{o}}/\tau}+(1-\gamma)e^{r_{\mathrm{n}}/\tau}.
\]

Now compute the normalizer \(Z\):
\[
Z
=
\E_{Y\sim q_0}[e^{r(Y)/\tau}]
=
\beta_0 I_{\mathrm{o}}+(1-\beta_0)I_{\mathrm{n}}.
\]
Substituting \(I_{\mathrm{o}}\) and \(Z\) into \eqref{eq:old_resp_start} yields
\[
\E_{Y\sim q^*}\big[r_{\mathrm{o}}^{(0)}(Y)\big]
=
\frac{
\beta_0\big((1-\gamma)e^{r_{\mathrm{o}}/\tau}+\gamma e^{r_{\mathrm{n}}/\tau}\big)
}{
\beta_0\big((1-\gamma)e^{r_{\mathrm{o}}/\tau}+\gamma e^{r_{\mathrm{n}}/\tau}\big)
+
(1-\beta_0)\big(\gamma e^{r_{\mathrm{o}}/\tau}+(1-\gamma)e^{r_{\mathrm{n}}/\tau}\big)
},
\]
which is \eqref{eq:OAPL_expected_old_resp}.
All terms in the numerator and denominator are strictly positive, so the ratio lies in \((0,1)\).

\paragraph{Proof of Part (B):} We start by computing the gradient identity and overlap bound. Define
\[
\Delta_{\beta,m_{\mathrm{n}}}(y):=\tau\log\frac{q_{\beta,m_{\mathrm{n}}}(y)}{q_0(y)}-A^*(y),
\qquad
J(\beta,m_{\mathrm{n}})
=
\E_{Y\sim q_0}\!\left[\Delta_{\beta,m_{\mathrm{n}}}(Y)^2\right].
\]
Since \(q_0\) and \(A^*\) are fixed, differentiation under the expectation gives
\[
\nabla_{m_{\mathrm{n}}}J(\beta,m_{\mathrm{n}})
=
2\,\E_{Y\sim q_0}\!\left[\Delta_{\beta,m_{\mathrm{n}}}(Y)\,\nabla_{m_{\mathrm{n}}}\Delta_{\beta,m_{\mathrm{n}}}(Y)\right]
=
2\tau\,\E_{Y\sim q_0}\!\left[\Delta_{\beta,m_{\mathrm{n}}}(Y)\,\nabla_{m_{\mathrm{n}}}\log q_{\beta,m_{\mathrm{n}}}(Y)\right].
\]
For
\[
q_{\beta,m_{\mathrm{n}}}(y)=\beta p_{\mathrm{o}}(y)+(1-\beta)\gauss{y}{m_{\mathrm{n}}}{\Sigma},
\]
only the new component depends on \(m_{\mathrm{n}}\), and using
\[
\nabla_m \gauss{y}{m}{\Sigma}=\gauss{y}{m}{\Sigma}\,\Sigma^{-1}(y-m),
\]
we obtain
\[
\nabla_{m_{\mathrm{n}}}\log q_{\beta,m_{\mathrm{n}}}(y)
=
\frac{(1-\beta)\gauss{y}{m_{\mathrm{n}}}{\Sigma}}{q_{\beta,m_{\mathrm{n}}}(y)}\,\Sigma^{-1}(y-m_{\mathrm{n}})
=
r_{\mathrm{n}}^{(\beta,m_{\mathrm{n}})}(y)\,\Sigma^{-1}(y-m_{\mathrm{n}}),
\]
which proves \eqref{eq:OAPL_grad_m_identity_consolidated}.

Now specialize to the synchronized point \((\beta,m_{\mathrm{n}})=(\beta_0,\mu_{\mathrm{n}})\). There \(q_{\beta,m_{\mathrm{n}}}=q_0\), so
\[
\Delta_{\beta_0,\mu_{\mathrm{n}}}(y)=\tau\log\frac{q_0(y)}{q_0(y)}-A^*(y)=-A^*(y).
\]
Thus the contribution of old-mode samples to the gradient is
\[
-2\tau\,\beta_0\,
\E_{Y\sim p_{\mathrm{o}}}\!\left[A^*(Y)\,r_{\mathrm{n}}^{(\beta_0,\mu_{\mathrm{n}})}(Y)\,\Sigma^{-1}(Y-\mu_{\mathrm{n}})\right].
\]
If \(|r(y)|\le R\), then
\[
|V^*|=\tau\Big|\log \E_{q_0}[e^{r/\tau}]\Big|\le R
\]
because \(e^{r/\tau}\in[e^{-R/\tau},e^{R/\tau}]\), so
\[
|A^*(y)|=|r(y)-V^*|\le 2R.
\]
By Cauchy--Schwarz and the fact that \(0\le r_{\mathrm{n}}\le 1\) implies \(r_{\mathrm{n}}^2\le r_{\mathrm{n}}\),
\begin{align*}
&\Big\|
\E_{Y\sim p_{\mathrm{o}}}\!\left[A^*(Y)\,r_{\mathrm{n}}^{(\beta_0,\mu_{\mathrm{n}})}(Y)\,\Sigma^{-1}(Y-\mu_{\mathrm{n}})\right]
\Big\|\\
&\qquad\le
\sqrt{
\E_{Y\sim p_{\mathrm{o}}}\!\left[A^*(Y)^2\,\|\Sigma^{-1}(Y-\mu_{\mathrm{n}})\|^2\right]
}
\sqrt{
\E_{Y\sim p_{\mathrm{o}}}\!\left[r_{\mathrm{n}}^{(\beta_0,\mu_{\mathrm{n}})}(Y)\right]
}\\
&\qquad\le
2R\sqrt{
\E_{Y\sim p_{\mathrm{o}}}\!\left[\|\Sigma^{-1}(Y-\mu_{\mathrm{n}})\|^2\right]
}
\sqrt{\varepsilon_{\mathrm{o}\to\mathrm{n}}^{\mathrm{ref}}}
=
2R\,\sqrt{M_{\mathrm{o}\to\mathrm{n}}}\,\sqrt{\varepsilon_{\mathrm{o}\to\mathrm{n}}^{\mathrm{ref}}}.
\end{align*}
Multiplying by \(2\tau\beta_0\) gives \eqref{eq:OAPL_oldmode_grad_bound_consolidated}.

Finally, the reference leakage bound follows directly from Lemma~\ref{lem:leakage} applied to the mixture
\(q_0=\beta_0 p_{\mathrm{o}}+(1-\beta_0)p_{\mathrm{n}}\) together with Lemma~\ref{lem:bc-gauss}:
\[
\varepsilon_{\mathrm{o}\to\mathrm{n}}^{\mathrm{ref}}
=
\E_{Y\sim p_{\mathrm{o}}}\big[r_{\mathrm{n}}^{(\beta_0,\mu_{\mathrm{n}})}(Y)\big]
=
\E_{Y\sim p_{\mathrm{o}}}\big[1-r_{\mathrm{o}}^{(\beta_0,\mu_{\mathrm{n}})}(Y)\big]
\le
\frac12\sqrt{\frac{1-\beta_0}{\beta_0}}
\exp\!\left(-\frac18\mnorm{\mu_{\mathrm{n}}-\mu_{\mathrm{o}}}{\Sigma^{-1}}^2\right).
\]
\end{proof}

\section{Extension to $f$-divergences}\label{app:f_extension}

\paragraph{$f$-divergence.}
Let $f:(0,\infty)\to\R$ be convex.
Given distributions $P,Q$ with densities $p,q$ w.r.t.\ a common reference $\mu$ and with $P\ll Q$, define the Csisz\'ar--Morimoto $f$-divergence
\[
D_f(P\|Q) := \int q(y)\, f\!\left(\frac{p(y)}{q(y)}\right)\,d\mu(y).
\]

\paragraph{Affine invariance.}
If $\tilde f(t) = f(t) + a(t-1)$ for any $a\in\R$, then $D_{\tilde f}(P\|Q)=D_f(P\|Q)$ because $\int q(t-1)=\int (p-q)=0$.
Thus we may assume \emph{without loss of generality} that
\[
f(1)=0,\qquad f'(1)=0,
\]
whenever $f$ is differentiable at $1$.

\paragraph{Adjoint (reverse) generator.}
Define the adjoint generator
\[
f^\diamond(t) := t\, f(1/t),\qquad t>0.
\]
We start by the following standard adjoint result with the proof provided for completeness.

\begin{lemma}[Adjoint identity]\label{lem:adjoint}
Assume $P\ll Q$ and $Q\ll P$ with densities $p,q$.
Then
\[
D_f(P\|Q) = \int p(y)\, f^\diamond\!\left(\frac{q(y)}{p(y)}\right)\,d\mu(y)
= D_{f^\diamond}(Q\|P).
\]
Moreover, if $f$ is convex then $f^\diamond$ is convex; if $f$ is strictly convex then $f^\diamond$ is strictly convex.
\end{lemma}

\begin{proof}
Using $D_f(P\|Q)=\int q f(p/q)$ and multiplying and dividing by $p$,
\[
\int q f(p/q) = \int p\left(\frac{q}{p}\right) f\!\left(\frac{p}{q}\right)
= \int p\, f^\diamond(q/p).
\]
The equality $D_f(P\|Q)=D_{f^\diamond}(Q\|P)$ is immediate from definitions.
Convexity/strict convexity of $f^\diamond$ follows from standard perspective-transform facts: $t\mapsto t f(1/t)$ preserves convexity on $(0,\infty)$.
\end{proof}

When taking gradients of $f$-divergences, the quantity
\[
\kappa_f(t) := t\, f''(t),\qquad t>0,
\]
plays a central role. For convex twice differentiable $f$, $\kappa_f(t)\ge 0$.

We now generalize Theorem~\ref{thm:kl_main} from KL to the a family of $f$-divergences. %Recall the adjoint generator $f^\diamond(t)=t f(1/t)$ (Lemma~\ref{lem:adjoint}) and the curvature factor $\kappa_f(t)=t f''(t)$.

\begin{theorem}\label{thm:fmain}
Assume $f:(0,\infty)\to\R$ is twice continuously differentiable and strictly convex, normalized so that $f(1)=f'(1)=0$.
Let $f^\diamond(t)=t f(1/t)$ and $\kappa_f(t)=t f''(t)$.

\medskip
\noindent\textbf{(A)} Fix $m_{\mathrm{o}}=\mu_{\mathrm{o}}$ and $m_{\mathrm{n}}=\mu_{\mathrm{n}}$ and define $q_\beta=\beta p_{\mathrm{o}}+(1-\beta)p_{\mathrm{n}}$.
Consider
\[
L_{\mathrm{SFT}}^f(\beta):=D_f(p_{\mathrm{n}}\|q_\beta),\qquad \beta\in[0,1].
\]
Then $L_{\mathrm{SFT}}^f(0)=0$ and $L_{\mathrm{SFT}}^f(\beta)>0$ for every $\beta\in(0,1]$, hence $\beta=0$ is the unique global minimizer.
Moreover, $L_{\mathrm{SFT}}^f$ is strictly increasing on $[0,1]$, and logit gradient flow $\dot\phi=-\frac{d}{d\phi}L_{\mathrm{SFT}}^f(\sigma(\phi))$
satisfies $\beta(t)=\sigma(\phi(t))\downarrow 0$.

\medskip
\noindent\textbf{(B)} Fix $\alpha\in(0,1)$ and consider
\[
L_{\mathrm{RL}}^f(\beta,m_{\mathrm{o}},m_{\mathrm{n}}):=D_f\!\big(q_{\beta,m_{\mathrm{o}},m_{\mathrm{n}}}\,\|\,p_\alpha\big).
\]
Assume $m_{\mathrm{o}}=\mu_{\mathrm{o}}$ and define the density ratio $w(y):=\frac{q_{\beta,\mu_{\mathrm{o}},m_{\mathrm{n}}}(y)}{p_\alpha(y)}$.
Then the gradient w.r.t.\ the old mean admits the \emph{exact} decomposition
\[
\nabla_{m_{\mathrm{o}}} L_{\mathrm{RL}}^f(\beta,\mu_{\mathrm{o}},m_{\mathrm{n}})
=
\beta\,\Sigma^{-1}\Big(
A_f(\beta,m_{\mathrm{n}})\,(m_{\mathrm{n}}-\mu_{\mathrm{o}})
-
B_f(\alpha,\beta,m_{\mathrm{n}})\,(\mu_{\mathrm{n}}-\mu_{\mathrm{o}})
\Big),
\]
where
\[
A_f(\beta,m_{\mathrm{n}}):=\E_{Y\sim p_{\mathrm{o}}}\big[\kappa_f(w(Y))\,(1-r_{\mathrm{o}}(Y))\big],
\qquad
B_f(\alpha,\beta,m_{\mathrm{n}}):=\E_{Y\sim p_{\mathrm{o}}}\big[\kappa_f(w(Y))\,(1-s_{\mathrm{o}}(Y))\big].
\]
If, in addition, $f$ has bounded curvature $0\le \kappa_f(t)\le C_f$ for all $t>0$, then
\[
\norm{\nabla_{m_{\mathrm{o}}} L_{\mathrm{RL}}^f(\beta,\mu_{\mathrm{o}},m_{\mathrm{n}})}
\le
\beta\,C_f\,\norm{\Sigma^{-1}}_{2}\left(
\varepsilon_q(\beta,m_{\mathrm{n}})\,\norm{m_{\mathrm{n}}-\mu_{\mathrm{o}}}
+
\varepsilon_p(\alpha)\,\norm{\mu_{\mathrm{n}}-\mu_{\mathrm{o}}}
\right),
\]
where $\varepsilon_q=\E_{p_{\mathrm{o}}}[1-r_{\mathrm{o}}]$ and $\varepsilon_p=\E_{p_{\mathrm{o}}}[1-s_{\mathrm{o}}]$ satisfy the same Gaussian overlap bounds
as in \eqref{eq:kl_eps_bounds}.
\end{theorem}

\begin{proof}[Proof of Theorem~\ref{thm:fmain}]
\medskip

\textbf{Proof of Part (A).} By Lemma~\ref{lem:adjoint},
\[
L_{\mathrm{SFT}}^f(\beta)
= D_f(p_{\mathrm{n}}\|q_\beta)
= \E_{Y\sim p_{\mathrm{n}}}\!\left[
f^\diamond\!\left(\frac{q_\beta(Y)}{p_{\mathrm{n}}(Y)}\right)
\right].
\]
Write $X(Y):=\frac{p_{\mathrm{o}}(Y)}{p_{\mathrm{n}}(Y)}$ and $Z_\beta:=(1-\beta)+\beta X(Y)$ so that
$\frac{q_\beta(Y)}{p_{\mathrm{n}}(Y)}=Z_\beta$ and $\E[Z_\beta]=1$.
Since $f^\diamond$ is strictly convex and minimized uniquely at $1$ (by normalization),
Jensen's inequality yields
\[
L_{\mathrm{SFT}}^f(\beta)=\E[f^\diamond(Z_\beta)]
\ge f^\diamond(\E[Z_\beta])=f^\diamond(1)=0,
\]
with strict inequality for $\beta>0$ because $X(Y)$ is nonconstant when $\mu_{\mathrm{o}}\neq\mu_{\mathrm{n}}$.
Hence $\beta=0$ is the unique minimizer.

For monotonicity, note that for each fixed $y$, the map $\beta\mapsto f^\diamond((1-\beta)+\beta X(y))$ is convex in $\beta$
(because $f^\diamond$ is convex and $(1-\beta)+\beta X(y)$ is affine).
Thus $L_{\mathrm{SFT}}^f$ is convex on $[0,1]$.
Moreover, for any $0<\beta_1<\beta_2\le 1$, convexity implies the secant slope is nondecreasing:
\[
\frac{L_{\mathrm{SFT}}^f(\beta_2)-L_{\mathrm{SFT}}^f(\beta_1)}{\beta_2-\beta_1}
\ \ge\
\frac{L_{\mathrm{SFT}}^f(\beta_1)-L_{\mathrm{SFT}}^f(0)}{\beta_1-0}
=
\frac{L_{\mathrm{SFT}}^f(\beta_1)}{\beta_1}
\ >\ 0,
\]
so $L_{\mathrm{SFT}}^f(\beta_2)>L_{\mathrm{SFT}}^f(\beta_1)$ and $L_{\mathrm{SFT}}^f$ is strictly increasing.
The logit-gradient-flow claim follows exactly as in the KL case: since $\beta'(\phi)=\beta(1-\beta)>0$,
$\frac{d}{d\phi}L_{\mathrm{SFT}}^f(\sigma(\phi))>0$ for $\beta\in(0,1)$, hence $\phi(t)$ decreases and $\beta(t)\downarrow 0$.

\medskip
\textbf{Proof of Part (B).} Let $q=q_{\beta,m_{\mathrm{o}},m_{\mathrm{n}}}$ and $p=p_\alpha$ and $w=q/p$.
Write
\[
D_f(q\|p)=\int p(y)\,f(w(y))\,dy.
\]
Since $p$ does not depend on $m_{\mathrm{o}}$,
\[
\nabla_{m_{\mathrm{o}}}D_f(q\|p)=\int p(y)\,f'(w(y))\,\nabla_{m_{\mathrm{o}}}w(y)\,dy
=\int f'(w(y))\,\nabla_{m_{\mathrm{o}}}q(y)\,dy.
\]
At $m_{\mathrm{o}}=\mu_{\mathrm{o}}$, $\nabla_{m_{\mathrm{o}}}q(y)=\beta\,p_{\mathrm{o}}(y)\Sigma^{-1}(y-\mu_{\mathrm{o}})$, so
\[
\nabla_{m_{\mathrm{o}}}L_{\mathrm{RL}}^f(\beta,\mu_{\mathrm{o}},m_{\mathrm{n}})
=\beta\,\Sigma^{-1}\E_{Y\sim p_{\mathrm{o}}}\!\left[f'(w(Y))(Y-\mu_{\mathrm{o}})\right].
\]
Apply Stein's identity (Lemma~\ref{lem:stein}) with $g(y)=f'(w(y))$ to get
\[
\nabla_{m_{\mathrm{o}}}L_{\mathrm{RL}}^f(\beta,\mu_{\mathrm{o}},m_{\mathrm{n}})
=\beta\,\E_{p_{\mathrm{o}}}\!\left[\nabla_y f'(w(Y))\right]
=\beta\,\E_{p_{\mathrm{o}}}\!\left[f''(w(Y))\,\nabla_y w(Y)\right].
\]
Since $\nabla w = w(\nabla\log q-\nabla\log p)$, this becomes
\[
\nabla_{m_{\mathrm{o}}}L_{\mathrm{RL}}^f(\beta,\mu_{\mathrm{o}},m_{\mathrm{n}})
=\beta\,\E_{p_{\mathrm{o}}}\!\left[\kappa_f(w(Y))\,(\nabla\log q(Y)-\nabla\log p(Y))\right].
\]
The score difference $\nabla\log q-\nabla\log p$ is exactly the same as in the KL proof, yielding the stated decomposition
with $A_f=\E_{p_{\mathrm{o}}}[\kappa_f(w)(1-r_{\mathrm{o}})]$ and $B_f=\E_{p_{\mathrm{o}}}[\kappa_f(w)(1-s_{\mathrm{o}})]$.
If $\kappa_f\le C_f$, then $A_f\le C_f\,\varepsilon_q$ and $B_f\le C_f\,\varepsilon_p$, giving the displayed bound.
The overlap bounds on $\varepsilon_q$ and $\varepsilon_p$ are identical to the KL case and follow from Lemma~\ref{lem:leakage} and Remark~\ref{lem:bc-gauss}.
\end{proof}

\begin{remark}[What Changes from KL to General $f$?]\label{rem:f_changes}
Relative to KL (where $\kappa_f\equiv 1$), the only new factor in the old-mean gradient is the curvature weight
$\kappa_f(w)=w f''(w)$ applied to the score difference.
When $\kappa_f$ is bounded, the qualitative message is unchanged: the old-mean drift remains controlled by overlap/misassignment probabilities,
which are exponentially small in the separation for Gaussians.
If $\kappa_f$ is unbounded, the exact decomposition still holds, but quantitative bounds must track the distribution of $w=q/p_\alpha$
under $Y\sim p_{\mathrm{o}}$.
\end{remark}

\begin{remark}[On the Bounded-curvature Assumption]\label{rem:bounded_curv}
The bounded-curvature condition
\[
\sup_{t>0}\kappa_f(t)=\sup_{t>0}\big(t f''(t)\big)<\infty
\]
is a convenient way to ensure that the overlap-gated terms appearing in Theorem~\ref{thm:fmain}(B) can be upper bounded purely by
misassignment probabilities (times geometric factors), without having to track additional tail behavior of the density ratio
$w(\cdot)=q(\cdot)/p_\alpha(\cdot)$.
It holds for several standard, smooth $f$-divergences with ``log-like'' curvature, for example:
\begin{itemize}
\item \textbf{KL:} $f_{\mathrm{KL}}(t)=t\log t-(t-1)$ gives $\kappa_f(t)\equiv 1$.
\item \textbf{Jensen--Shannon:} one generator is
$f_{\mathrm{JS}}(t)=t\log t-(t+1)\log\!\big(\frac{t+1}{2}\big)$, for which
$f_{\mathrm{JS}}''(t)=\frac{1}{t(t+1)}$ and hence $\kappa_{f_{\mathrm{JS}}}(t)=\frac{1}{t+1}\le 1$.
\item \textbf{Triangular discrimination:} one generator is
$f_{\triangle}(t)=\frac{(t-1)^2}{t+1}$, for which $f_{\triangle}''(t)=\frac{8}{(t+1)^3}$ and hence
$\kappa_{f_{\triangle}}(t)=\frac{8t}{(t+1)^3}\le \frac{32}{27}$.
\end{itemize}
By contrast, the condition fails for many popular divergences whose curvature blows up either as $t\downarrow 0$ or $t\uparrow\infty$, e.g.
\begin{itemize}
\item \textbf{Squared Hellinger:} $f(t)=(\sqrt{t}-1)^2$ gives $\kappa_f(t)=\frac{1}{2\sqrt{t}}$ (unbounded as $t\downarrow 0$).
\item \textbf{Pearson $\chi^2$:} $f(t)=(t-1)^2$ gives $\kappa_f(t)=2t$ (unbounded as $t\uparrow\infty$).
\item \textbf{Neyman $\chi^2$:} $f(t)=\frac{(1-t)^2}{t}$ gives $\kappa_f(t)=2/t^2$ (unbounded as $t\downarrow 0$).
\item \textbf{Power/$\alpha$-divergences:} for
$f_\alpha(t)=\frac{t^\alpha-\alpha(t-1)-1}{\alpha(\alpha-1)}$ one has $\kappa_{f_\alpha}(t)=t^{\alpha-1}$, which is unbounded for every $\alpha\neq 1$
(as $t\downarrow 0$ when $\alpha<1$ and as $t\uparrow\infty$ when $\alpha>1$); the only bounded-curvature member of this family is the $\alpha\to 1$ limit (KL).
\end{itemize}
Finally, some widely used discrepancies (e.g.\ total variation with $f(t)=\tfrac12|t-1|$) are not $C^2$ and therefore fall outside the present smooth framework.

In all these cases, the \emph{exact} decomposition of Theorem~\ref{thm:fmain}(B) still holds, but one must bound the weighted terms
$\E_{p_{\mathrm{o}}}[\kappa_f(w(Y))(1-r_{\mathrm{o}}(Y))]$ and $\E_{p_{\mathrm{o}}}[\kappa_f(w(Y))(1-s_{\mathrm{o}}(Y))]$
by exploiting additional structure (e.g.\ explicit tail control of $w$ under $p_{\mathrm{o}}$, clipping/regularization, or refined overlap estimates).
\end{remark}

\section{Finite $K$-mode Gaussian mixtures}\label{app:finitek}

The two-mode analysis makes the separation between forward- and reverse-KL especially transparent, but real models are typically multi-modal. We therefore extend the picture to a finite $K$-component Gaussian mixture with shared covariance.
The goal is two fold: first, to show that the \emph{mode-locality} property of reverse-KL persists in the multi-mode setting, with gradients on a matched mode controlled by pairwise overlaps; and second, to show that forward-KL trained on a subset of modes still induces exact weight collapse on the complement.
Together, these results demonstrate that the qualitative foward-vs.-reverse-KL contrast is not an artifact of the $K=2$ case.

\begin{lemma}[Linear independence of Fixed-covariance Gaussian Translates]\label{lem:gauss_translate_indep}
Fix $\Sigma\succ 0$ and distinct means $\mu_1,\dots,\mu_K\in\mathbb{R}^d$.
If coefficients $c_1,\dots,c_K\in\mathbb{R}$ satisfy
\[
\sum_{k=1}^K c_k\,\gauss{y}{\mu_k}{\Sigma}=0\qquad\forall y\in\mathbb{R}^d,
\]
then $c_1=\cdots=c_K=0$.
\end{lemma}

\begin{proof}
Take Fourier transforms. The Fourier transform of $\gauss{\cdot}{\mu_k}{\Sigma}$ equals
$e^{-i t^\top \mu_k}\,e^{-\frac12 t^\top\Sigma t}$.
Hence the assumption implies
\[
0=\sum_{k=1}^K c_k\,e^{-i t^\top \mu_k}\,e^{-\frac12 t^\top\Sigma t}
=e^{-\frac12 t^\top\Sigma t}\sum_{k=1}^K c_k\,e^{-i t^\top \mu_k}
\qquad\forall t\in\mathbb{R}^d.
\]
Since $e^{-\frac12 t^\top\Sigma t}>0$, we have $\sum_k c_k e^{-i t^\top \mu_k}=0$ for all $t$.
Choose a vector $v\in\mathbb{R}^d$ such that the scalars $a_k:=v^\top\mu_k$ are pairwise distinct
(this holds for all $v$ outside a finite union of hyperplanes).
Then for all $s\in\mathbb{R}$,
\[
0=\sum_{k=1}^K c_k e^{-i s a_k}.
\]
Differentiating $n=0,\dots,K-1$ times at $s=0$ gives the Vandermonde system
$\sum_k c_k(-i a_k)^n=0$, whose coefficient matrix is invertible since the $a_k$ are distinct.
Hence $c_k=0$ for all $k$.
\end{proof}

\begin{lemma}[Pairwise Responsibility Upper Bound in a $K$-mixture]\label{lem:pairwise_resp_bound}
Let $q(y)=\sum_{\ell=1}^K \beta_\ell f_\ell(y)$ be a mixture of densities with weights $\beta_\ell>0$.
Define responsibilities $r_j(y)=\frac{\beta_j f_j(y)}{q(y)}$.
Then for any $j\neq k$ and all $y$,
\[
r_j(y)\le \frac{\beta_j f_j(y)}{\beta_j f_j(y)+\beta_k f_k(y)}.
\]
Consequently, for $Y\sim f_k$,
\[
\E[r_j(Y)]
\le \frac12\sqrt{\frac{\beta_j}{\beta_k}}\;\mathrm{BC}(f_j,f_k).
\]
\end{lemma}

\begin{proof}
The pointwise bound follows since $q(y)\ge \beta_j f_j(y)+\beta_k f_k(y)$.
For the expectation, apply Lemma~\ref{lem:leakage} to the two-component mixture
$(\beta_j f_j+\beta_k f_k)/(\beta_j+\beta_k)$ and use the definition of the Bhattacharyya coefficient.
\end{proof}

\begin{theorem}\label{thm:K_mode_gauss}
Fix $\Sigma\succ 0$, distinct means $\mu_1,\dots,\mu_K\in\mathbb{R}^d$, and weights $\alpha\in\Delta^{K-1}$ with $\alpha_k>0$.
Define the target mixture
\[
p(y):=\sum_{k=1}^K \alpha_k\,\gauss{y}{\mu_k}{\Sigma},
\qquad
s_k(y):=\frac{\alpha_k\,\gauss{y}{\mu_k}{\Sigma}}{p(y)}.
\]
Let the model be
\[
q(y):=\sum_{k=1}^K \beta_k\,\gauss{y}{m_k}{\Sigma},
\qquad
r_k(y):=\frac{\beta_k\,\gauss{y}{m_k}{\Sigma}}{q(y)},
\qquad \beta\in\Delta^{K-1},\ \beta_k>0.
\]
\medskip
\noindent\textbf{(A)} Let $T\subset\{1,\dots,K\}$ be nonempty and define
\[
p_T(y):=\sum_{k\in T}\tilde\alpha_k\,\gauss{y}{\mu_k}{\Sigma},
\qquad
\tilde\alpha_k:=\frac{\alpha_k}{\sum_{j\in T}\alpha_j}.
\]
Fix $m_k=\mu_k$ for all $k$ and optimize only $\beta$.
Then $\KL(p_T\|q_\beta)$ has the unique minimizer
\[
\beta_k^\star=
\begin{cases}
\tilde\alpha_k,& k\in T,\\
0,& k\notin T.
\end{cases}
\]

\medskip
\noindent\textbf{(B)} Fix $k$ and assume $m_k=\mu_k$.
Then
\[
\nabla_{m_k}\KL(q\|p)
=
\beta_k\,\Sigma^{-1}\left(
\sum_{j\neq k} \varepsilon^{(q)}_{k\to j}\,(m_j-\mu_k)
-
\sum_{j\neq k} \varepsilon^{(p)}_{k\to j}\,(\mu_j-\mu_k)
\right),
\]
where
\[
\varepsilon^{(q)}_{k\to j}:=\E_{Y\sim \mathcal{N}(\mu_k,\Sigma)}[r_j(Y)],
\qquad
\varepsilon^{(p)}_{k\to j}:=\E_{Y\sim \mathcal{N}(\mu_k,\Sigma)}[s_j(Y)].
\]
Moreover,
\[
\varepsilon^{(q)}_{k\to j}
\le \frac12\sqrt{\frac{\beta_j}{\beta_k}}\,
\exp\!\left(-\frac18\,\mnorm{m_j-\mu_k}{\Sigma^{-1}}^{2}\right),
\qquad
\varepsilon^{(p)}_{k\to j}
\le \frac12\sqrt{\frac{\alpha_j}{\alpha_k}}\,
\exp\!\left(-\frac18\,\mnorm{\mu_j-\mu_k}{\Sigma^{-1}}^{2}\right).
\]
\end{theorem}

\begin{proof}\textbf{Proof of Part (A):} If $\KL(p_T\|q_\beta)=0$, then $q_\beta\equiv p_T$ almost everywhere.
By Lemma~\ref{lem:gauss_translate_indep}, the Gaussian translates are linearly independent,
so coefficients must match exactly, giving the stated $\beta^\star$.
\medskip

\textbf{Proof of Part (B):} Differentiate $\KL(q\|p)$ w.r.t.\ $m_k$ and apply the same score-difference computation as in the two-mode case.
Using $\sum_\ell r_\ell=1$ and $\sum_\ell s_\ell=1$, collect terms to obtain the exact decomposition.
The exponential bounds follow from Lemma~\ref{lem:pairwise_resp_bound} and Remark~\ref{lem:bc-gauss}.

\end{proof}

\begin{remark}[Multi-mode Forgetting: Local vs.\ Global Effects]
The multi-mode result reveals how the two forms of forgetting introduced earlier, \emph{mass forgetting} and \emph{component drift}, extend beyond the two-mode setting.

Part~(A) characterizes \emph{mass forgetting} under forward-KL objectives. When training data come only from a subset of modes $T$, the forward-KL objective $\KL(p_T\|q_\beta)$ is minimized exactly by allocating mixture mass only to those observed modes. All components $k\notin T$ must receive zero optimal weight, $\beta_k^\star=0$.
Thus forward-KL induces \emph{mass collapse}: any behavior not represented in the training distribution is eliminated at the population optimum. This formalizes catastrophic forgetting in the mixture model as a global reallocation of mixture mass driven by the support of the data distribution.

Part~(B) characterizes \emph{component drift} under reverse-KL objectives. If a component $k$ is already correctly placed ($m_k=\mu_k$), its gradient depends only on pairwise overlaps with the remaining modes through the misassignment probabilities $\varepsilon^{(q)}_{k\to j}$ and $\varepsilon^{(p)}_{k\to j}$. These quantities measure how often samples from mode $k$ are attributed to another mode $j$ under the model or the target. When the modes are well separated, the overlap bounds show that these probabilities decay exponentially in the Mahalanobis separation between modes. Consequently, the update signal acting on an already-correct component becomes exponentially small, so reverse-KL updates can adjust other modes while inducing only negligible drift on matched ones.In continual-learning terms, this means that previously learned behaviors (represented by correctly matched components) are locally protected.

Taken together, the multi-mode analysis reinforces the qualitative contrast observed in the two-mode case.
Reverse-KL objectives exhibit \emph{mode-local} updates that protect matched components up to exponentially small overlap effects, whereas forward-KL objectives induce \emph{global} mass reallocation, collapsing components absent from the training data.
\end{remark}

% \begin{remark}[Multi-mode Forgetting: Local vs.\ Global Effects]
% Part (A) shows that reverse-KL gradients remain \emph{mode-local} even with many components:
% if mode $k$ is already matched in mean, its gradient depends only on pairwise overlaps with other modes.
% When separations are large, these overlaps decay exponentially, so matched modes are effectively protected.
% In contrast, Part (B) shows that forward-KL trained on a subset of modes induces exact weight collapse on all others:
% any component absent from the training distribution must receive zero optimal weight.
% Thus the multi-mode setting preserves the same qualitative contrast as the two-mode case:
% reverse-KL protects matched modes up to exponentially small overlap,
% while forward-KL reallocates mass globally according to the support of the data distribution.
% \end{remark}

\section{From Gaussians to (Strongly) Log-Concave Location Families}\label{app:logconcaveext}

This section shows that the two main mechanisms from the Gaussian analysis persist for much broader
\emph{log-concave} component families.
The key difference is that in the Gaussian case the mixture-score differences become \emph{constants} (linear scores),
whereas for general log-concave components we obtain \emph{overlap-controlled bounds} that depend on (i) a smoothness constant
for the score map and (ii) an overlap quantity such as the Bhattacharyya coefficient.

\paragraph{Log-concave Location Family.}
Let $V:\mathbb{R}^d\to\mathbb{R}$ be $C^2$ and convex, and define the log-concave density
\[
\rho(x)=\frac{1}{Z}\exp(-V(x)),\qquad Z:=\int_{\mathbb{R}^d}e^{-V(x)}\,dx<\infty.
\]
For $\mu\in\mathbb{R}^d$, let the \emph{location-shift} density be
\[
\rho_\mu(y):=\rho(y-\mu).
\]
Then each $\rho_\mu$ is log-concave and strictly positive.

\paragraph{Old/New Components and Mixtures.}
Fix $\mu_{\mathrm{o}},\mu_{\mathrm{n}}\in\mathbb{R}^d$ and $\alpha\in(0,1)$, and define
\[
p_\alpha(y)=\alpha\,\rho_{\mu_{\mathrm{o}}}(y)+(1-\alpha)\,\rho_{\mu_{\mathrm{n}}}(y).
\]
For parameters $\beta\in(0,1)$ and $m_{\mathrm{n}}\in\mathbb{R}^d$ (with $m_{\mathrm{o}}$ fixed to $\mu_{\mathrm{o}}$), define
\[
q_{\beta,m_{\mathrm{n}}}(y)=\beta\,\rho_{\mu_{\mathrm{o}}}(y)+(1-\beta)\,\rho_{m_{\mathrm{n}}}(y),
\qquad
L_{\mathrm{RL}}(\beta,m_{\mathrm{n}}):=\KL\!\big(q_{\beta,m_{\mathrm{n}}}\,\|\,p_\alpha\big).
\]
Define responsibilities
\[
r_{\mathrm{o}}(y):=\frac{\beta\,\rho_{\mu_{\mathrm{o}}}(y)}{q_{\beta,m_{\mathrm{n}}}(y)},\qquad
s_{\mathrm{o}}(y):=\frac{\alpha\,\rho_{\mu_{\mathrm{o}}}(y)}{p_\alpha(y)}.
\]

We start by the following standard identity with the proof provided for completeness.

\begin{lemma}[Integration-by-parts Identity for Location Parameters]\label{lem:ibp_location}
Let $\rho$ be as above and assume additionally that $\rho\in C^1$ and that
\[
\lim_{R\to\infty}\int_{\|y\|=R}\rho_\mu(y)\,\big|g(y)\big|\,dS(y)=0
\]
for every $\mu$ and every $C^1$ function $g$ appearing below (this holds, e.g., if $g$ has at most polynomial growth
and $\rho_\mu$ has at least exponential tails, which is true for many log-concave families used in practice).
Then for any $C^1$ function $g:\mathbb{R}^d\to\mathbb{R}$ with suitable integrability,
\[
\nabla_\mu\int_{\mathbb{R}^d}\rho_\mu(y)\,g(y)\,dy
=
\int_{\mathbb{R}^d}\rho_\mu(y)\,\nabla_y g(y)\,dy.
\]
\end{lemma}

\begin{proof}
Because $\rho_\mu(y)=\rho(y-\mu)$, we have $\nabla_\mu \rho_\mu(y)=-\nabla_y\rho_\mu(y)$.
Differentiating under the integral (justified by dominated convergence under the integrability assumptions) yields
\[
\nabla_\mu\int \rho_\mu(y)g(y)\,dy = \int (\nabla_\mu\rho_\mu(y))g(y)\,dy
= -\int (\nabla_y\rho_\mu(y))g(y)\,dy.
\]
Integrate by parts on $\mathbb{R}^d$:
\[
-\int (\nabla_y\rho_\mu)g\,dy
= \int \rho_\mu\,\nabla_y g\,dy - \lim_{R\to\infty}\int_{\|y\|=R}\rho_\mu(y)\,g(y)\,n(y)\,dS(y),
\]
where $n(y)$ is the outward normal. The boundary term vanishes by assumption, giving the claim.
\end{proof}

\begin{lemma}[Bhattacharyya Coefficient Bound for Strongly Log-concave Shifts]\label{lem:bc_strong_logconcave}
Assume $V$ is $m$-strongly convex for some $m>0$, i.e.\ $\nabla^2V(x)\succeq m I$ for all $x$.
Then for any $\mu_1,\mu_2\in\mathbb{R}^d$,
\[
\mathrm{BC}(\rho_{\mu_1},\rho_{\mu_2})
:=\int_{\mathbb{R}^d}\sqrt{\rho_{\mu_1}(y)\rho_{\mu_2}(y)}\,dy
\ \le\
\exp\!\left(-\frac{m}{8}\,\|\mu_1-\mu_2\|^2\right).
\]
\end{lemma}

\begin{proof}
Let $\Delta:=\mu_1-\mu_2$ and write $\rho_\mu(y)=Z^{-1}\exp(-V(y-\mu))$.
Then
\[
\sqrt{\rho_{\mu_1}(y)\rho_{\mu_2}(y)}
=\frac{1}{Z}\exp\!\left(-\frac{V(y-\mu_1)+V(y-\mu_2)}{2}\right).
\]
Apply the strong convexity midpoint inequality (equivalent to $\nabla^2V\succeq mI$):
for all $u,v$,
\[
V(u)+V(v)\ \ge\ 2V\!\left(\frac{u+v}{2}\right)+\frac{m}{4}\|u-v\|^2.
\]
With $u=y-\mu_1$ and $v=y-\mu_2$, we have $(u+v)/2 = y-(\mu_1+\mu_2)/2$ and $u-v=\mu_2-\mu_1=-\Delta$,
so
\[
\frac{V(y-\mu_1)+V(y-\mu_2)}{2}
\ \ge\
V\!\left(y-\frac{\mu_1+\mu_2}{2}\right)+\frac{m}{8}\|\Delta\|^2.
\]
Therefore
\[
\sqrt{\rho_{\mu_1}(y)\rho_{\mu_2}(y)}
\le
\frac{1}{Z}\exp\!\left(-V\!\left(y-\frac{\mu_1+\mu_2}{2}\right)\right)\exp\!\left(-\frac{m}{8}\|\Delta\|^2\right)
=
\rho_{(\mu_1+\mu_2)/2}(y)\,e^{-m\|\Delta\|^2/8}.
\]
Integrating over $y$ and using $\int \rho_{(\mu_1+\mu_2)/2}(y)\,dy=1$ yields the bound.
\end{proof}

\begin{theorem}\label{thm:kl_logconcave_contrast}
Assume $\rho$ is a $C^2$ log-concave density of the form $\rho(x)\propto e^{-V(x)}$ with $V$ convex.
Fix $\mu_{\mathrm{o}}\neq \mu_{\mathrm{n}}$ and $\alpha\in(0,1)$, and define $p_\alpha$ and $q_{\beta,m_{\mathrm{n}}}$ as above.

\medskip
\noindent\textbf{(A)} For $\beta\in[0,1]$, define $q_\beta^{\mathrm{new}} := \beta\,\rho_{\mu_{\mathrm{o}}}+(1-\beta)\,\rho_{\mu_{\mathrm{n}}}$ and
\[
L_{\mathrm{SFT}}(\beta):=\KL\!\big(\rho_{\mu_{\mathrm{n}}}\,\|\,q_\beta^{\mathrm{new}}\big).
\]
Then $L_{\mathrm{SFT}}(0)=0$ and $L_{\mathrm{SFT}}(\beta)>0$ for every $\beta\in(0,1]$; moreover $L_{\mathrm{SFT}}$ is strictly increasing on $[0,1]$.
In particular, the unique minimizer is $\beta^\star=0$.

\medskip
\noindent\textbf{(B)} Assume additionally that $\nabla V$ is $L$-Lipschitz (equivalently, $\nabla^2V(x)\preceq L I$ for all $x$), so that the score map
\[
u(y;\mu):=\nabla_y\log\rho_\mu(y)=-\nabla V(y-\mu)
\]
satisfies the uniform Lipschitz property
\[
\|u(y;\mu_1)-u(y;\mu_2)\|\le L\|\mu_1-\mu_2\|\qquad\forall y,\mu_1,\mu_2.
\]
Then $L_{\mathrm{RL}}(\beta,m_{\mathrm{n}})=\KL(q_{\beta,m_{\mathrm{n}}}\|p_\alpha)$ is differentiable and its gradient with respect to the \emph{old} location parameter
$m_{\mathrm{o}}$ (evaluated at $m_{\mathrm{o}}=\mu_{\mathrm{o}}$) obeys the bound
\[
\Big\|\nabla_{m_{\mathrm{o}}}\KL\!\big(q_{\beta,\mu_{\mathrm{o}},m_{\mathrm{n}}}\,\|\,p_\alpha\big)\Big\|
\ \le\
\beta\,L\Big(
\varepsilon_q(\beta,m_{\mathrm{n}})\,\|m_{\mathrm{n}}-\mu_{\mathrm{o}}\|
+
\varepsilon_p(\alpha)\,\|\mu_{\mathrm{n}}-\mu_{\mathrm{o}}\|
\Big),
\]
where the misassignment probabilities are
\[
\varepsilon_q(\beta,m_{\mathrm{n}}):=\E_{Y\sim\rho_{\mu_{\mathrm{o}}}}\big[1-r_{\mathrm{o}}(Y)\big],\qquad
\varepsilon_p(\alpha):=\E_{Y\sim\rho_{\mu_{\mathrm{o}}}}\big[1-s_{\mathrm{o}}(Y)\big].
\]
Moreover, for any densities (no log-concavity needed), Lemma~\ref{lem:leakage} implies the overlap bounds
\[
\varepsilon_q(\beta,m_{\mathrm{n}})
\le \frac12\sqrt{\frac{1-\beta}{\beta}}\;\mathrm{BC}(\rho_{\mu_{\mathrm{o}}},\rho_{m_{\mathrm{n}}}),
\qquad
\varepsilon_p(\alpha)
\le \frac12\sqrt{\frac{1-\alpha}{\alpha}}\;\mathrm{BC}(\rho_{\mu_{\mathrm{o}}},\rho_{\mu_{\mathrm{n}}}).
\]
If, in addition, $V$ is $m$-strongly convex for some $m>0$, then by Lemma~\ref{lem:bc_strong_logconcave},
\[
\varepsilon_q(\beta,m_{\mathrm{n}})
\le \frac12\sqrt{\frac{1-\beta}{\beta}}\exp\!\left(-\frac{m}{8}\|m_{\mathrm{n}}-\mu_{\mathrm{o}}\|^2\right),
\qquad
\varepsilon_p(\alpha)
\le \frac12\sqrt{\frac{1-\alpha}{\alpha}}\exp\!\left(-\frac{m}{8}\|\mu_{\mathrm{n}}-\mu_{\mathrm{o}}\|^2\right),
\]
so the old-location gradient is exponentially small in the separation when modes are well separated.

\medskip
\noindent\textbf{(C)} Let $L(\beta,m_{\mathrm{n}})=\KL(q_{\beta,m_{\mathrm{n}}}\|p_\alpha)$ with $m_{\mathrm{o}}=\mu_{\mathrm{o}}$ fixed.
Then $(\beta,m_{\mathrm{n}})=(\alpha,\mu_{\mathrm{n}})$ satisfies $q_{\alpha,\mu_{\mathrm{n}}}\equiv p_\alpha$, hence
\[
\left.\frac{\partial}{\partial\beta}L(\beta,m_{\mathrm{n}})\right|_{(\beta,m_{\mathrm{n}})=(\alpha,\mu_{\mathrm{n}})}=0,
\qquad
\left.\nabla_{m_{\mathrm{n}}}L(\beta,m_{\mathrm{n}})\right|_{(\beta,m_{\mathrm{n}})=(\alpha,\mu_{\mathrm{n}})}=0,
\]
and $L(\alpha,\mu_{\mathrm{n}})=0$.
\end{theorem}

\begin{proof}
\textbf{Proof of Part (A):} Write $q_\beta^{\mathrm{new}}=(1-\beta)\rho_{\mu_{\mathrm{n}}}+\beta\rho_{\mu_{\mathrm{o}}}$ and define the likelihood ratio
$X(y):=\rho_{\mu_{\mathrm{o}}}(y)/\rho_{\mu_{\mathrm{n}}}(y)$.
Under $Y\sim\rho_{\mu_{\mathrm{n}}}$ we have $\E[X(Y)]=\int \rho_{\mu_{\mathrm{o}}}=1$.
Then
\[
\KL(\rho_{\mu_{\mathrm{n}}}\|q_\beta^{\mathrm{new}})
= \E_{\rho_{\mu_{\mathrm{n}}}}\!\left[\log\frac{\rho_{\mu_{\mathrm{n}}}(Y)}{(1-\beta)\rho_{\mu_{\mathrm{n}}}(Y)+\beta\rho_{\mu_{\mathrm{o}}}(Y)}\right]
= -\E\Big[\log\big((1-\beta)+\beta X(Y)\big)\Big].
\]
By concavity of $\log$,
\[
\E\Big[\log\big((1-\beta)+\beta X(Y)\big)\Big]\le \log\Big((1-\beta)+\beta\E[X(Y)]\Big)=\log(1)=0,
\]
with strict inequality for every $\beta>0$ because $X(Y)$ is non-constant when $\mu_{\mathrm{o}}\neq\mu_{\mathrm{n}}$ (two distinct shifts of a positive density cannot coincide a.e.).
Therefore $L_{\mathrm{SFT}}(0)=0$ and $L_{\mathrm{SFT}}(\beta)>0$ for $\beta>0$.

To show strict increase: the map $g(\beta):=\E[\log((1-\beta)+\beta X)]$ is strictly concave in $\beta$ whenever $X$ is non-degenerate,
since $\log$ is strictly concave and $(1-\beta)+\beta X$ is affine in $\beta$ with nonzero randomness.
Because $g(0)=0$ and $g(\beta)<0$ for all $\beta>0$, strict concavity implies $g$ is strictly decreasing on $[0,1]$.
Hence $L_{\mathrm{SFT}}(\beta)=-g(\beta)$ is strictly increasing.

\medskip
\textbf{Proof of Part (B):} Let $p=p_\alpha$ and $q=q_{\beta,m_{\mathrm{n}}}$, and denote $m_{\mathrm{o}}=\mu_{\mathrm{o}}$.
Using the general identity $\nabla_\theta \KL(q_\theta\|p)=\int (\nabla_\theta q_\theta)\log(q_\theta/p)$ (as in Theorem~\ref{thm:kl_stationary_beta_mn}),
together with $\nabla_{m_{\mathrm{o}}}\rho_{m_{\mathrm{o}}}(y)=-\nabla_y\rho_{m_{\mathrm{o}}}(y)$ and Lemma~\ref{lem:ibp_location},
one obtains
\[
\nabla_{m_{\mathrm{o}}}\KL(q\|p)
=
\beta\int \rho_{\mu_{\mathrm{o}}}(y)\,\nabla_y\log\frac{q(y)}{p(y)}\,dy
= \beta\,\E_{Y\sim\rho_{\mu_{\mathrm{o}}}}\!\left[\nabla_y\log\frac{q(Y)}{p(Y)}\right].
\]
Next expand mixture scores:
\[
\nabla_y\log q(y)= r_{\mathrm{o}}(y)\,u(y;\mu_{\mathrm{o}})+(1-r_{\mathrm{o}}(y))\,u(y;m_{\mathrm{n}}),\qquad
\nabla_y\log p(y)= s_{\mathrm{o}}(y)\,u(y;\mu_{\mathrm{o}})+(1-s_{\mathrm{o}}(y))\,u(y;\mu_{\mathrm{n}}).
\]
Subtracting gives
\[
\nabla_y\log\frac{q(y)}{p(y)}
= (1-r_{\mathrm{o}}(y))\big(u(y;m_{\mathrm{n}})-u(y;\mu_{\mathrm{o}})\big)
-(1-s_{\mathrm{o}}(y))\big(u(y;\mu_{\mathrm{n}})-u(y;\mu_{\mathrm{o}})\big).
\]
Taking norms and using the Lipschitz bound on $u$ yields
\begin{align*}
\Big\|\nabla_{m_{\mathrm{o}}}\KL(q\|p)\Big\|
&\le
\beta\,\E_{\rho_{\mu_{\mathrm{o}}}}\!\left[(1-r_{\mathrm{o}}(Y))\|u(Y;m_{\mathrm{n}})-u(Y;\mu_{\mathrm{o}})\|\right]
+
\beta\,\E_{\rho_{\mu_{\mathrm{o}}}}\!\left[(1-s_{\mathrm{o}}(Y))\|u(Y;\mu_{\mathrm{n}})-u(Y;\mu_{\mathrm{o}})\|\right]\\
&\le
\beta\,L\Big(\E[1-r_{\mathrm{o}}(Y)]\,\|m_{\mathrm{n}}-\mu_{\mathrm{o}}\|
+\E[1-s_{\mathrm{o}}(Y)]\,\|\mu_{\mathrm{n}}-\mu_{\mathrm{o}}\|\Big),
\end{align*}
which is the stated inequality with $\varepsilon_q$ and $\varepsilon_p$.
The bounds on $\varepsilon_q$ and $\varepsilon_p$ in terms of $\mathrm{BC}$ follow directly from Lemma~\ref{lem:leakage}.
Under strong convexity, apply Lemma~\ref{lem:bc_strong_logconcave} to the relevant shifted pairs.

\medskip
\textbf{Proof of Part (C):} At $(\beta,m_{\mathrm{n}})=(\alpha,\mu_{\mathrm{n}})$ we have $q_{\alpha,\mu_{\mathrm{n}}}\equiv p_\alpha$, so $\log(q/p)\equiv 0$.
Therefore any gradient formula of the form $\nabla_\theta\KL(q_\theta\|p)=\int (\nabla_\theta q_\theta)\log(q_\theta/p)$ evaluates to zero.
Also $\KL(q\|p)\ge 0$ with equality iff $q=p$ a.e., so $L(\alpha,\mu_{\mathrm{n}})=0$.
\end{proof}

\begin{remark}[Relation to the Gaussian bounds]
For Gaussians with covariance $\Sigma$, one may take $V(x)=\tfrac12 x^\top\Sigma^{-1}x$,
so $m=\lambda_{\min}(\Sigma^{-1})$ and $L=\lambda_{\max}(\Sigma^{-1})$.
Then Lemma~\ref{lem:bc_strong_logconcave} recovers the familiar Gaussian overlap decay
$\mathrm{BC}\le \exp(-\|\mu_1-\mu_2\|_{\Sigma^{-1}}^2/8)$ up to replacing the Mahalanobis norm by its spectral bounds.
Moreover the Gaussian score is linear, which strengthens Part (B) from a bound to the exact identity derived earlier.
\end{remark}

\subsection{Local PL Geometry and Exponential Convergence for Strongly Log-concave Mixtures}

We also provide a qualitative local PL condition along with exponential convergence in this case. We start with the following standard result with proof provided for completeness.

\begin{lemma}[Fisher identity for strongly log-concave location families]\label{lem:fisher_logconcave_location}
Let $\rho(x)=Z^{-1}e^{-V(x)}$ on $\mathbb{R}^d$ with $V\in C^2$ and $\int e^{-V}<\infty$.
Assume integration by parts is valid (e.g.\ $V$ grows superlinearly so that boundary terms vanish).
If $X\sim \rho$, then
\[
\E\big[\nabla V(X)\,\nabla V(X)^\top\big]=\E\big[\nabla^2 V(X)\big].
\]
In particular, if $\nabla^2V(x)\succeq m I$ for all $x$ (i.e.\ $V$ is $m$-strongly convex), then
$\E[\nabla V(X)\nabla V(X)^\top]\succeq m I$.
\end{lemma}
\begin{proof}
For each $i,j$, apply integration by parts with density $\rho$:
\[
0=\int \partial_i\big(e^{-V(x)}\,\partial_j V(x)\big)\,dx
=\int e^{-V(x)}\big(\partial_{ij}V(x)-\partial_i V(x)\,\partial_j V(x)\big)\,dx.
\]
Divide by $Z$ to obtain $\E[\partial_{ij}V(X)]=\E[\partial_iV(X)\partial_jV(X)]$.
\end{proof}

\begin{theorem}\label{thm:logconcave_local_rate}
Let $\rho(x)=Z^{-1}e^{-V(x)}$ on $\mathbb{R}^d$, where $V\in C^3$ is $m$-strongly convex ($\nabla^2V\succeq mI$) and $L$-smooth ($\nabla^2V\preceq L I$).
Assume $\E_{X\sim \rho}\big[\|\nabla V(X)\|^4\big]<\infty$.
For $\mu\in\mathbb{R}^d$ define the shifted density $\rho_\mu(y):=\rho(y-\mu)$.
Fix $\mu_{\mathrm{o}}\neq \mu_{\mathrm{n}}$ and $\alpha\in(0,1)$, and define
\[
p_\alpha(y):=\alpha\,\rho_{\mu_{\mathrm{o}}}(y)+(1-\alpha)\,\rho_{\mu_{\mathrm{n}}}(y).
\]
Fix $m_{\mathrm{o}}=\mu_{\mathrm{o}}$ and parameterize the model by $\theta=(\phi,m)\in\mathbb{R}\times\mathbb{R}^d$:
\[
\beta(\phi)=\sigma(\phi),\qquad
q_\theta(y)=\beta(\phi)\,\rho_{\mu_{\mathrm{o}}}(y)+(1-\beta(\phi))\,\rho_{m}(y),
\qquad
L(\theta)=\KL(q_\theta\|p_\alpha).
\]
Let $\theta^\star=(\phi^\star,m^\star)$ with $\phi^\star=\log\frac{\alpha}{1-\alpha}$ and $m^\star=\mu_{\mathrm{n}}$, so $q_{\theta^\star}\equiv p_\alpha$ and $L(\theta^\star)=0$.

\medskip
\noindent\textbf{(A)} $L$ is $C^2$ in a neighborhood of $\theta^\star$ and
\[
H_\star:=\nabla^2 L(\theta^\star)=\E_{Y\sim p_\alpha}\big[s(Y)\,s(Y)^\top\big],
\]
where the score vector is
\[
s(Y)=\begin{pmatrix}
r_{\mathrm{o}}^\star(Y)-\alpha\\[2pt]
r_{\mathrm{n}}^\star(Y)\,\nabla V(Y-\mu_{\mathrm{n}})
\end{pmatrix},
\qquad
r_{\mathrm{o}}^\star(y)=\frac{\alpha\,\rho_{\mu_{\mathrm{o}}}(y)}{p_\alpha(y)},\qquad r_{\mathrm{n}}^\star=1-r_{\mathrm{o}}^\star.
\]

\medskip
\noindent\textbf{(B)} Let $\Delta:=\mu_{\mathrm{n}}-\mu_{\mathrm{o}}$ and define the overlap proxy
\[
\rho_{\mathrm{sep}}:=\exp\!\left(-\frac{m}{8}\,\|\Delta\|^2\right),
\]
which upper bounds $\mathrm{BC}(\rho_{\mu_{\mathrm{o}}},\rho_{\mu_{\mathrm{n}}})$ by Lemma~\ref{lem:bc_strong_logconcave}.
Define
\[
\varepsilon_{\mathrm{o}\to \mathrm{n}}:=\frac12\sqrt{\frac{1-\alpha}{\alpha}}\rho_{\mathrm{sep}},
\qquad
\varepsilon_{\mathrm{n}\to \mathrm{o}}:=\frac12\sqrt{\frac{\alpha}{1-\alpha}}\rho_{\mathrm{sep}},
\qquad
v:=\sqrt{\alpha(1-\alpha)}\,\rho_{\mathrm{sep}}.
\]
Let $G_2:=\E_{Y\sim p_\alpha}[\|\nabla V(Y-\mu_{\mathrm{n}})\|^2]$ and $G_4:=\E_{X\sim\rho}[\|\nabla V(X)\|^4]$.
Then
\[
\lambda_{\min}(H_\star)
\ \ge\
\min\Big\{
\alpha(1-\alpha)-v,\;\;
(1-\alpha)m - 2(1-\alpha)\sqrt{\varepsilon_{\mathrm{n}\to \mathrm{o}}}\,\sqrt{G_4}
\Big\}
\;-\;3\sqrt{v}\,\sqrt{G_2}.
\]
In particular, for $\|\Delta\|$ large enough (so that the right-hand side is positive), $\lambda_{\min}(H_\star)>0$.

\medskip
\noindent\textbf{(C)} If $\mu_\star:=\lambda_{\min}(H_\star)>0$, then there exists $\varepsilon>0$ such that on the sublevel set
$\{\theta:\,L(\theta)\le \varepsilon\}$ the Polyak--\L ojasiewicz inequality holds:
\[
\|\nabla L(\theta)\|^2 \ \ge\ \frac{\mu_\star}{2}\,L(\theta).
\]
Consequently, any gradient-flow solution $\dot\theta(t)=-\nabla L(\theta(t))$ with $L(\theta(0))\le\varepsilon$
satisfies the exponential rate
\[
L(\theta(t))\le L(\theta(0))\,e^{-(\mu_\star/2)t}\qquad\forall t\ge 0.
\]
\end{theorem}

\begin{proof}
\textbf{Proof of Part (A):} The Fisher/Hessian identity at $\theta^\star$ is standard for smooth parametric families:
since $q_{\theta^\star}=p_\alpha$, one has $L(\theta)=\KL(q_\theta\|q_{\theta^\star})$ and hence
$\nabla^2 L(\theta^\star)=\E_{q_{\theta^\star}}[\nabla\log q_{\theta^\star}\,\nabla\log q_{\theta^\star}^\top]$,
provided differentiation under the integral is justified (here ensured by smoothness and log-concave tails).
The score components follow from:
(i) $\partial_\phi \log q_\theta = r_{\mathrm{o}}-\beta(\phi)$ for a two-component mixture, and at $\theta^\star$ $\beta(\phi^\star)=\alpha$;
(ii) $\nabla_m\log q_\theta(y)=r_{\mathrm{n}}(y)\,\nabla_m\log\rho_m(y)$ and
$\nabla_m\log\rho_m(y)=\nabla V(y-m)$ for location shifts $\rho_m(y)=Z^{-1}e^{-V(y-m)}$.

\medskip
\textbf{Proof of Part (B):} As in the Gaussian proof of Theorem~\ref{thm:kl_RL_local_rate}, write $H_\star=\begin{psmallmatrix}A&B^\top\\B&C\end{psmallmatrix}$.
Introduce the latent label $Z\in\{0,1\}$ with $\Pr(Z=1)=\alpha$ and
$Y|Z=1\sim\rho_{\mu_{\mathrm{o}}}$, $Y|Z=0\sim\rho_{\mu_{\mathrm{n}}}$; then $r_{\mathrm{o}}^\star(Y)=\E[Z|Y]$ and
$e:=r_{\mathrm{o}}^\star-Z$ satisfies $e^2=\Var(Z|Y)=r_{\mathrm{o}}^\star(1-r_{\mathrm{o}}^\star)$.

\emph{(i) Bound $A$.}
Exactly as before,
\[
A=\Var(r_{\mathrm{o}}^\star)=\alpha(1-\alpha)-\E[e^2].
\]
Using $r(1-r)\le 1-r$ and $r(1-r)\le r$, we obtain
$\E[e^2]\le \alpha\,\E[1-r_{\mathrm{o}}^\star|Z=1]+(1-\alpha)\E[r_{\mathrm{o}}^\star|Z=0]$.
By Lemma~\ref{lem:leakage} applied to the two-component mixture $p_\alpha$ and
Lemma~\ref{lem:bc_strong_logconcave} (via $\mathrm{BC}\le \rho_{\mathrm{sep}}$), these are bounded by
$\varepsilon_{\mathrm{o}\to\mathrm{n}}$ and $\varepsilon_{\mathrm{n}\to\mathrm{o}}$, yielding $\E[e^2]\le v$ and hence $A\ge \alpha(1-\alpha)-v$.

\emph{(ii) Bound $\lambda_{\min}(C)$.}
Here $s_m=r_{\mathrm{n}}^\star(Y)\,\nabla V(Y-\mu_{\mathrm{n}})=(1-r_{\mathrm{o}}^\star)\,\nabla V(Y-\mu_{\mathrm{n}})$, so
\[
C=\E\big[(1-r_{\mathrm{o}}^\star)^2\,\nabla V(Y-\mu_{\mathrm{n}})\,\nabla V(Y-\mu_{\mathrm{n}})^\top\big]
\succeq (1-\alpha)\,\E\big[(1-r_{\mathrm{o}}^\star)^2\,G(Y)\,G(Y)^\top\mid Z=0\big],
\]
where $G(Y):=\nabla V(Y-\mu_{\mathrm{n}})$.
On $Z=0$, $(1-r_{\mathrm{o}}^\star)^2\ge 1-2r_{\mathrm{o}}^\star$, hence
\[
\E[(1-r_{\mathrm{o}}^\star)^2 GG^\top\mid Z=0]\succeq \E[GG^\top\mid Z=0]-2\,\E[r_{\mathrm{o}}^\star GG^\top\mid Z=0].
\]
Now $Y-\mu_{\mathrm{n}}\sim\rho$, so $\E[GG^\top\mid Z=0]=\E_{X\sim\rho}[\nabla V(X)\nabla V(X)^\top]\succeq mI$ by Lemma~\ref{lem:fisher_logconcave_location}.
Also $\E[r_{\mathrm{o}}^\star GG^\top\mid Z=0]$ is PSD and
\[
\big\|\E[r_{\mathrm{o}}^\star GG^\top\mid Z=0]\big\|_2
\le \E[r_{\mathrm{o}}^\star\,\|G\|^2\mid Z=0]
\le \sqrt{\E[r_{\mathrm{o}}^\star\mid Z=0]}\,\sqrt{\E[\|G\|^4\mid Z=0]}
\le \sqrt{\varepsilon_{\mathrm{n}\to\mathrm{o}}}\,\sqrt{G_4},
\]
using $r_{\mathrm{o}}^{\star 2}\le r_{\mathrm{o}}^\star$ and the definition of $G_4$.
Therefore $\lambda_{\min}(C)\ge (1-\alpha)m-2(1-\alpha)\sqrt{\varepsilon_{\mathrm{n}\to\mathrm{o}}}\sqrt{G_4}$.

\emph{(iii) Bound $\|B\|_2$.}
Now $B=\E[(r_{\mathrm{o}}^\star-\alpha)(1-r_{\mathrm{o}}^\star)\,G(Y)]$.
Repeating the algebra from the Gaussian case shows
\[
B=\E[\Delta(Y)\,G(Y)],\qquad |\Delta(Y)|\le 3|e(Y)|.
\]
Hence $\|B\|_2\le 3\,\E[|e|\,\|G\|]\le 3\sqrt{\E[e^2]}\sqrt{\E[\|G\|^2]}\le 3\sqrt{v}\sqrt{G_2}$.

\emph{(iv) Conclude.}
As before, Weyl's inequality yields
$\lambda_{\min}(H_\star)\ge \min\{A,\lambda_{\min}(C)\}-\|B\|_2$, giving the stated bound.

\medskip
\textbf{Proof of Part (C):}
If $\mu_\star=\lambda_{\min}(H_\star)>0$, continuity of the Hessian implies that in some neighborhood $U$ of $\theta^\star$,
$\nabla^2 L(\theta)\succeq \frac{\mu_\star}{2}I$, i.e.\ $L$ is $\mu_\star/2$-strongly convex on $U$.
Strong convexity implies the PL inequality $$\|\nabla L(\theta)\|^2\ge (\mu_\star/2)\,(L(\theta)-L(\theta^\star))=(\mu_\star/2)L(\theta)$$ on $U$.
Choose $\varepsilon>0$ so that the sublevel set $\{L\le\varepsilon\}\subset U$.
Along the gradient flow, $$\frac{d}{dt}L(\theta(t))=-\|\nabla L(\theta(t))\|^2,$$ combining with PL and integrating yields
$L(\theta(t))\le L(\theta(0))e^{-(\mu_\star/2)t}$ for all $t\ge 0$.
\end{proof}

\end{document}